\documentclass[journal]{IEEEtran}
\usepackage{cite}
\usepackage{amsmath,amssymb,amsfonts,mathrsfs}
\usepackage{algorithm,algorithmic,setspace}
\usepackage{graphicx}
\usepackage{textcomp}
\usepackage{xcolor}
\usepackage{mathtools}
\usepackage{dsfont}
\usepackage{subfig}
\usepackage{algorithm}
\usepackage{algorithmic}

\usepackage{epstopdf} 
\usepackage{graphics} 
\usepackage{subfig}
\usepackage{epstopdf} 
\usepackage{url}
\usepackage{color,soul}
\usepackage{amsthm}
\usepackage[colorlinks=true,linkcolor=black]{hyperref}

\DeclarePairedDelimiter{\norm}{\lVert}{\rVert}

\theoremstyle{definition}

\newtheorem{theorem}{Theorem}
\newtheorem{assumption}{Assumption}

\newtheorem{proposition}{Proposition}
\newtheorem{lemma}{Lemma}
\newtheorem{remark}{Remark}

\allowdisplaybreaks

\ifCLASSINFOpdf
\else
\fi

\title{
	Hybrid Feedback for Autonomous Navigation in Environments with Arbitrary Convex Obstacles 
}

\author{Mayur Sawant, Soulaimane Berkane, Ilia Polushin and Abdelhamid Tayebi  
	\thanks{This work was supported by the National Sciences and Engineering Research Council of Canada (NSERC), under the grants RGPIN-2020-06270, RGPIN-2020-0644 and RGPIN-2020-04759.}
	\thanks{M. Sawant, I. Polushin and A. Tayebi are with the Department of Electrical and Computer Engineering, Western University, London, ON N6A 3K7, Canada. (e-mail: {\tt\small msawant2, ipolushi, atayebi@uwo.ca}). S. Berkane is with the Department of Computer Science and Engineering, University of Quebec in Outaouais, 101 St-Jean Bosco, Gatineau, QC, J8X 3X7, Canada (e-mail: {\tt\small soulaimane.berkane@uqo.ca}). A. Tayebi and S. Berkane are also with the Department of Electrical Engineering, Lakehead University, Thunder Bay, ON P7B 5E1, Canada. (e-mail: {\tt\small atayebi, sberkane@lakeheadu.ca}).}%
}%

\begin{document}

	\maketitle


\begin{abstract}
We develop an autonomous navigation algorithm for a robot operating in two-dimensional environments cluttered with obstacles having arbitrary convex shapes. The proposed navigation approach relies on a  hybrid feedback to guarantee global asymptotic stabilization of the robot towards a predefined target location while ensuring the forward invariance of the obstacle-free workspace. The main idea consists in designing an appropriate switching strategy between the \textit{move-to-target} mode and the \textit{obstacle-avoidance} mode based on the proximity of the robot with respect to the nearest obstacle. The proposed hybrid controller generates continuous velocity input trajectories when the robot is initialized away from the boundaries of the unsafe regions. Finally, we provide an algorithmic procedure for the sensor-based implementation of the proposed hybrid controller and validate its effectiveness through some simulation results.
\end{abstract}
\section{Introduction}
The development of autonomous navigation techniques in realistic environments is one of the major trends in mobile robotics. One of the widely explored techniques in that regard is the artificial potential fields (APF) \cite{khatib1986real}. A combination of an attractive field which pulls the robot towards the target location and a repulsive field which pushes the robot away from the obstacle boundaries is used to generate a collision-free path between the initial and final locations. However, this approach suffers from the presence of undesired local minima. To mitigate this issue, in \cite{koditschek1990robot}, the authors proposed a navigation function (NF)-based approach which, provided a proper parameter tuning, ensures almost global convergence of the robot towards the target location. The NF-based approaches are directly applicable to sphere world environments \cite{rimon1992exact}, \cite{verginis2021adaptive} or environments that contain sufficiently curved obstacles \cite{filippidis2012navigation}. To make it applicable for environments consisting of more general convex and star-shaped obstacles, one can utilize the diffeomorphic mappings provided in \cite{rimon1992exact}, \cite{li2018navigation}, to transform a given environment into a sphere world. However, in order to perform these diffeomorphic mappings, the robot must have global knowledge about the environment, which makes the NF-based mobile robot navigation schemes less attractive in practical applications. 

Another research work \cite{paternain2017navigation}, extended the NF-based approach to environments consisting of convex obstacles wherein the authors provided a sufficient condition on the eccentricity of the obstacles to ensure almost global convergence in the local neighbourhood of the \textit{a priori} unknown target location. However, this approach is limited to obstacles which are neither too flat nor too close to the target. In \cite{kumar2019navigation}, for the case of ellipsoidal worlds, the authors removed the flatness limitation in \cite{paternain2017navigation}, by providing a controller design which locally transforms the region near the obstacle into a spherical region by using the Hessian information. In \cite{loizou2011navigation}, the authors provided a methodology for the design of a harmonic potential-based NF, with almost global convergence guarantees, for autonomous navigation in \textit{a priori} known environments which are diffeomorphic to the point world. The proposed NF is \textit{correct-by-construction} \textit{i.e.}, it does not have undesired local minima by design. This work was extended in \cite{loizou2021correct}, for unknown environments with a sensor-based robot navigation approach. However, similar to \cite{paternain2017navigation}, the shape of the obstacles is assumed to become known when the robot visits their respective neighbourhood. In the research work \cite{arslan2019sensor}, the authors proposed a purely reactive power diagram-based approach for robots operating in environments cluttered with unknown but sufficiently separated and strongly convex obstacles while ensuring almost global asymptotic stabilization towards the target location. This approach has been extended in \cite{vasilopoulos2018reactive}, for partially known non-convex environments, wherein it is assumed that the robot has the geometrical information of the non-convex obstacles but not their locations in the workspace.

However, using the approaches discussed above, one can at best provide almost global convergence guarantees. The appearance of undesired equilibria is unavoidable when considering continuous time-invariant vector fields \cite{koditschek1990robot}. Recently, in \cite{vrohidis2018prescribed}, for a single robot navigation in the sphere world, the authors proposed a time-varying vector field planner which utilizes the prescribed performance control technique to impose predetermined convergence to some neighbourhood of the target location while avoiding collisions, from all initial conditions. In \cite{sanfelice2006robust} and \cite{casau2019hybrid}, hybrid control techniques are used to ensure robust global asymptotic stabilization in $\mathbb{R}^2$ of the robot towards the target location while avoiding collision with a single spherical obstacle. The approach in \cite{sanfelice2006robust} has been extended in \cite{poveda2018hybrid}, to steer a group of planar robots in formation towards the source of an unknown but measurable signal, while avoiding a single obstacle. In \cite{braun2018unsafe}, the authors proposed a hybrid control law to globally asymptotically stabilize a class of linear systems while avoiding neighbourhoods of unsafe isolated points.

 In other works such as \cite{matveev2011method}, \cite{berkane2021obstacle}, the proposed hybrid control techniques allow the robot to operate either in the \textit{obstacle-avoidance} mode when it is in close proximity of an obstacle or in the \textit{move-to-target} mode when it is away from the obstacles. The strategies used in these research works are similar to the point robot path planning algorithms referred to as the bug algorithms \cite{lumelsky1986dynamic}. For some special obstacle arrangements, when the robot instead of converging to the predefined target location, retraces a previously followed path, these algorithms terminate the path planning process establishing the failure to converge to the target location due to the presence of a closed trajectory around the target location. The authors in \cite{matveev2011method}, \cite{berkane2021obstacle}, remove these special scenarios by restricting the possible inter-obstacle arrangements. For example, in \cite[Assumption 10]{matveev2011method}, the authors assume that if there is a blocking obstacle $\mathcal{O}_b$ for a given obstacle $\mathcal{O}_a$, then every point on the boundary of the blocking obstacle $\mathcal{O}_b$, which can be connected to the target via a line segment without intersecting the interior of the obstacle $\mathcal{O}_b$, must lie closer to the target location than any point on the given obstacle $\mathcal{O}_a$, in the sense of the Euclidean norm. In \cite[Theorem 2]{berkane2021obstacle}, for the case of ellipsoidal worlds, the authors require the obstacles to be sufficiently pairwise disjoint, see \cite[Definition 2]{berkane2021obstacle}.

In this paper, we propose a hybrid controller that allows to steer a holonomic planar robot, modelled as a single integrator, to reach a predefined target location while avoiding convex obstacles. The proposed controller, enjoying global asymptotic stability guarantees, operates in the \textit{move-to-target} mode when the robot is away from the obstacles and in the \textit{obstacle-avoidance} mode when it is in close proximity to an obstacle. The main contributions of the proposed research work are as follows:
 \begin{enumerate}
     \item \textit{Global asymptotic stability:} The proposed autonomous navigation solution provides global asymptotic stability guarantees for robots operating in environments with convex obstacles of arbitrary shapes. Note that the few existing results in the literature achieving such strong stability results are of hybrid type and are restricted to elliptically-shaped convex obstacles \cite{berkane2021obstacle}.  
     \item \textit{Arbitrarily-shaped convex obstacles:} The proposed hybrid feedback controller is applicable to environments consisting of convex obstacles with arbitrary shapes. Compared to this, the recently developed separating hyperplane-based approach is restricted to smooth obstacles which satisfy some curvature conditions \cite[Assumption 2]{arslan2019sensor}. Similarly, in \cite{berkane2021obstacle}, the obstacles are assumed to be ellipsoidal.
     \item \textit{Arbitrary inter-obstacle arrangements:} There are no restrictions on the inter-obstacle arrangements such as those in \cite[Assumption 10]{matveev2011method}, \cite[Theorem 2]{berkane2021obstacle}, except for the widely used mild ones stated in Assumption 1, \textit{i.e.}, the robot can pass in between any two obstacles while maintaining a positive distance.
     \item \textit{Continuous vector field:} The proposed hybrid controller generates continuous velocity input trajectories as long as the robot is initialized away from the boundaries of the unsafe regions. This is a very interesting feature for practical implementations that distinguishes our approach with respect to the hybrid approach of \cite{berkane2021obstacle}.      
     \item \textit{Applicable in \text{a priori} unknown environments:} The proposed obstacle avoidance approach can be implemented using only range scanners ({\it e.g.,} LiDAR)  without \textit{a priori} global knowledge of the environment (sensor-based technique). 
 \end{enumerate}

The remainder of paper is organized as follows. In Section \ref{sec:preliminaries}, we provide the notations and some preliminaries that will be used throughout the paper. The problem formulation is given in Section \ref{sec:problem_statement}, and the proposed hybrid control algorithm is presented in Section \ref{sec:hybrid_controller_design}. The stability and safety guarantees of the proposed navigation control scheme are provided in Section \ref{sec:stability}. A sensor-based implementation of the proposed obstacle avoidance algorithm, using 2D range scanners (LiDAR), is given in Section \ref{sensor-based-implementation}. Simulation results are given in Section \ref{section:simulation} to illustrate the effectiveness of the algorithm, and the paper is wrapped up with some concluding remarks in Section \ref{sec:conclusion}.

\section{Notations and Preliminaries}\label{sec:preliminaries}
\subsection{Notations} The sets of real, non-negative real and natural numbers are denoted by $\mathbb{R}$, $\mathbb{R}_{\geq}$ and $\mathbb{N}$, respectively. We identify vectors using bold lowercase letters. The Euclidean norm of a vector $\mathbf{p}\in\mathbb{R}^n$ is denoted by $\norm{\mathbf{p}}$, and an Euclidean ball of radius $r>0$ centered at $\mathbf{p}$ is represented by $\mathcal{B}_r(\mathbf{p}) = \{\mathbf{q}\in\mathbb{R}^n|\norm{\mathbf{q} - \mathbf{p}} \leq r\}$. Let $\mathbf{p}, \mathbf{q}\in\mathbb{R}^2$, then $\alpha_s(\mathbf{p}, \mathbf{q})$ denotes the angle measured from the vector $\mathbf{p}$ to $\mathbf{q}$. In this case, measurement in the counter-clockwise direction is considered positive. 

For two sets $\mathcal{A}, \mathcal{B}\subset\mathbb{R}^n$, the relative complement of $\mathcal{B}$ with respect to $\mathcal{A}$ is denoted by $\mathcal{A}\backslash\mathcal{B} =\{\mathbf{a}\in\mathcal{A}|\mathbf{a}\notin \mathcal{B}\}$. Given a set $\mathcal{A}$, the symbols $\partial\mathcal{A}, \mathcal{A}^{\circ}$, $\mathcal{A}^c$ and $\bar{\mathcal{A}}$ represent the boundary, interior, complement and the closure of the set $\mathcal{A}$, respectively, where $\partial\mathcal{A} = \bar{\mathcal{A}}\backslash\mathcal{A}^{\circ}$. Let $\mathcal{V}\subset\mathbb{R}^2$ be a continuous curve with two end points $\mathbf{v}_1$ and $\mathbf{v}_2$, then $\mathcal{V}^{\bullet} = \mathcal{V}\backslash\{\mathbf{v}_1, \mathbf{v}_2\}$ denotes the curve $\mathcal{V}$ without the end points. Given two sets $\mathcal{A}\subset\mathbb{R}^n$ and $\mathcal{B}\subset\mathbb{R}^n$, the Minkowski sum of $\mathcal{A}$ and $\mathcal{B}$ is denoted by, $\mathcal{A} \oplus\mathcal{B} = \{\mathbf{a} + \mathbf{b}|\mathbf{a}\in\mathcal{A}, \mathbf{b}\in\mathcal{B}\}$ and if $\mathcal{A}$ and $\mathcal{B}$ are convex then $\mathcal{A}\oplus\mathcal{B}$ is also convex \cite[Theorem 3.1]{rockafellar1970convex}. Given a set $\mathcal{A}$, its dilated version with $r \geq 0$, is represented by $\mathcal{D}_r(\mathcal{A}) = \mathcal{A} \oplus\mathcal{B}_r(\mathbf{0})$.

\subsection{Projection maps} 

\subsubsection{Projection on a set}
The Euclidean distance of a point $\mathbf{q}\in\mathbb{R}^n$ to a closed set $\mathcal{V}\subset\mathbb{R}^n$ is given by
\begin{equation}
    d(\mathbf{q}, \mathcal{V})  = \underset{\mathbf{a}\in\mathcal{V}}{\text{min }}\norm{\mathbf{q} - \mathbf{a}}.
\end{equation}
We denote all points in the set $\mathcal{V}$ which are closest to $\mathbf{q}$, in the sense of the Euclidean norm, as $\Pi(\mathbf{q}, \mathcal{V})$ such that
\begin{equation}
    \Pi(\mathbf{q}, \mathcal{V})  = \underset{\mathbf{a}\in\mathcal{V}}{\text{arg min }}\norm{\mathbf{q} - \mathbf{a}},    \label{projection_on_a_set}
\end{equation}
where $\Pi(\mathbf{q}, \mathcal{V})$ is referred to as the projection of $\mathbf{q}$ on $\mathcal{V}$, and if the set $\mathcal{V}$ is convex, then $\Pi(\mathbf{q}, \mathcal{V})$ is a singleton \cite[Section 8.1]{boyd2004convex}.

\subsubsection{Orthogonal rotation}
The orthogonal rotation of a non-zero vector $\mathbf{p}\in\mathbb{R}^2\backslash\{\mathbf{0}\}$, represented by $\nu_z(\mathbf{p})$ is defined as
\begin{equation}
    \nu_z(\mathbf{p}) = \begin{bmatrix}0 & z\\ -z & 0\end{bmatrix}\mathbf{p}, \; z\in\{-1, 1\},
\end{equation}
where $z = 1$ corresponds to the clockwise rotation while $z = -1$ denotes the counter-clockwise rotation of the vector $\mathbf{p}$, respectively.

\subsection{Geometric subsets of $\mathbb{R}^n$}
\subsubsection{Line} Let $\mathbf{q}\in\mathbb{R}^n\backslash\{\mathbf{0}\}$, and $\mathbf{p}\in\mathbb{R}^n$, then a line passing through $\mathbf{p}$ in the direction of $\mathbf{q}$ is defined as
    \begin{equation}
        \mathcal{L}(\mathbf{p}, \mathbf{q}) := \{\mathbf{x}\in\mathbb{R}^n|\mathbf{x} = \mathbf{p} + \lambda\mathbf{q}, \lambda \in \mathbb{R}\}\label{equation_of_line}.
    \end{equation}
    If $\lambda \geq 0$ (respectively, $\lambda > 0$), then we get the positive half-line $\mathcal{L}_{\geq}(\mathbf{p}, \mathbf{q})$ (respectively, $\mathcal{L}_{>}(\mathbf{p}, \mathbf{q})$). Similarly, we define the negative half-lines for $\lambda \leq 0$ and $\lambda < 0$ as $\mathcal{L}_{\leq}(\mathbf{p}, \mathbf{q})$ and $\mathcal{L}_{<}(\mathbf{p}, \mathbf{q})$, respectively.
 
\subsubsection{Line segment} Let $\mathbf{p}\in\mathbb{R}^n$ and $\mathbf{q}\in\mathbb{R}^n$, then a line segment joining $\mathbf{p}$ and $\mathbf{q}$ is given by
\begin{equation}
    \mathcal{L}_s(\mathbf{p}, \mathbf{q}) := \{\mathbf{x}\in\mathbb{R}^n|\mathbf{x} = \lambda \mathbf{p} + (1 - \lambda) \mathbf{q}, \lambda \in[0, 1]\}.\label{line_segment}
\end{equation}
    
\subsubsection{Hyperplane} Given $\mathbf{p}\in\mathbb{R}^n$, and $\mathbf{q}\in\mathbb{R}^n\backslash\{\mathbf{0}\}$, a hyperplane passing through $\mathbf{p}$ and orthogonal to $\mathbf{q}$ is given by
    \begin{align}
    \mathcal{P}(\mathbf{p}, \mathbf{q}) := \{\mathbf{x}\in\mathbb{R}^n| \mathbf{q}^\intercal(\mathbf{x} - \mathbf{p}) = 0\}.\label{equation_of_hyperplane}
    \end{align}
    The hyperplane divides the Euclidean space $\mathbb{R}^n$ into two half-spaces \textit{i.e.}, a closed positive half-space $\mathcal{P}_{\geq}(\mathbf{p}, \mathbf{q})$ and a closed negative half-space $\mathcal{P}_{\leq}(\mathbf{p}, \mathbf{q})$ which are obtained by substituting `$=$' with `$\geq$' and `$\leq$' respectively, in the right-hand side of \eqref{equation_of_hyperplane}. We also use the notations $\mathcal{P}_{>}(\mathbf{p}, \mathbf{q})$ and $\mathcal{P}_{<}(\mathbf{p} ,\mathbf{q})$ to denote the open positive and the open negative half-spaces such that $\mathcal{P}_{>}(\mathbf{p}, \mathbf{q}) = \mathcal{P}_{\geq}(\mathbf{p}, \mathbf{q})\backslash\mathcal{P}(\mathbf{p}, \mathbf{q})$ and $\mathcal{P}_{<}(\mathbf{p} ,\mathbf{q})= \mathcal{P}_{\leq}(\mathbf{p}, \mathbf{q})\backslash\mathcal{P}(\mathbf{p}, \mathbf{q})$.
\subsubsection{Supporting hyperplane \cite{boyd2004convex}} Given a closed convex set $\mathcal{A}\subset\mathbb{R}^n$, $\mathbf{p}\in\partial\mathcal{A}$ and $\mathbf{q}\in\mathbb{R}^n\backslash\{ \mathbf{0}\}$, a hyperplane $\mathcal{P}(\mathbf{p}, \mathbf{q})$ is a supporting hyperplane to $\mathcal{A}$ at point $\mathbf{p}$, if
\begin{align}
\mathbf{q}^{\intercal}(\mathbf{x} - \mathbf{p}) \leq 0, \;\forall \mathbf{x}\in\mathcal{A}.\label{supporting_hyperplane_condition}
\end{align}
In this case, the vector $\mathbf{q}$ is normal to the set $\mathcal{A}$ at $\mathbf{p}$, the supporting hyperplane $\mathcal{P}(\mathbf{p}, \mathbf{q})$ is tangent to $\mathcal{A}$ at $\mathbf{p}$, and the negative half-space $\mathcal{P}_{\leq}(\mathbf{p}, \mathbf{q})$ contains $\mathcal{A}.$
    
\subsubsection{Convex cone \cite{boyd2004convex}} Given $\mathbf{p}_1\in\mathbb{R}^2\backslash\mathbf{0}$, and $ \mathbf{p}_2\in\mathbb{R}^2\backslash\mathbf{0}$, a convex cone $\mathcal{C}(\mathbf{p}_1, \mathbf{p}_2)$ with its vertex at the origin is defined as
    \begin{equation}
        \mathcal{C}(\mathbf{p}_1, \mathbf{p}_2) :=  \{\mathbf{x}\in\mathbb{R}^2|\mathbf{x} = \lambda_1\mathbf{p}_1 + \lambda_2\mathbf{p}_2, \forall\lambda_1 \geq 0, \forall\lambda_2 \geq 0\}.
    \end{equation}

The next Lemma provides a property of the projection \eqref{projection_on_a_set} such that given a closed convex set $\mathcal{V}$ and a point $\mathbf{q}\in\mathbb{R}^n$ outside the set $\mathcal{V}$, the vector joining the projection of $\mathbf{q}$ on $\mathcal{V}$ with the point $\mathbf{q}$ is always normal to the dilated versions of the set $\mathcal{V}$, $\mathcal{D}_{r}(\mathcal{V})$ where $r\in[0, d(\mathbf{q}, \mathcal{V})].$
\begin{lemma}
Let $\mathcal{V}\subset\mathbb{R}^n$ be a closed convex set and $\mathbf{q}\in\mathbb{R}^n\backslash\mathcal{V}$. Then $(\mathbf{q} - \Pi(\mathbf{q}, \mathcal{V}))$ is normal to $\mathcal{D}_r(\mathcal{V})$ at $\Pi(\mathbf{q}, \mathcal{D}_r(\mathcal{V}))$ and $\mathcal{P}(\Pi(\mathbf{q}, \mathcal{D}_r(\mathcal{V})), (\mathbf{q} - \Pi(\mathbf{q}, \mathcal{V})))$ is a supporting hyperplane to $\mathcal{D}_r(\mathcal{V})$ at $\Pi(\mathbf{q}, \mathcal{D}_r(\mathcal{V}))$ for $r\in[0, d(\mathbf{q}, \mathcal{V})]$.
\label{lemma:normal_to_dilated_obstacle}
\end{lemma}
\begin{proof}
See Appendix \ref{proof:normal_lemma}.
\end{proof}
\subsection{Hybrid system framework}
\label{sec:hybrid_threory}
A hybrid dynamical system \cite{goedel2012hybrid} is represented using differential and difference inclusions for the state $\mathbf{\xi}\in\mathbb{R}^n$ as follows:
\begin{align}
    \begin{cases}\begin{matrix}\mathbf{\dot{\xi}} \in \mathbf{F}(\mathbf{\xi}) , & \mathbf{\xi} \in \mathcal{F}, \\
    \mathbf{\xi}^{+}\in \mathbf{J}(\mathbf{\xi}), & \mathbf{\xi}\in\mathcal{J},\end{matrix}\end{cases}\label{hybrid_system_general_model}
\end{align}
where the \textit{flow map} $\mathbf{F}:\mathbb{R}^n\rightrightarrows\mathbb{R}^n$ is the differential inclusion which governs the continuous evolution when $\mathbf{\xi}$ belongs to the \textit{flow set} $\mathcal{F}\subseteq\mathbb{R}^n$, where the symbol `$\rightrightarrows$' represents set-valued mapping. The \textit{jump map} $\mathbf{J}:\mathbb{R}^n\rightrightarrows\mathbb{R}^n$ is the difference inclusion that governs the discrete evolution when $\mathbf{\xi}$ belongs to the \textit{jump set} $\mathcal{J}\subseteq\mathbb{R}^n$. The hybrid system \eqref{hybrid_system_general_model} is defined by its data and denoted as $\mathcal{H} = (\mathcal{F}, \mathbf{F}, \mathcal{J}, \mathbf{J}).$

A subset $\mathbb{T}\subset\mathbb{R}_{\geq}\times\mathbb{N}$ is a \textit{hybrid time domain} if it is a union of a finite or infinite sequence of intervals $[t_j, t_{j + 1}]\times \{j\},$ where the last interval (if existent) is possibly of the form $[t_j, T)$ with $T$ finite or $T = +\infty$. The ordering of points on each hybrid time domain is such that $(t, j)\preceq(t^{\prime}, j^{\prime})$ if $t < t^{\prime},$ or $t = t^{\prime}$ and $j \leq j^{\prime}$. A \textit{hybrid solution} $\phi$ is maximal if it cannot be extended, and complete if its domain dom $\phi$ (which is a hybrid time domain) is unbounded.

\section{Problem Formulation}
\label{sec:problem_statement}We consider a disk-shaped robot with radius $r \geq 0$, operating in a two dimensional Euclidean space $\mathcal{W}\subseteq\mathbb{R}^2$ as shown in Fig. \ref{list_of_dilated_regions}. The workspace contains a finite number of compact convex obstacles $\mathcal{O}_i\subset\mathcal{W}, i\in\{1, \ldots, b\} := \mathbb{I}$, where $b \in\mathbb{N}$ is the total number of obstacles. The task is to reach a predefined obstacle-free target location from any obstacle-free region while avoiding collisions. Without loss of generality, consider the origin $\mathbf{0}$ as the target location. Throughout this work we will make the following feasibility assumptions.
\begin{assumption}The minimum separation between any pair of obstacles should be greater than $2r$ \textit{i.e.,} for all $i, j\in\mathbb{I}, i\ne j$, one has
$
d(\mathcal{O}_i, \mathcal{O}_j):=\underset{\mathbf{p}\in\mathcal{O}_i, \mathbf{q}\in\mathcal{O}_j}{\text{ min }}\norm{\mathbf{p} - \mathbf{q}}>  2r.
$
\label{assumption:robot_pass_through}
\end{assumption}

According to Assumption \ref{assumption:robot_pass_through} and the compactness of the obstacles, there exists a minimum separating distance between any pair of obstacles $\bar{r} = \underset{i, j\in\mathbb{I}, i\ne j}{\min}d(\mathcal{O}_i, \mathcal{O}_j) > 2r$. Moreover, for collision-free navigation we require $d(\mathbf{0}, \mathcal{O}_{\mathcal{W}}) - r > 0$, where $\mathcal{O}_{\mathcal{W}} := \bigcup_{i\in\mathbb{I}}\mathcal{O}_i$. We define a positive real $\bar{r}_s$ as
\begin{equation}
    \bar{r}_s = \underset{}{\min}\left\{\frac{\bar{r}}{2} - r, d(\mathbf{0}, \mathcal{O}_{\mathcal{W}}) -r\right\}.\label{choice_of_epsilon_d}
\end{equation}
We then pick an arbitrarily small value $r_s\in(0, \bar{r}_s)$ as the minimum distance that the robot should maintain with respect to any obstacle.

The obstacle-free workspace is then defined as
\begin{equation}
    \mathcal{W}_0 := \mathcal{W}\backslash\bigcup_{i\in\mathbb{I}}\big(\mathcal{O}_i\big)^{\circ}.\nonumber
\end{equation}
Given $y\geq0$, an eroded version of the obstacle-free workspace, $\mathcal{W}_y$ is defined as
\begin{equation}
    \mathcal{W}_y := \mathcal{W}\backslash\bigcup_{i\in\mathbb{I}}\big(\mathcal{D}_{y}(\mathcal{O}_i)\big)^{\circ} \subset\mathcal{W}_0.\label{eroded_workspace}
\end{equation}
Hence, $\mathcal{W}_{r_a}$ with $r_a = r + r_s$ is a free workspace with respect to the center of the robot \textit{i.e.}, $\mathbf{x} \in\mathcal{W}_{r_a} \iff\mathcal{B}_{r_a}(\mathbf{x})\subset\mathcal{W}_0$. The robot is governed by a single integrator dynamics
\begin{equation}
    \mathbf{\dot{x}} = \mathbf{u},\label{single_integrator_control_law}
\end{equation}
where $\mathbf{u}\in\mathbb{R}^2$ is the control input. Given a target location in the interior of the obstacle-free workspace \textit{i.e.}, $\mathbf{0}\in\big(\mathcal{W}_{r_a}\big)^{\circ}$, we aim to design a feedback control law such that:
\begin{enumerate}
    \item the obstacle-free space $\mathcal{W}_{r_a}$ is forward invariant,
    \item the target location $\mathbf{x} = \mathbf{0}$ is a globally asymptotically stable equilibrium for the closed-loop system.
\end{enumerate}

\begin{figure}
    \centering
    \includegraphics[width = 0.8\linewidth]{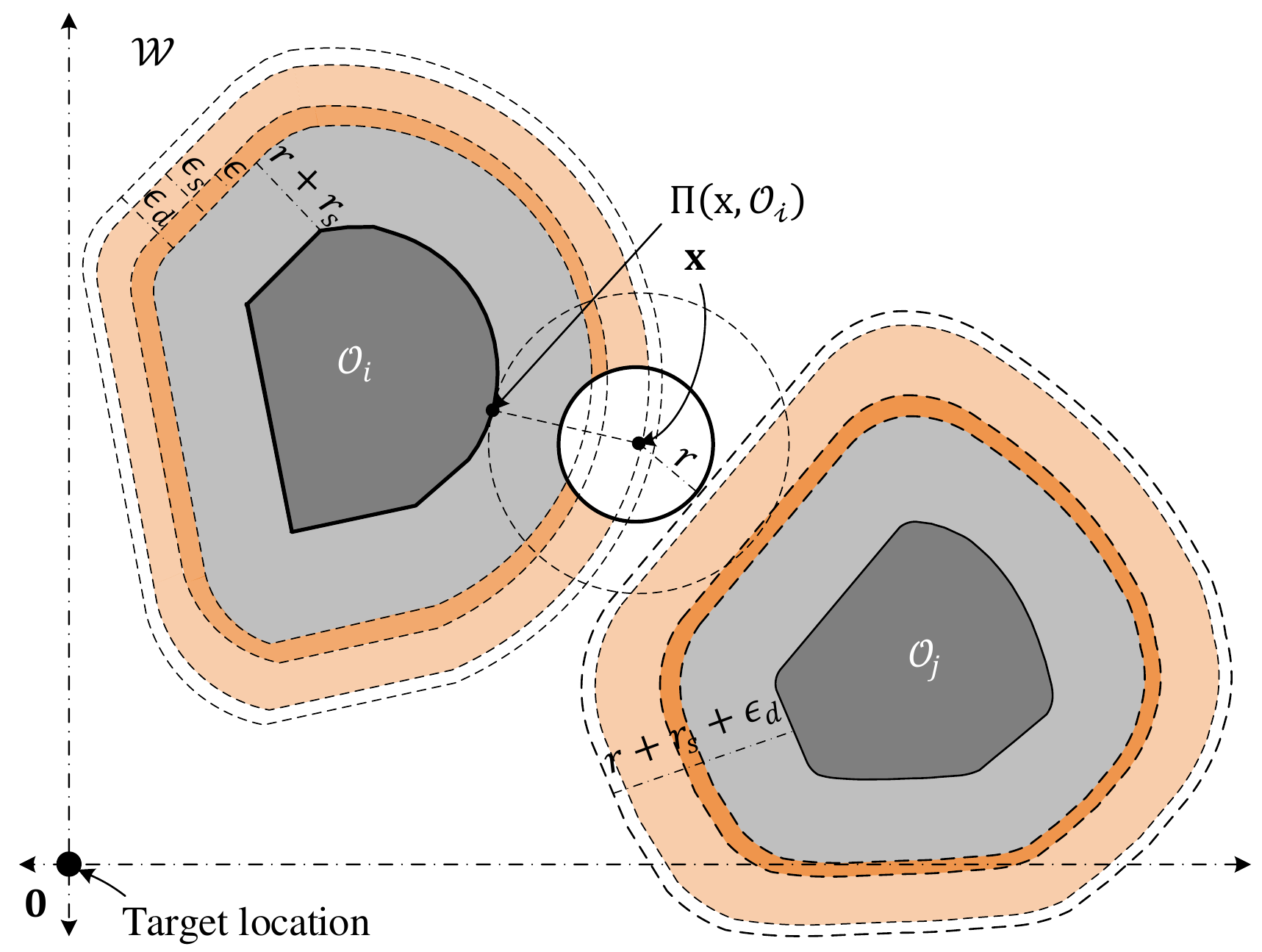}
    \caption{An illustration of the valid workspace $\mathcal{W}$ with two convex obstacles separated by the distance greater than or equal to $2(r_a + \epsilon_d)$. The $\epsilon_d-$neighbourhood of dilated obstacles is further partitioned into three region such that $\epsilon_d > \epsilon_s > \epsilon > 0.$ }
    \label{list_of_dilated_regions}
\end{figure}

\section{Hybrid Control for Obstacle Avoidance}\label{sec:hybrid_controller_design}
In the proposed scheme, similar to \cite{ berkane2021obstacle}, depending upon the value of a mode indicator variable $m\in\{-1, 0, 1\}:=\mathbb{M}$, the robot operates in two different modes, namely the \textit{move-to-target} mode $(m = 0)$ when it is away from the obstacles and the \textit{obstacle-avoidance} mode $(m \in\{-1, 1\})$ when it is in the close proximity of any obstacle. In the \textit{move-to-target} mode, the robot moves straight towards the target whereas during the \textit{obstacle-avoidance} mode the robot moves around the nearest obstacle, either in the clockwise direction $(m = 1)$ or in the counter-clockwise direction $(m = -1)$. We allow the robot to exit the \textit{obstacle-avoidance} mode whenever it can move straight towards the target without reducing its proximity from the nearest obstacle. We ensure that when the robot switches between the two modes, the change in the velocity vector remains continuous.

Now, notice that, if the robot were to arbitrarily choose between the clockwise and counter-clockwise
motions when it switches from the \textit{move-to-target} mode to the \textit{obstacle-avoidance} mode, then for some inter-obstacle arrangements the robot might get trapped in a closed trajectory around the target location. To that end, we propose a switching strategy which allows the robot to decide between the clockwise motion and the counter-clockwise motion based on its location and a line segment joining its initial location and the target location. When the robot operates in the \textit{obstacle-avoidance} mode, it always attempts to go towards this line, which in turn ensures that the robot does not get trapped in a closed trajectory around the target and eventually converges towards it.
Our proposed hybrid controller takes the form $\mathbf{u}(\mathbf{x}, m)$  with
\begin{equation}
    \underbrace{\begin{matrix}
    \dot{m}
    \end{matrix} \begin{matrix*}[l]\;=0,\end{matrix*}}_{(\mathbf{x}, m)\in\mathcal{F}_{\mathcal{W}}}\quad\quad\underbrace{\begin{matrix}
    m^+
    \end{matrix} = \mathbf{L}(\mathbf{x}, m)}_{(\mathbf{x}, m)\in\mathcal{J}_{\mathcal{W}}},\label{proposed_hybrid_controller_1}
\end{equation}
The sets $\mathcal{F}_{\mathcal{W}}$ and $\mathcal{J}_{\mathcal{W}}$ represent the flow and the jump sets, respectively. Next we provide a geometric construction of $\mathcal{F}_{\mathcal{W}}$ and $\mathcal{J}_{\mathcal{W}}$ which will be followed by an explicit design of $\mathbf{u}(\mathbf{x}, m)$ and $\mathbf{L}(\mathbf{x}, m).$ 

\subsection{Geometric construction of the flow and jump sets}

We define an $\epsilon_d-$neighbourhood around the dilated obstacle $\mathcal{D}_{r_a}(\mathcal{O}_i), i\in\mathbb{I}$, where $\epsilon_d\in(0,\bar{r}_s - r_s)$.\footnote{Note that the value of the parameter $\epsilon_d$ can be small. Since this parameter adds a safety margin around the robot's body, one can choose a sufficiently small value while compensating for the measurement noise.\label{footnote:smallepsilon}} 
For an obstacle $\mathcal{O}_i, i\in\mathbb{I}$, the compact tubular neighbourhood $\mathcal{T}(\mathcal{O}_i):=\mathcal{D}_{r_a + \epsilon_d}(\mathcal{O}_i)\backslash\big(\mathcal{D}_{r_a}(\mathcal{O}_i)\big)^{\circ}$ is partitioned into several sub-regions, as shown in Fig. \ref{Partitions_of_local_neighbourhood}, which then will be used to construct the jump and the flow sets used in \eqref{proposed_hybrid_controller_1}.
\begin{figure}
    \centering
    \includegraphics[width = 0.8\linewidth]{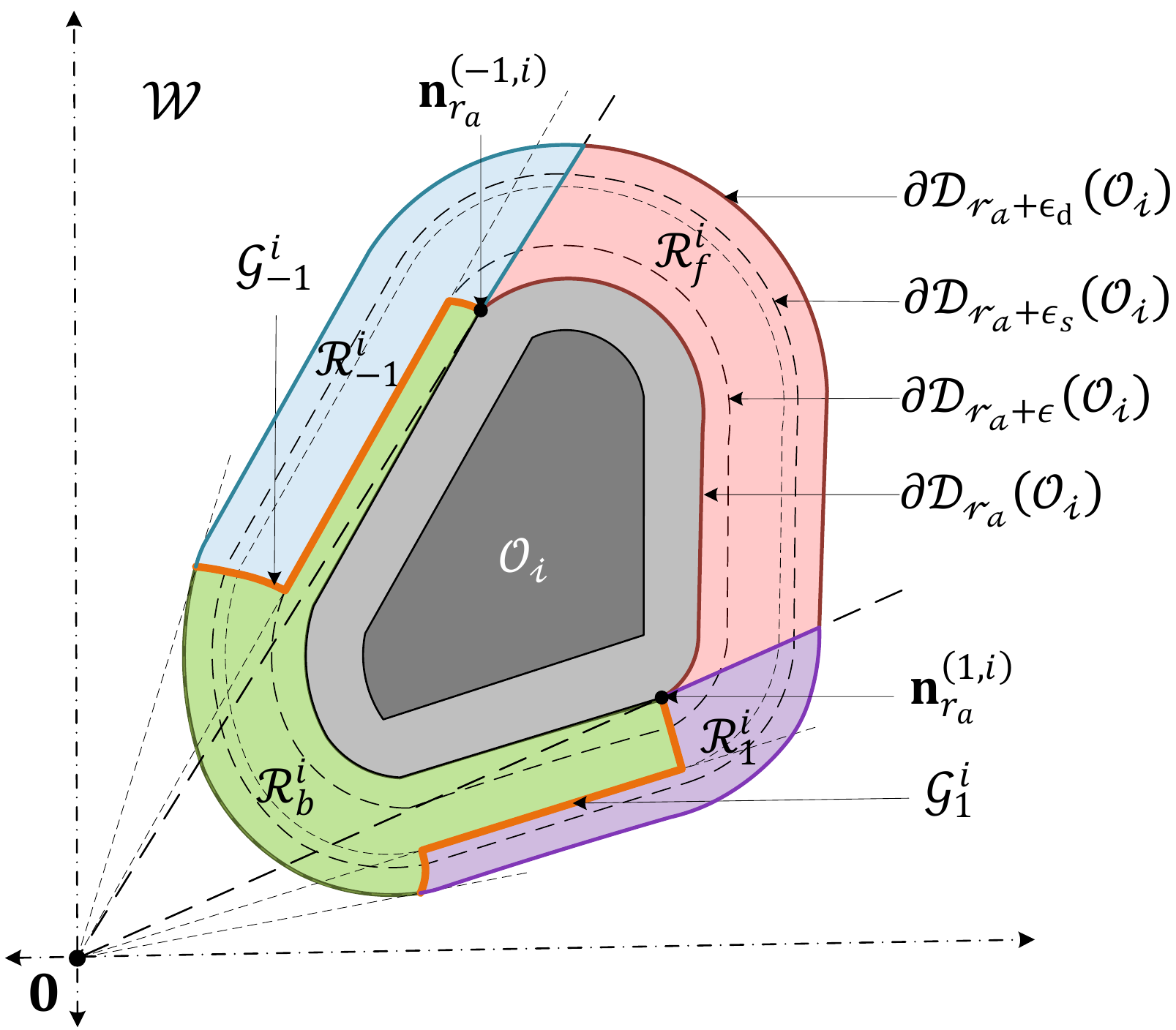}
    \caption{The partitions of the local neighbourhood of the obstacle $\mathcal{O}_i, \;i\in\mathbb{I}$ based on the location of the target.}
    \label{Partitions_of_local_neighbourhood}
\end{figure}
\subsubsection{Back region} $\mathcal{R}_b^i$ is defined as
\begin{equation}
        \mathcal{R}_b^i = \{\mathbf{q}\in \mathcal{T}(\mathcal{O}_i)|\mathbf{q}^{\intercal}(\mathbf{q} - \Pi(\mathbf{q}, \mathcal{O}_i)) \leq 0\}.\label{back_region}
\end{equation}
The back region $\mathcal{R}_b^i$ is a closed connected subset of the region $\mathcal{T}(\mathcal{O}_i)$ such that for all $\mathbf{q}\in\mathcal{R}_b^i$ the angle measured from the vector $\mathbf{q}$ to the vector $\mathbf{q} - \Pi(\mathbf{q}, \mathcal{O}_i)$ satisfies $\alpha_s(\mathbf{q}, (\mathbf{q} - \Pi(\mathbf{q}, \mathcal{O}_i))) \in[\frac{\pi}{2}, \frac{3\pi}{2}]$.

\subsubsection{Gates} $\mathcal{G}_m^i,\;m\in\{-1, 1\},$ defined as
\begin{equation}\mathcal{G}_{m}^i = \{\mathbf{q}\in \mathcal{T}(\mathcal{O}_i)|\alpha_s(\mathbf{q}, (\mathbf{q} - \Pi(\mathbf{q}, \mathcal{O}_i))) = \frac{-m\pi}{2}\},\label{gate_region}
\end{equation}
are the regions where the vectors $\mathbf{q}$ and $(\mathbf{q} - \Pi(\mathbf{q}, \mathcal{O}_i))$ are orthogonal to each other. In Proposition \ref{proposition:continuous_control}, it is shown that, while operating in the \textit{obstacle-avoidance} mode, away from the boundary of the respective obstacle's unsafe region, if the robot switches to the \textit{move-to-target} mode in the gate region, then the control input trajectories remain continuous. 

\subsubsection{Front region} $\mathcal{R}_f^i$ is defined as
\begin{equation}
    \begin{aligned}
        \mathcal{R}_f^i = \{&\mathbf{q}\in\mathcal{T}(\mathcal{O}_i)|\mathcal{L}_s(\mathbf{q}, \mathbf{0})\cap\left(\mathcal{D}_{r_a}(\mathcal{O}_i)\right)^{\circ}\ne\emptyset\}.\label{front_region}
    \end{aligned}
\end{equation}
If the robot, located in the front region $\mathcal{R}_f^i$, moves straight towards the origin, it will eventually enter in the unsafe region related to the obstacle $\mathcal{O}_i$. To avoid this, in this region the robot should switch to the \textit{obstacle-avoidance} mode.

\subsubsection{Side regions}
$\mathcal{R}_{m}^i$, $m\in\{-1, 1\}$, are constructed as
\begin{equation}
    \begin{aligned}
        \mathcal{R}_{1}^i  = \{&\mathbf{q}\in\mathcal{T}(\mathcal{O}_i)\backslash\big(\mathcal{R}_b^i\cup\mathcal{R}_f^i\big)|\\&\alpha_s(\mathbf{q}, \mathbf{q} - \Pi(\mathbf{q}, \mathcal{O}_i)) \in [-\pi/2, 0]\},\\
        \mathcal{R}_{-1}^i  = \{&\mathbf{q}\in\mathcal{T}(\mathcal{O}_i)\backslash\big(\mathcal{R}_b^i\cup\mathcal{R}_f^i\big)|\\&\alpha_s(\mathbf{q}, \mathbf{q} - \Pi(\mathbf{q}, \mathcal{O}_i)) \in [0, \pi/2]\}.
    \end{aligned}
\end{equation}
Since the intersection of the interior of the conic hull \cite[Section 2.1.5]{boyd2004convex} for the dilated obstacle $\mathcal{D}_{r_a}(\mathcal{O}_i)$, having its vertex at the origin, with the side regions is empty, for $\mathbf{x}\in \mathcal{R}_{m}^i, m\in\{-1, 1\}$, the robot can move straight towards the target location, see Fig. \ref{Partitions_of_local_neighbourhood}.

\begin{remark}
\label{remark:back_region_is_closed_connected_subset}

For an arbitrary compact convex obstacle $\mathcal{O}_i, i\in\mathbb{I},$ the projection of a point $\mathbf{q}\notin\mathcal{O}_i$ onto the obstacle $\mathcal{O}_i$ \textit{i.e.}, $\Pi(\mathbf{q}, \mathcal{O}_i)$ is continuous with respect to $\mathbf{q}$. As a result, the respective back region $\mathcal{R}_b^i$ defined in \eqref{back_region}, is a compact connected subset of the $\epsilon_d-$neighbourhood of the obstacle $\mathcal{O}_i$. Hence, one has for all $\mathbf{q}\in\mathcal{T}(\mathcal{O}_i)\backslash\mathcal{R}_b^i$
 \begin{equation}
     \mathbf{q}^\intercal(\mathbf{q} - \Pi(\mathbf{q}, \mathcal{O}_i)) > 0.\label{in_the_neighbourhood_other_than_back_region}
 \end{equation}
\end{remark}

This concludes the partitioning of the $\epsilon_d-$neighbourhood of the dilated obstacle $\mathcal{D}_{r_a}(\mathcal{O}_i)$ for some $i\in\mathbb{I}$. Similar regions are defined for all remaining obstacles. Next, we utilize these regions to define the flow and the jump sets for each mode of operation.

\subsubsection{Flow and jump sets (\textit{move-to-target} mode)} 
In this mode, the robot moves straight towards the origin. Consider Fig. \ref{jump_and_flow_sets} for a visual representation. As discussed earlier, the robot moving straight towards the origin should switch to the \textit{obstacle-avoidance} mode whenever it enters $\mathcal{R}_f^i,\;\forall i\in\mathbb{I}$, otherwise it will collide with the obstacle $\mathcal{O}_i$. Hence, we define the jump set $\mathcal{J}_0^i$ for each obstacle $i\in\mathbb{I}$ as
\begin{equation}
    \mathcal{J}_0^i:= \mathcal{D}_{r_a +\epsilon_s}(\mathcal{O}_i)\cap\bar{\mathcal{R}}_f^i,\label{jumpset_individual_obstacle}
\end{equation}
where $\epsilon_s\in(0, \epsilon_d)$. 

However, if the center of the robot is located in the back or side regions of an obstacle $\mathcal{O}_i, i\in\mathbb{I}$, the robot can navigate safely with respect to the obstacle $\mathcal{O}_i$ towards the target in the \textit{move-to-target} mode.
Hence, the flow set of the \textit{move-to-target} mode $\mathcal{F}_0^i$ for each obstacle $\mathcal{O}_i, i \in\mathbb{I}$, defined as
\begin{equation}
    \mathcal{F}_0^i:=\left(\mathcal{W}_{r_a}\backslash\big(\mathcal{D}_{r_a + \epsilon_s}(\mathcal{O}_i)\big)^{\circ}\right)\cup\mathcal{R}_{-1}^i\cup\mathcal{R}_{1}^i\cup\mathcal{R}_b^i,\label{flowset_individual_obstacle}
\end{equation}
includes the union of the back and side regions of the respective obstacle. Inspired by \cite[(17)]{berkane2021obstacle}, taking all obstacles into consideration, the flow and jump sets for the \textit{move-to-target} mode $m = 0$ are defined as
\begin{align}
    \mathcal{F}_0 :=\bigcap_{i\in\mathbb{I}}\mathcal{F}_0^i, \quad \mathcal{J}_0:=\bigcup_{i \in\mathbb{I}}\mathcal{J}_0^i.\label{stabilization_mode_jumpflow_set_final}
\end{align}
Next, we define the flow and jump sets for the \textit{obstacle-avoidance} mode.
\subsubsection{Flow and jump sets (\textit{obstacle-avoidance} mode)} This mode is activated only if the robot enters in the $\epsilon_d-$neighbourhood of some obstacle $\mathcal{O}_i, \;i\in\mathbb{I}$, which according to Assumption \ref{assumption:robot_pass_through} can only be valid for at most one obstacle at any given time. We now consider the construction of the flow and jump sets for the \textit{obstacle-avoidance} mode $(m\in\{-1, 1\})$ with a specific obstacle $\mathcal{O}_i$, as shown in Fig. \ref{jump_and_flow_sets}. Here the mode indicator variable $m = -1$ and $m = 1$, prompts the robot to move in the counter-clockwise and clockwise directions with respect to the obstacle's boundary $\partial\mathcal{O}_i$, respectively. Hence for $m\in\{-1, 1\}$ the flow sets are constructed as follows:
\begin{equation}
\begin{aligned}
    &\mathcal{F}_{m}^i :=  \mathcal{R}_{m}^i\cup\overline{\mathcal{ER}_f^i\backslash\mathcal{R}_b^i},\label{avoidance_flow_set_individual}
\end{aligned}
\end{equation}
where the set $\mathcal{ER}_f^i$ is defined as
\begin{equation}
    \begin{aligned}
        \mathcal{ER}_f^i = \{&\mathbf{q}\in \mathcal{T}(\mathcal{O}_i)|\mathcal{L}_s(\mathbf{q}, \mathbf{0})\cap\left(\mathcal{D}_{r_a + \epsilon}(\mathcal{O}_i)\right)^{\circ}\ne\emptyset\}.\label{enlarged_front_region}
    \end{aligned}
\end{equation}
where $\epsilon \in(0, \epsilon_s)$. 
The jump set $\mathcal{J}_{m}^i$ of the respective mode, which includes the relative complement of the set $\mathcal{W}_{r_a}$ with respect to the interior of the flow set $\big(\mathcal{F}_{m}^i\big)^{\circ}$, is defined as
\begin{equation}
\begin{aligned}
    \mathcal{J}_{m}^i :=& \left(\mathcal{W}_{r_a}\backslash(\mathcal{D}_{r_a+\epsilon_d}(\mathcal{O}_i))^{\circ}\right)\cup\mathcal{R}_b^i\cup\left(\mathcal{R}_{-m}^i\backslash\mathcal{ER}_f^i\right).\label{individual_jump_set_avoidance_mode}
    \end{aligned}
\end{equation}

Finally, taking all the obstacles into consideration, the flow and jump sets for the \textit{obstacle-avoidance} mode are defined as
\begin{equation}
\begin{aligned}
    \mathcal{F}_{m}&:=\bigcup_{i\in\mathbb{I}}\mathcal{F}_{m}^i,\quad\mathcal{J}_{m}:= \bigcap_{i\in\mathbb{I}}\mathcal{J}_{m}^i\label{avoidance_final_set},
\end{aligned}
\end{equation}
where $m\in\{-1, 1\}$. Finally the flow set $\mathcal{F}_{\mathcal{W}}$ and the jump set $\mathcal{J}_{\mathcal{W}}$ in \eqref{proposed_hybrid_controller_1} are defined as
\begin{align}
    \mathcal{F}_{\mathcal{W}}:=\bigcup_{m\in\mathbb{M}}(\mathcal{F}_m\times\{m\}),\; \mathcal{J}_{\mathcal{W}}:=\bigcup_{m\in\mathbb{M}}(\mathcal{J}_m\times\{m\}),\label{composite_jump_and_flow_sets}
\end{align}
with $\mathcal{F}_m, \mathcal{J}_m$ defined in \eqref{stabilization_mode_jumpflow_set_final} for $m = 0$ and in \eqref{avoidance_final_set} for $m\in\{-1, 1\}$. Next we provide the formalism for the proposed hybrid control input $\mathbf{u}(\mathbf{x}, m)$.

\begin{figure}[]
    \centering
    \includegraphics[width = 1\linewidth]{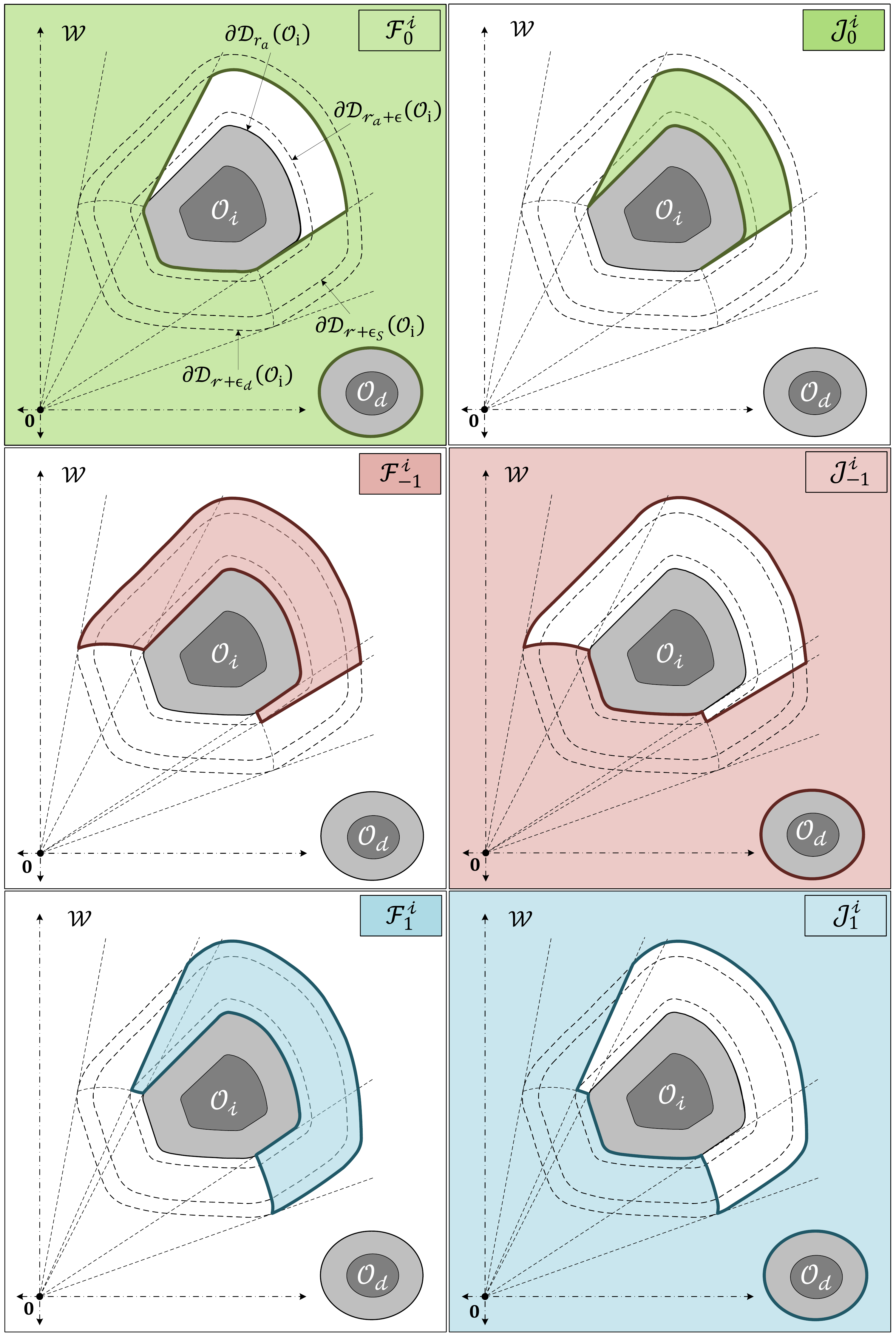}
    \caption{Geometric representations of the flow and jump sets for different modes of operation related to obstacle $\mathcal{O}_i,\; i\in\mathbb{I}$. The top figures $(m = 0)$ illustrate the case where the robot operates in the \textit{move-to-target} mode and moves straight toward the target location. The middle figures $(m = -1)$ illustrate the case where the robot, operating in the \textit{obstacle-avoidance} mode, moves in the counter-clockwise direction with respect to $\partial\mathcal{O}_i$. The bottom figures $(m = 1)$ illustrate the case where the robot, operating in the \textit{obstacle-avoidance} mode, moves in the clockwise direction with respect to $\partial\mathcal{O}_i$. }
    \label{jump_and_flow_sets}
\end{figure}

\subsection{Hybrid control input $\mathbf{u}(\mathbf{x}, m)$}
\label{sec:control_law}
The proposed hybrid control law $\mathbf{u}(\mathbf{x}, m)$ is given as
\begin{subequations}
\begin{align}
    &\mathbf{u}(\mathbf{x}, m) = -\gamma\kappa(\mathbf{x}, m)\mathbf{x} + \gamma[1 - \kappa(\mathbf{x}, m)]\mathbf{v}(\mathbf{x}, m),\label{fxim}\\
    &\hspace{2cm}\underbrace{\begin{matrix}
    \dot{m}
    \end{matrix} \begin{matrix*}[l]\;=0,\end{matrix*}}_{(\mathbf{x}, m)\in\mathcal{F}_{\mathcal{W}}}\quad\quad\underbrace{\begin{matrix}
    m^+
    \end{matrix} = \mathbf{L}(\mathbf{x}, m)}_{(\mathbf{x}, m)\in\mathcal{J}_{\mathcal{W}}},\label{update_law_control}
    \end{align}
    \label{proposed_hybrid_controller_2}
\end{subequations}where $\gamma > 0$, and $\mathbf{x}\in\mathcal{W}_{r_a}$ is the location of the center of the robot. The discrete variable $m\in\mathbb{M}$ is the mode indicator. The formalism for the update law $\mathbf{L}(\mathbf{x}, m)\in\mathbb{M}$ in \eqref{update_law_control} will be provided later in \eqref{update_law_for_m}-\eqref{direction_decision}. The flow set $\mathcal{F}_{\mathcal{W}}$ and the jump set $\mathcal{J}_{\mathcal{W}}$ are defined in \eqref{composite_jump_and_flow_sets}.  

Next we provide the construction of the vector $\mathbf{v}(\mathbf{x}, m)\in\mathbb{R}^2$ along with the scalar function $\kappa(\mathbf{x}, m)\in[0, 1]$, used in \eqref{fxim}. The vector $\mathbf{v}(\mathbf{x}, m)\in\mathbb{R}^2$ is defined as
\begin{align}
    \mathbf{v}(\mathbf{x}, m) = \begin{bmatrix}0 & m\\ -m & 0 \end{bmatrix} \frac{\mathbf{x} - \Pi(\mathbf{x}, \mathcal{O}_{\mathcal{W}})}{\norm{\mathbf{x} - \Pi(\mathbf{x}, \mathcal{O}_{\mathcal{W}})}} \norm{\mathbf{x}},\label{definition_of_vim}
\end{align}
where $\mathcal{O}_{\mathcal{W}} := \bigcup_{i\in\mathbb{I}}\mathcal{O}_i$, and $\Pi(\mathbf{x}, \mathcal{O}_{\mathcal{W}})$ is the projection of the center of the robot on the set $\mathcal{O}_{\mathcal{W}}$ in the sense of the Euclidean norm. It should be noted that, if the center of the robot is present within the $(r_a+{\epsilon}_d)-$neighbourhood of any obstacle, let us say $\mathcal{O}_k, k\in\mathbb{I}$, where ${\epsilon}_d\in(0, \bar{r}_s - r_s)$, then the projection $\Pi(\mathbf{x}, \mathcal{O}_{\mathcal{W}})$ is unique and equals $\Pi(\mathbf{x}, \mathcal{O}_k)$. In this case, the rotational vector $\mathbf{v}(\mathbf{x}, m)$ allows the robot to revolve around the obstacle $\mathcal{O}_k$. The direction of the rotation depends on the value of the mode indicator variable. When $m = 1$ the robot moves in the clockwise direction, whereas if $m = -1$ it moves in the counter-clockwise direction with respect to the boundary of the obstacle $\mathcal{O}_k$.

The scalar function $\kappa : \mathcal{W}_{r_a}\times\mathbb{M}\to[0, 1]$ defined as 
\begin{align}
    \kappa(\mathbf{x}, m) = 1 + m^2\big[\eta(\varrho(\mathbf{x})) - 1\big],\label{kappa_function_formulation}
\end{align}
allows for a continuous transition between the stabilizing vector $-\mathbf{x}$ and the rotational vector $\mathbf{v}(\mathbf{x}, m)$ whenever the robot operates in the \textit{obstacle-avoidance} mode \textit{i.e.}, $m \in\{-1, 1\}$, depending on its proximity with respect to the boundary of the set $\mathcal{O}_{\mathcal{W}}$. In \eqref{kappa_function_formulation}, the scalar function $\varrho(\mathbf{x})$ evaluates the proximity of the robot with the unsafe region $\mathcal{O}_{\mathcal{W}}$, and is given by
\begin{align}
    \varrho(\mathbf{x}) = d(\mathbf{x}, \mathcal{O}_{\mathcal{W}}) - r_a.\label{proximity_function}
\end{align}
According to \eqref{kappa_function_formulation}, whenever $m\in\{-1, 1\}$, the value of $\kappa(\mathbf{x}, m)$ equals to the value of a non-decreasing continuous scalar function $\eta :\mathbb{R}\to[0, 1]$, defined as
\begin{align}
        \eta(\varrho(\mathbf{x})) &= \begin{cases}\begin{matrix*}[l]1, & \varrho(\mathbf{x})\geq \epsilon_s,\\
     \frac{\varrho(\mathbf{x}) - \epsilon}{\epsilon_s - \epsilon},
      &\epsilon \leq \varrho(\mathbf{x}) \leq \epsilon_s,\\
      0, & \varrho(\mathbf{x})\leq \epsilon.\end{matrix*}
     \end{cases}\label{beta_function_definition}
\end{align}
The scalar function $\eta$ is constructed such that for the robot operating in the \textit{obstacle-avoidance} mode, within the ${\epsilon}_d-$neighbourhood of an obstacle, let us say $\mathcal{O}_k$, as the robot approaches the $ \epsilon-$neighbourhood of the obstacle $\mathcal{O}_k$, \textit{i.e.}, $\mathcal{D}_{r_a + \epsilon}(\mathcal{O}_k),$ the influence of the stabilizing control vector in \eqref{fxim} decreases, whereas the contribution from the rotational control vector increases. 

It is evident from \eqref{kappa_function_formulation}-\eqref{beta_function_definition}, that for the robot operating in the \textit{obstacle-avoidance} mode, the rotational part of the control law $\mathbf{u}$ in \eqref{fxim}, is active only if the robot is within the $\epsilon_s-$neighbourhood of any obstacle, which according to Assumption \ref{assumption:robot_pass_through} can only be valid for at most one obstacle at any instant of time. 

At the gate region \eqref{gate_region}, the vector $\mathbf{v}(\mathbf{x}, m)$, defined in \eqref{definition_of_vim}, becomes equivalent to the vector $-\mathbf{x}$ transforming the control input vector $\mathbf{u}$ in \eqref{fxim} to $-\gamma\mathbf{x}$, which equals to the stabilization control vector used in the \textit{move-to-target} mode. Hence, if the robot operating in the \textit{obstacle-avoidance} mode, switches to the \textit{move-to-target} mode at the gate region, then the proposed hybrid control law \eqref{proposed_hybrid_controller_2} ensures the continuity of the control input trajectories.

According to the definition of the \textit{move-to-target} mode jump set $\mathcal{J}_0\times\{0\}$, the robot operating in the \textit{move-to-target} mode in the $\epsilon_d-$neighbourhood of an obstacle $\mathcal{O}_i, i\in\mathbb{I}$, can enter in the \textit{obstacle-avoidance} mode via the boundary of the $\epsilon_s-$neighbourhood of the obstacle $\mathcal{O}_i$. In this case, the continuous switching function $\eta$ \eqref{beta_function_definition}, used in the definition of the scalar function $\kappa$ \eqref{kappa_function_formulation}, maintains the continuity of the control input trajectories.

This concludes the design of the hybrid control input $\mathbf{u}(\mathbf{x}, m)$ in \eqref{fxim}. Next, we provide the formalism for the update law $\mathbf{L}$ used in \eqref{proposed_hybrid_controller_2}.

\subsection{Mode selection map $\mathbf{L}(x,m)$}\label{sec:update_law}
The update law $\mathbf{L}(\mathbf{x}, m)$, used in \eqref{proposed_hybrid_controller_2}, allows the robot to update the value of the mode indicator variable $m\in\mathbb{M}$ when the state $(\mathbf{x}, m)$ belongs to the jump set $\mathcal{J}$ defined in \eqref{composite_jump_and_flow_sets}, is given as 
\begin{equation}
    \mathbf{L}(\mathbf{x}, m) = \begin{cases}\begin{matrix*}[l]0,&(\mathbf{x}, m)\in\mathcal{J}_m\times\{m\}, m\in\{-1, 1\},\\
    \mathbf{M}(\mathbf{x}),&(\mathbf{x}, m)\in\mathcal{J}_{0}\times\{0\}.\end{matrix*}\end{cases}\label{update_law_for_m}
\end{equation}
When the robot enters in the jump set of the \textit{obstacle-avoidance} mode, the value of the mode indicator variable switches to $0$. The mapping $\mathbf{M}$, which is based on the current location of the center of the robot $\mathbf{x}$ with respect to the hyperplane $\mathcal{P}(\mathbf{0}, \mathbf{s})$ is defined as
\begin{align}
    \mathbf{M}(\mathbf{x}) := \begin{cases}\begin{matrix*}[l]\{-1\}, & \mathbf{x}^\intercal\mathbf{s} < 0,\\
    \{-1, 1\}, & \mathbf{x}^\intercal\mathbf{s} = 0,\\
    \{1\}, &\mathbf{x}^\intercal\mathbf{s} > 0,\end{matrix*}\end{cases}\label{direction_decision}
\end{align}
where $\mathbf{s} \in\mathbb{R}^2\backslash\{\mathbf{0}\}$ is an arbitrary non-zero constant vector. The switching strategy in \eqref{direction_decision}, allows the robot to choose between the clockwise and counter-clockwise motions based on its current location whenever the state $(\mathbf{x}, m)$ belongs to the jump set of the \textit{move-to-target} mode, $\mathcal{J}_0\times\{0\}$. This strategy is crucial to establish global asymptotic convergence of the robot towards the target location, irrespective of the arrangements of the obstacles as it is going to be stated later in Theorem \ref{theorem:global_asymptotic_convergence}. 

It is shown that the robot governed by the proposed hybrid controller, belonging to the region $\mathcal{P}_{>}(\mathbf{0}, \mathbf{s})$ or $\mathcal{P}_{<}(\mathbf{0}, \mathbf{s})$, away from the boundary of the unsafe region at some hybrid time instant $(t_0, j_0)$, with $m(t_0, j_0) = 0$, cannot cross or intersect the half-line $\mathcal{L}_{>}(\mathbf{0}, \nu_{-1}(\mathbf{s}))$ for all $\forall (t, j)\succeq(t_0, j_0)$, which ensures that, while operating in environments satisfying Assumption \ref{assumption:robot_pass_through}, the robot can never revolve around the target location. Moreover, the robot cannot move indefinitely in either $\mathcal{P}_{>}(\mathbf{0}, \mathbf{s})$ or $\mathcal{P}_{<}(\mathbf{0}, \mathbf{s})$ \textit{i.e.}, during the motion it will either directly converge to the origin or intersect the half-line $\mathcal{L}_{>}(\mathbf{0}, \nu_1(\mathbf{s}))$, and with every consecutive intersection with this half-line, the robot gets closer to the origin. This concludes the definition of the update law $\mathbf{L}$ in \eqref{proposed_hybrid_controller_1} and the design of the proposed hybrid controller.

\section{Forward Invariance and Stability analysis}\label{sec:stability}
As noted earlier, the robot operating with the proposed hybrid control law avoids one obstacle at a time. In general, while moving in the workspace, the robot might encounter more than one obstacles. To model this change of obstacle to be avoided, we introduce a discrete variable $k\in\mathbb{I}$, which corresponds to the index of the obstacle being avoided by the robot while operating in the \textit{obstacle-avoidance} mode. The hybrid evolution of the state $k$ is given as
\begin{equation}
    \underbrace{\begin{matrix}
    \dot{k}
    \end{matrix} \begin{matrix*}[l]\;=0,\end{matrix*}}_{\xi\in\mathcal{F}}\quad\quad\underbrace{\begin{matrix}
    k^+
    \end{matrix} = \mathbf{N}(\xi)}_{\xi\in\mathcal{J}},\label{hybrid_dynamic_of_state_k}
\end{equation}
where 
\begin{equation}\xi:= (\mathbf{x}, m, k)\in\mathcal{W}_{r_a}\times\mathbb{M}\times\mathbb{I}=:\mathcal{K},\label{composite_state_vector}\end{equation} 
is the composite state vector. The flow set $\mathcal{F}$ and the jump set $\mathcal{J}$ are defined as
\begin{equation}
    \mathcal{F} = \mathcal{F}_{\mathcal{W}}\times\mathbb{I},\quad\quad\mathcal{J} = \mathcal{J}_{\mathcal{W}}\times\mathbb{I}.\label{overall_flow_set_jump_set}
\end{equation}
The update law for the state $k\in\mathbb{I}$ \textit{i.e.}, $\mathbf{N}(\xi)$ is designed such that whenever the robot encounters the jump set of the \textit{move-to-target} mode, the value of the variable $k$ is changed to the index of the closest obstacle, otherwise it is kept unchanged. Hence, $\mathbf{N}(\xi)$ is given as
\begin{equation}
    \mathbf{N}(\xi) = \begin{cases}\begin{matrix*}[l]k^{\prime},&\xi\in\mathcal{J}_0^{k^{\prime}}\times\{0\}\times\mathbb{I},\\
    k,&\xi\in\mathcal{J}_{m}\times\{m\}\times\mathbb{I}, m\in\{-1, 1\}.\end{matrix*}\end{cases}\label{update_law_for_k}
\end{equation}

The hybrid closed-loop system, resulting from the control law \eqref{proposed_hybrid_controller_2} and auxiliary state hybrid dynamics \eqref{hybrid_dynamic_of_state_k}, is given by
\begin{equation}
    \underbrace{\begin{matrix}
    \mathbf{\dot{x}}\\
    \dot{m}\\
    \dot{k}
    \end{matrix} 
    \begin{matrix*}[l]
    \;= \mathbf{u}(\xi)\\
    \;= 0\\
    \;= 0
    \end{matrix*}
    }_{\dot{\xi} = \mathbf{F}(\xi)},\;\xi\in\mathcal{F},\quad\quad\underbrace{\begin{matrix*}
    \mathbf{x}^+ \\m^+\\
    k^+\end{matrix*}\begin{matrix*}[l] \;=\mathbf{x}\\\;\in\mathbf{L}(\xi)\\\;\in \mathbf{N}(\xi)
    \end{matrix*}}_{{\xi}^+ = \mathbf{J}(\xi)},\;\xi\in\mathcal{J},\label{hybrid_closed_loop_system}
\end{equation}
where $\mathbf{u}(\xi)$ is defined in \eqref{fxim}, and the update laws $\mathbf{L}(\xi)$ and $\mathbf{N}(\xi)$ are provided in \eqref{update_law_for_m}-\eqref{direction_decision} and \eqref{update_law_for_k}, respectively. The definitions of the flow set $\mathcal{F}$ and the jump set $\mathcal{J}$ are provided in \eqref{composite_jump_and_flow_sets}, \eqref{overall_flow_set_jump_set}. In the next section, we analyze the hybrid closed-loop system \eqref{hybrid_closed_loop_system} in terms of the forward invariance of the obstacle-free state space $\mathcal{K}$ along with the stability properties of the target set $\mathcal{A}$, which is defined as
\begin{equation}
\mathcal{A}:= \{\mathbf{0}\}\times\mathbb{M}\times\mathbb{I}.\label{hybrid_target_set}
\end{equation}

First, we analyze the forward invariance of the obstacle-free workspace, which then will be followed by the convergence analysis.

For safe autonomous navigation, the robot should belong to the obstacle-free workspace $\mathcal{W}_{0}$ for all time \textit{i.e.}, the state $\mathbf{x}$ must always evolve within the robot-centered free workspace $\mathcal{W}_{r_a}$, regardless of the variables $k$ and $m$. This is equivalent to showing that the set $\mathcal{K}$, defined in \eqref{composite_state_vector}, is forward invariant with respect to the hybrid closed-loop system \eqref{hybrid_closed_loop_system}. This is stated in the next Lemma.
\begin{lemma}[Safety]
Under Assumption \ref{assumption:robot_pass_through}, for the obstacle-free set $\mathcal{K}$, defined in \eqref{composite_state_vector}, and the hybrid closed-loop system \eqref{hybrid_closed_loop_system}, the set $\mathcal{K}$ is forward invariant.
\label{forward_invariance_theorem_1}
\end{lemma}
\begin{proof}
See Appendix \ref{proof:lemma:normal_to_dilated_obstacle}.
\end{proof}

The proposed switching strategy between the \textit{move-to-target} mode and the \textit{obstacle-avoidance} mode is similar to the strategies used in the sensor-based path planning algorithms for a point robot, referred to as bug algorithms \cite{ng2007performance}. For some special obstacle arrangements, when the robot instead of converging to the predefined target location, retraces the previously followed path, the bug algorithms terminate the path planning process establishing the failure to converge to the target location due to the presence of a closed trajectory around the target location. As discussed earlier, in the research works \cite{matveev2011method}, \cite{berkane2021obstacle}, the authors imposed restrictions on the inter-obstacle arrangements to avoid closed trajectories around the target location. In the present paper, non-existence of closed trajectories, around the target location, is guaranteed by design without imposing any restrictions on the inter-obstacle arrangements except for the ones stated in Assumption \ref{assumption:robot_pass_through}. The next lemma shows that the robot operating with the proposed hybrid controller \eqref{proposed_hybrid_controller_2}, in environments satisfying Assumption \ref{assumption:robot_pass_through}, does not get stuck in any closed trajectory around the target location.
\begin{lemma}
Consider the hybrid closed-loop system \eqref{hybrid_closed_loop_system} and let Assumption \ref{assumption:robot_pass_through} hold. If $\xi(t_0, j_0)\in\mathcal{W}_{r_a}\backslash\mathcal{L}_{>}(\mathbf{0}, \nu_{-1}(\mathbf{s}))\times\{0\}\times\mathbb{I}$ at some $(t_0, j_0)\in\text{ dom }\xi$, then $\xi(t, j)\notin\mathcal{L}_{>}(\mathbf{0}, \nu_{-1}(\mathbf{s}))\times\mathbb{M}\times\mathbb{I}$ for all $(t, j)\succeq(t_0, j_0)$.\label{no_revolution_around_the_target}
\end{lemma}
\begin{proof}
See Appendix \ref{proof:no_closed_trajectories}.
\end{proof}

Lemma \ref{no_revolution_around_the_target} shows that if the robot, operating in the obstacle-free workspace $\mathcal{W}_{r_a}$, with the \textit{move-to-target} mode, does not belong to the half-line $\mathcal{L}_{>}(\mathbf{0}, \nu_{-1}(\mathbf{s}))$, see Fig. \ref{final_image}, at some time $(t_0, j_0)$, then it can never intersect the half-line $\mathcal{L}_{>}(\mathbf{0}, \nu_{-1}(\mathbf{s}))$ for all $(t, j)\succeq(t_0, j_0)$. An example is provided in Fig. \ref{final_image}, showing that the robot's trajectory, initialized at $\mathbf{x}_2$, does not intersect the half-line $\mathcal{L}_{>}(\mathbf{0}, \nu_{-1}(\mathbf{s}))$ represented in red. The main reason behind this behaviour is the switching strategy \eqref{direction_decision} for the mode indicator variable when the solution enters in the \textit{obstacle-avoidance} mode, which assigns the direction of motion that always steers the robot away from the half-line $\mathcal{L}_{>}(\mathbf{0}, \nu_{-1}(\mathbf{s}))$. This feature ensures that the robot cannot revolve around the target location. Next, we provide one of our main results which establishes the fact that for all initial conditions in the obstacle-free set $\mathcal{K}$, the proposed hybrid controller not only ensures safe navigation but also guarantees global asymptotic convergence to the predefined target location at the origin. 

We define the Lebesgue measure zero set $\mathcal{Z}_0:=\mathcal{M}_0\times\mathbb{M}\times\mathbb{I}$ such that
\begin{equation}
    \mathcal{M}_{0} := \bigcup_{i\in\mathbb{I}}(\partial\mathcal{D}_{r_a}(\mathcal{O}_i)\cap\mathcal{J}_0^i),\label{initial_condition_for_zeno_behaviour}
\end{equation} 
is the intersection of the boundaries of the unsafe region and the \textit{move-to-target} mode jump set. The intersection of the \textit{move-to-target} mode jump set $\mathcal{J}_0\times\{0\}\times\mathbb{I}$ and the \textit{obstacle-avoidance} mode jump set $\mathcal{J}_m\times\{m\}\times\mathbb{I}, m\in\{-1, 1\},$ is not empty, for $\mathbf{x}\in\mathcal{M} := \bigcup_{i\in\mathbb{I}}\{\mathbf{n}_{r_a}^{(-1, i)}, \mathbf{n}_{r_a}^{(1, i)}\}$, where
\begin{equation}
\begin{aligned}
    \mathbf{n}_{y}^{(m, i)} &= \underset{\mathbf{q}\in\partial\mathcal{D}_y(\mathcal{O}_i)\cap\mathcal{G}_{m}^i}{\text{arg max }}\norm{\mathbf{q}}.
    \end{aligned}\label{point_intersection_modified}
\end{equation}
Note that $\mathbf{n}_{y}^{(m, i)}$ is the farthest point from the origin, which belongs to the boundary of the dilated obstacle $\mathcal{D}_{y}(\mathcal{O}_i)$ such that $\alpha_s(\mathbf{q},\mathbf{q} - \Pi(\mathbf{q}, \mathcal{O}_i))$ is $-m\pi/2$. For example, see points $\mathbf{n}_{r_a}^{(m, i)}, m\in\{-1, 1\}$, depicted in Fig. \ref{Partitions_of_local_neighbourhood}. As stated next in Theorem \ref{theorem:global_asymptotic_convergence}, the solution will reach the Zeno set $\mathcal{M}\times\mathbb{M}\times\mathbb{I} =:\mathcal{Z}$ when it is initialized in the Lebesgue measure zero set $\mathcal{Z}_0$.
\begin{theorem}
Consider the hybrid closed-loop system \eqref{hybrid_closed_loop_system} and let Assumption \ref{assumption:robot_pass_through} hold. Then,
\begin{itemize}
    \item[$i)$] the obstacle-free set $\mathcal{K}$ is forward invariant,
    \item[$ii)$]  the set $\mathcal{A}$ is almost globally asymptotically stable,
    \item[$iii)$]  the solutions will converge to $\mathcal{A}$ from all initial conditions except from the set of Lebesgue measure zero $\mathcal{Z}_0$ where the solutions may stay jumping in $\mathcal{Z}$ (Zeno behavior),
    \item[$iv)$]  if the flows are  prioritized (forced) over the jumps then the set $\mathcal{A}$ is globally asymptotically stable and the solution is Zeno-free. 
\end{itemize}
\label{theorem:global_asymptotic_convergence}
\end{theorem}

\begin{proof}
See Appendix \ref{proof:the_main_theorem}.
\end{proof}

It is worth pointing out that the almost global stability result established in Theorem \ref{theorem:global_asymptotic_convergence} is not due to the existence of undesired saddle points as in \cite{koditschek1990robot}, but due to the potential existence of a Zeno behaviour \cite{goedel2012hybrid}. In fact, if the robot is initialized in  $\mathcal{M}_0$, which consists of all the inner boundaries of the front regions, it may not converge to the target and instead will converge to some isolated points where it will experience a Zeno behaviour by switching indefinitely between the two modes of operation. This behaviour is obtained only if the jumps are  prioritized over the flows in the implementation.

\begin{remark}
\label{dwell_time_remark}
When prioritizing flows over jumps, the useful structural and robustness properties of the set of solutions guaranteed by the hybrid basic conditions \cite[Assumption 6.5]{goedel2012hybrid} may not be present. Another practical way (different from prioritizing the flows over jumps), that helps in avoiding the Zeno behaviour, consists in introducing a small enough time period $\tau > 0$ (dwell time) between consecutive jumps. This will force the solution $\xi$, after switching once, to leave the set $\mathcal{M}$ through the flow. The hybrid basic conditions are in this case preserved.

\end{remark}
\begin{remark}[Continuous input]
The continuous scalar function $\kappa$ in \eqref{kappa_function_formulation}, will help to guarantee the continuity of the control input when the mode transitions between the \textit{obstacle-avoidance} mode and the \textit{move-to-target} mode at the gate region. This interesting feature of the proposed hybrid control law \eqref{proposed_hybrid_controller_2} is formalized in Proposition \ref{proposition:continuous_control}.
\end{remark}

\begin{proposition}
 Consider the closed-loop system \eqref{single_integrator_control_law} and let Assumption \ref{assumption:robot_pass_through} hold. If $\xi(t_0, j_0)\in\mathcal{K}\backslash\mathcal{Z}_0$\label{proposition:continuous_control}, for some $(t_0, j_0)\in\text{dom }\xi$, then the control input trajectories $\mathbf{u}(\xi(t,j))$, generated according to  \eqref{proposed_hybrid_controller_2}, are continuous $\forall(t, j)\succeq(t_0, j_0)$.
\end{proposition}
\begin{proof}
See Appendix \ref{proof:proposition_continuous}.
\end{proof}

According to Proposition \ref{proposition:continuous_control}, for the solution which belongs to the set $\mathcal{K}\backslash\mathcal{Z}_0$ at some time $(t_0, j_0)\in\text{dom }\xi$ during the evolution, the proposed hybrid feedback control law \eqref{proposed_hybrid_controller_2} will generate continuous control input trajectories for all $(t, j)\succeq(t_0, j_0)$. 
This concludes the stability analysis of the hybrid closed-loop system \eqref{hybrid_closed_loop_system}. Next, we provide procedural steps to implement the proposed hybrid feedback controller \eqref{proposed_hybrid_controller_2} for safe autonomous navigation of a mobile robot operating in an \textit{a priori} known and unknown environments.

\section{Sensor-based implementation procedure}\label{sensor-based-implementation}

Without loss of generality, we assume that the target location is at the origin, and we set $m(0, 0) = 0$. The non-zero vector $\mathbf{s}$, used in \eqref{direction_decision}, can be selected such that $\mathbf{x}(0, 0)\in\mathcal{L}_{>}(\nu_1(\mathbf{s}))$.
Then the robot can implement the proposed hybrid control law \eqref{proposed_hybrid_controller_2}, using Algorithm \ref{alg:sensor_based_implementation}. Since the parameters $\epsilon_d > \epsilon_s > \epsilon > 0$ can be tuned offline, Algorithm \ref{alg:sensor_based_implementation} should be implemented excluding the steps highlighted in blue color. The blue colored steps are essential for the sensor-based implementation in an \textit{a priori} unknown environment.

\begin{algorithm}
\caption{General implementation of the proposed hybrid control law \eqref{proposed_hybrid_controller_2}.}
\begin{algorithmic}[1] 
\STATE\textbf{Set} target location at the origin $\mathbf{0}$.
\STATE\textbf{Initialize} $\mathbf{x}(0, 0)\in\mathcal{W}_{r_a}$, $m(0, 0) = 0$, $\epsilon_d\in(0,\bar{r}_s - r_s)$. Pick $\mathbf{s}\in\mathbb{R}^2\backslash\{\mathbf{0}\}$, used in \eqref{direction_decision}, such that $\mathbf{x}(0, 0)\in\mathcal{L}_{\geq}(\mathbf{0}, \nu_{1}(\mathbf{s}))$.
\STATE\textbf{Measure} $\mathbf{x}$.
\IF{$m= 0$,}
\begin{color}{blue}\STATE\textbf{Implement} Algorithm \ref{alg:jumpset_participation}.\end{color}
\IF{$(\mathbf{x}, m)\in\mathcal{J}_0\times\{0\}$,}
\begin{color}{blue}\STATE \textbf{Set} $\epsilon_s = d(\mathbf{x}, \mathcal{O}_{\mathcal{W}}) - r_a$.
\STATE \textbf{Select} $\epsilon = (0, \epsilon_s).$\end{color}
\STATE\textbf{Update} $m\leftarrow{\mathbf{L}(\mathbf{x}, m)}$ using \eqref{update_law_for_m}, \eqref{direction_decision}.
\ENDIF
\ENDIF
\IF{$m \in\{-1, 1\}$,}
\begin{color}{blue}\STATE\textbf{Implement} Algorithm \ref{alg:jumpset_participation}.\end{color}
\IF{$(\mathbf{x}, m)\in\mathcal{J}_{m}\times\{m\}$,}
\STATE\textbf{Update} $ m\leftarrow\mathbf{L}(\mathbf{x}, m)$ using \eqref{update_law_for_m}, \eqref{direction_decision}. 
\ENDIF
\ENDIF
\STATE\textbf{Execute} $\mathbf{u}(\mathbf{x}, m)$ \eqref{fxim}, used in \eqref{hybrid_closed_loop_system}.
\STATE \textbf{Go to} step 3.
\end{algorithmic}
\label{alg:sensor_based_implementation}
\end{algorithm}

In the case where the environment is \textit{a priori} unknown, we choose a sufficiently small value of the parameter $\epsilon_d\in(0, \bar{r}_s - r_s)$, which ensures that the target location is at a distance greater than $r_a + \epsilon_d$ from the unsafe region. We assume that the robot is equipped with a range-bearing sensor with angular scanning range of $360^{\circ}$ and sensing radius $R_s> r_a + \epsilon_d$. Due to the limited sensing radius, the robot can only detect a subset of the obstacles $\mathbb{I}_{\mathbf{x}}\subseteq\mathbb{I}$, defined as
\begin{equation}
    \mathbb{I}_{\mathbf{x}} = \{i\in\mathbb{I}|d(\mathbf{x}, \mathcal{O}_i)\leq R_s\}.
\end{equation}

Based on the local sensing information, we provide a procedure that allows to identify whether the state $(\mathbf{x}, m)$ belongs to the jump set $\mathcal{J}_{\mathcal{W}}$ \eqref{composite_jump_and_flow_sets} or not, when $m(0, 0) = 0$, which is summarized in Algorithm 2. Similar to \cite{arslan2019sensor,berkane2021navigation}, the range-bearing sensor is modeled using a polar curve $r_g(\mathbf{x}, \theta):\mathcal{W}_{r_a}\times[-\pi, \pi]\to[0, R_s]$,

\begin{equation}
    r_g(\mathbf{x}, \theta) = \min\left\{R_s, \underset{\begin{matrix}\mathbf{y}\in\partial\mathcal{O}_{\mathcal{W}}\\\text{atan2}(\mathbf{y}-\mathbf{x}) = \theta\end{matrix}}{\min}\norm{\mathbf{x} - \mathbf{y}}\right\},
\end{equation}
which represents the distance between the center of the robot and the boundary of the unsafe region $\partial\mathcal{O}_{\mathcal{W}}$, measured by the sensor, in the direction defined by the angle $\theta$. Given the location of the center of the robot $\mathbf{x}$, along with the bearing angle $\theta$, the mapping $\lambda(\mathbf{x}, \theta):\mathcal{W}_{r_a}\times[-\pi, \pi]\to\mathcal{W}_{0}$, given by
\begin{equation}
    \lambda(\mathbf{x}, \theta) = \mathbf{x} + r_g(\mathbf{x}, \theta)[\cos\theta, \sin\theta]^\intercal,
\end{equation}
evaluates the Cartesian coordinates of the detected point. 

The robot should identify the minimum distance from the set $\mathcal{O}_{\mathcal{W}}$,
\begin{equation}
    d(\mathbf{x}, \mathcal{O}_{\mathcal{W}}) = \underset{\theta\in[-\pi, \pi]}{\min}r_g( \mathbf{x}, \theta).\label{sensor:distance_from_obstacles}
\end{equation}
If $d(\mathbf{x}, \mathcal{O}_{\mathcal{W}})$ is greater than or equal to $r_a + \epsilon_d$, then, according to \eqref{jumpset_individual_obstacle}, \eqref{stabilization_mode_jumpflow_set_final} and \eqref{composite_jump_and_flow_sets}, the state $(\mathbf{x}, m)\notin\mathcal{J}_{0}\times\{0\},$ \textit{i.e.}, the robot should continue to operate in the \text{move-to-target} mode. On the other hand, if $d(\mathbf{x}, \mathcal{O}_{\mathcal{W}}) \in [r_a, r_a +\epsilon_d]$, which, according to Assumption \ref{assumption:robot_pass_through}, can be true for only one obstacle, let us say $k\in\mathbb{I}_{\mathbf{x}}$, then the robot should identify whether the state $\mathbf{x}$ belongs to the back region of the respective obstacle $\mathcal{R}_b^k$ or not. The robot should locate the projection of its center onto the obstacle $\mathcal{O}_k$ \textit{i.e.}, $\Pi(\mathbf{x}, \mathcal{O}_k)$, which is unique and given by
\begin{equation}
    \Pi(\mathbf{x}, \mathcal{O}_k) = \lambda(\mathbf{x}, \theta^*),\label{projection_using_sensor}
\end{equation}
where
\begin{equation}
    \theta^* = \underset{\theta\in[-\pi, \pi]}{\text{argmin}}r_g(\mathbf{x}, \theta).\label{sensor:closest_theta}
\end{equation}
Then, the robot should verify the following condition:
\begin{equation}
    \mathbf{x}^\intercal(\mathbf{x} - \Pi(\mathbf{x}, \mathcal{O}_k)) \leq 0.\label{back_region_condition}
\end{equation}
The satisfaction of \eqref{back_region_condition} implies that the state $\mathbf{x}$ belongs to the back region of the obstacle $\mathcal{O}_k$, and, according to \eqref{jumpset_individual_obstacle}, \eqref{stabilization_mode_jumpflow_set_final} and \eqref{composite_jump_and_flow_sets}, $(\mathbf{x}, m)\notin\mathcal{J}_{0}\times\{0\}$ implying that the robot can continue to operate in the \textit{move-to-target} mode. Otherwise, if \eqref{back_region_condition} is not satisfied, the robot should further investigate the possibility of collision while operating in the \textit{move-to-target} mode.

\begin{figure}[h]
    \centering
    \includegraphics[width = 0.435\linewidth]{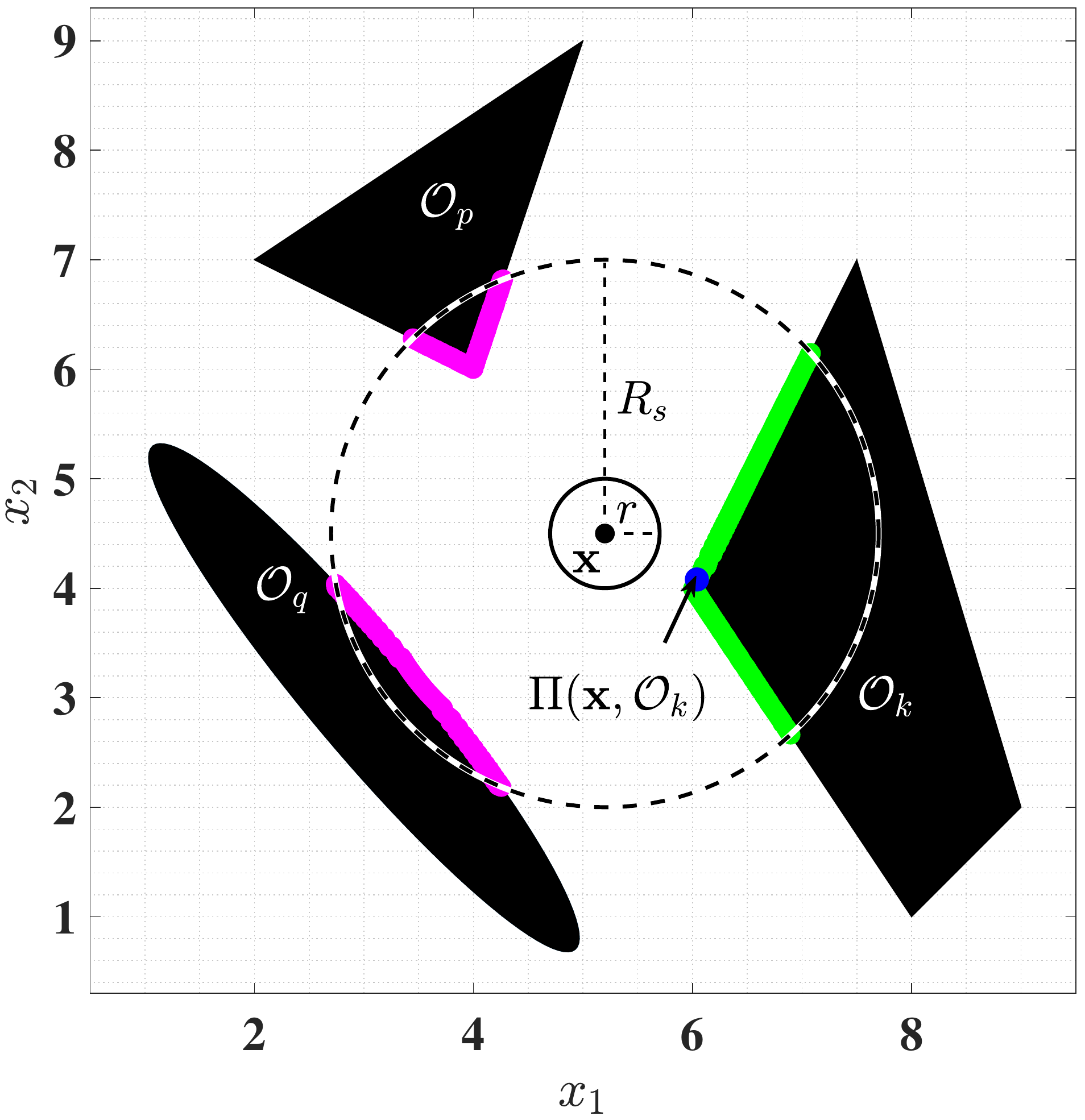}
    \includegraphics[width = 0.54\linewidth]{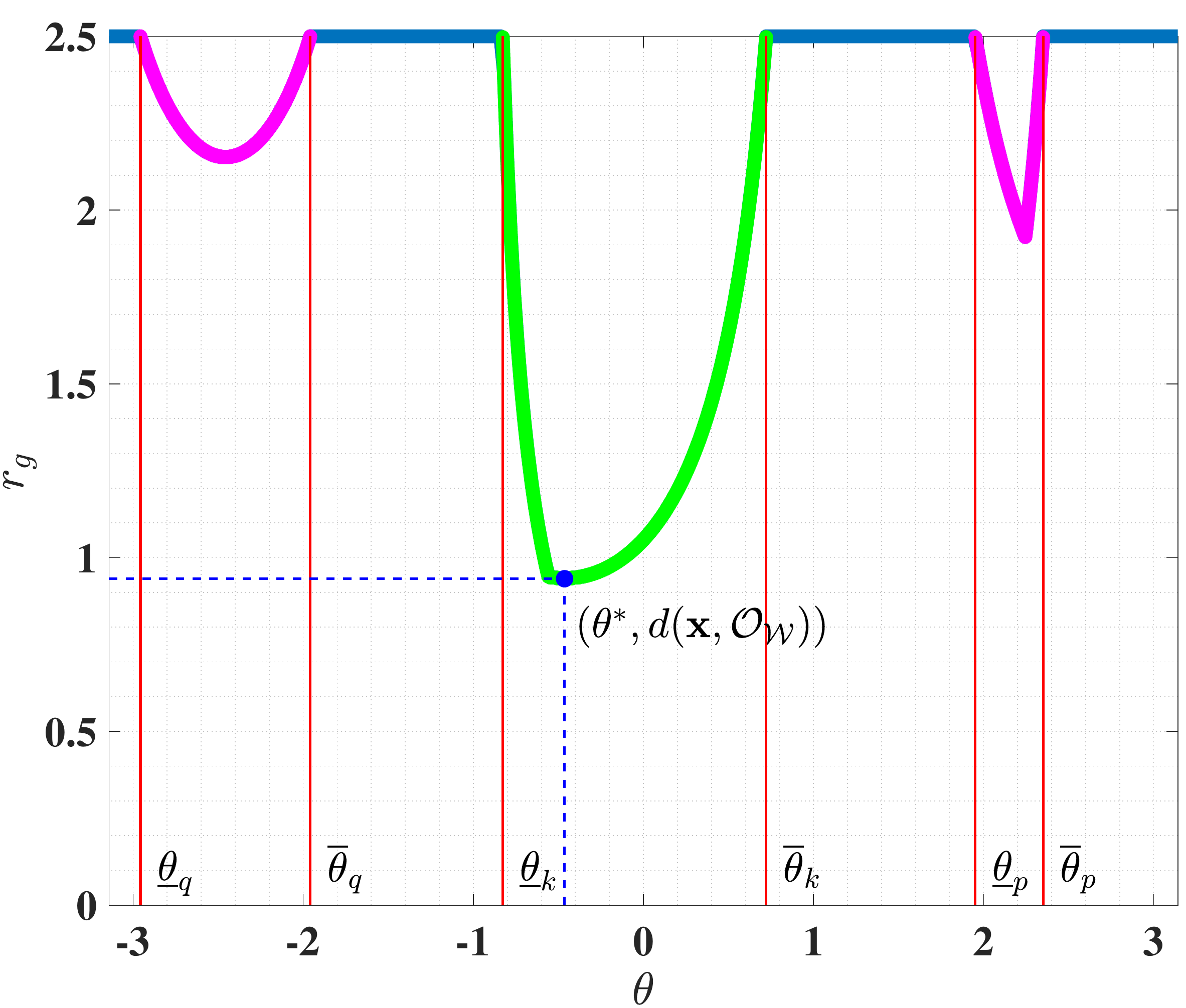}
    \caption{The left figure shows a robot, with radius $r = 0.5m$, using a range-bearing sensor to locate the partial boundary of the obstacles within the sensing range $R_s = 2.5m$, wherein observed boundary of the nearest obstacle is shown in green and the observed boundaries of the remaining obstacles are represented by the blue colored curves. The right figure displays the range measurements obtained with a sensor located at $\mathbf{x}$ for $\theta\in[-\pi, \pi].$ Similar colors have been used to represent the correlation between the points on the boundary of the observed obstacles (left figure) and the rang/bearing measurements (right figure).}
    \label{sensor_environment} 
\end{figure}
Next, the robot should identify the boundary curve $\eth\mathcal{O}_k\subset\partial\mathcal{O}_k$, which is a set of points which belongs to the boundary of the obstacle $\mathcal{O}_k$ and are in the line-of-sight of the center of the robot. Figure \ref{sensor_environment} illustrates the measurements obtained via a range-bearing sensor when the robot is located in the obstacle-free space. Since the obstacles are disjoint with a minimum separation greater than $2r$ as per Assumption \ref{assumption:robot_pass_through}, the range-bearing measurement graph, shown in Fig. \ref{sensor_environment}, consists of convex curves, one for each of the obstacles present within the sensing region $\mathcal{B}_{R_s}(\mathbf{x})$.

For each obstacle $\mathcal{O}_i, i\in\mathbb{I}_{\mathbf{x}}$, the robot can identify $\underline{\theta}_i,\overline{\theta}_i\in[-\pi, \pi]$ such that the measurements related to the obstacle $\mathcal{O}_i, i\in\mathbb{I}_{\mathbf{x}},$ lie within the angular range of $\left[\underline{\theta}_i, \overline{\theta}_i\right]$, as shown in Fig. \ref{sensor_environment}. Since the measurements are acquired in the line-of-sight format, there cannot be any overlap between the angular intervals related to any two obstacles \text{i.e.}, $\left[\underline{\theta}_i, \overline{\theta}_i\right] \cap \left[\underline{\theta}_j, \overline{\theta}_j\right] = \emptyset, i, j\in\mathbb{I}_{\mathbf{x}}, i\ne j$. The robot should then identify $\left[\underline{\theta}_k, \overline{\theta}_k\right]$ for the obstacle $\mathcal{O}_k$ where
\begin{equation}
    k = \left\{i\in\mathbb{I}_{\mathbf{x}}|\theta^*\in\left[\underline{\theta}_i,\overline{\theta}_i\right]\right\},\label{sensor:closest_obstacle}
\end{equation}
then the set $\eth\mathcal{O}_k$ can be defined as
\begin{equation}
    \eth\mathcal{O}_k = \left\{\lambda(\mathbf{x}, \theta)\in\mathcal{W}_0| \theta \in \left[\underline{\theta}_k, \overline{\theta}_k\right]\right\}.\label{partial_boundary_of_obstacle}
\end{equation}

\begin{figure}[h]
    \centering
    \includegraphics[width = 0.6\linewidth]{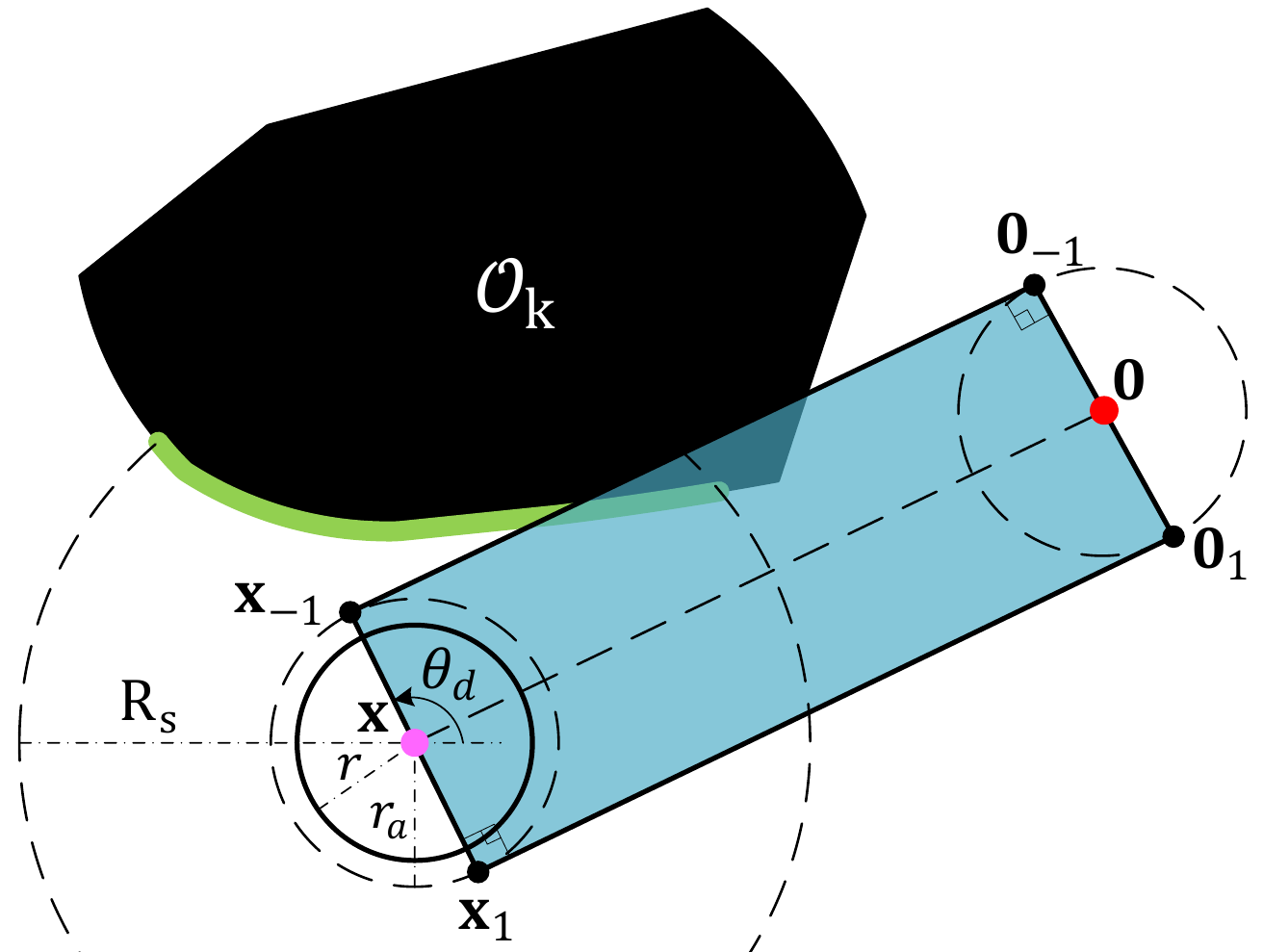}
    \caption{Construction of the rectangle $\Box(\mathbf{x})$, used in \eqref{final_condition}, based on the location of the robot and the target location at the origin.}
    \label{sensor_based_jumpset_identification_figure}
\end{figure}

The robot should then construct a rectangle $\Box(\mathbf{x})$ with its vertices located at the points $\mathbf{x}_{-1}, \mathbf{x}_1, \mathbf{0}_{-1}$ and $\mathbf{0}_1$, as shown in Fig. \ref{sensor_based_jumpset_identification_figure}, defined as
\begin{equation}
    \begin{aligned}
    \begin{bmatrix}
    \mathbf{x}_z\\
    \mathbf{0}_z
    \end{bmatrix} = \begin{bmatrix}
    \mathbf{x}\\
    \mathbf{0}
    \end{bmatrix} + zr_a\begin{bmatrix}\mathbf{I}\\\mathbf{I}\end{bmatrix}\begin{bmatrix}\cos\theta_d\\\sin\theta_d\end{bmatrix}, z\in\{-1, 1\},
    \end{aligned}\label{vertices_of_rectangle}
\end{equation}
where $\mathbf{I}$ is a $2\times2$ identity matrix, and $\theta_d = \pi/2 + \text{atan2}(\mathbf{x})$. It is straightforward to notice that, the robot can continue to navigate in the \textit{move-to-target} mode, safely with respect to partial obstacle boundary $\eth\mathcal{O}_k$, if and only if the following condition holds true:
\begin{equation}
    \eth\mathcal{O}_k\cap\Box(\mathbf{x}) = \emptyset.\label{final_condition}
\end{equation}
If above condition is not satisfied at some $(t, j)$, as illustrated in Fig. \ref{sensor_based_jumpset_identification_figure}, then the robot can conclude that continuing to move in the \textit{move-to-target} mode will result in the collision with obstacle $\mathcal{O}_k$ \textit{i.e.}, $(\mathbf{x}(t, j), m(t, j))\in\mathcal{J}_0\times\{0\}$ and that now it should operate in the \textit{obstacle-avoidance} mode. At this instant, the robot should set $\epsilon_s = d(\mathbf{x}, \mathcal{O}_{\mathcal{W}}) - r_a$ and $\epsilon = p\epsilon_s, p\in(0, 1).$ Finally, as the robot starts operating in the \textit{obstacle-avoidance} mode, it should continuously verify \eqref{back_region_condition} to identify whether it has entered in the back region of the obstacle $\mathcal{O}_k$. Satisfaction of \eqref{back_region_condition}, according to \eqref{jumpset_individual_obstacle}, \eqref{avoidance_final_set} and \eqref{composite_jump_and_flow_sets}, implies that the state $(\mathbf{x}, m)$ belongs to the jump set of the \textit{obstacle-avoidance} mode.

\begin{remark}
The robot operating with the proposed hybrid control law \eqref{proposed_hybrid_controller_2}, in the $\epsilon_d-$neighbourhood of any obstacle, let us say $\mathcal{O}_k, k\in\mathbb{I}_{\mathbf{x}}$ \textit{i.e.}, $d(\mathbf{x} ,\mathcal{O}_{k}) \in[r_a, r_a + \epsilon_d]$, needs to identify the partial boundary curve $\eth\mathcal{O}_k$ only to identify whether the state $(\mathbf{x}, m)$ belongs to the jump set of the \textit{move-to-target} mode or not. Otherwise, the proposed hybrid control law \eqref{proposed_hybrid_controller_2} only requires the state $(\mathbf{x}, m)$ and the projection of the $\mathbf{x}$ component of the state on the nearest obstacle.\label{remark:information_required}.
\end{remark}

\begin{algorithm}
\caption{Sensor-based identification of the jump set.}

\begin{algorithmic}[1] 
\STATE\textbf{Measure} $d(\mathbf{x}, \mathcal{O}_{\mathcal{W}})$ defined in \eqref{sensor:distance_from_obstacles}.
\IF{$m= 0$,}
\IF{$d(\mathbf{x}, \mathcal{O}_{\mathcal{W}}) \leq r_a + \epsilon_d,$}
\STATE\textbf{Identify} $\mathcal{O}_k, k\in\mathbb{I}$ using \eqref{sensor:closest_theta}, \eqref{sensor:closest_obstacle}.
\STATE\textbf{Locate} $\Pi(\mathbf{x}, \mathcal{O}_k)$ defined in \eqref{projection_using_sensor}.
\IF{$\mathbf{x}^\intercal(\mathbf{x} -\Pi(\mathbf{x}, \mathcal{O}_k)) > 0$, see \eqref{back_region_condition},}
\STATE\textbf{Identify} $\eth\mathcal{O}_k$ using \eqref{partial_boundary_of_obstacle}.
\STATE\textbf{Construct} $\Box(\mathbf{x})$ using \eqref{vertices_of_rectangle}.
\IF{$\eth\mathcal{O}_k\cap\Box(\mathbf{x}) \ne \emptyset$, see \eqref{final_condition},}
\STATE $(\mathbf{x}, m)\in\mathcal{J}_0\times\{0\}$.
\ELSE
\STATE$(\mathbf{x}, m)\notin\mathcal{J}_0\times\{0\}$.
\ENDIF
\ELSE
\STATE$(\mathbf{x}, m)\notin\mathcal{J}_0\times\{0\}$.
\ENDIF
\ELSE
\STATE$(\mathbf{x}, m)\notin\mathcal{J}_0\times\{0\}$.
\ENDIF
\ENDIF
\IF{$m \in\{-1, 1\}$,}
\IF{$d(\mathbf{x}, \mathcal{O}_{\mathcal{W}})\leq r_a + \epsilon_d$,}
\STATE\textbf{Identify} $\mathcal{O}_k, k\in\mathbb{I}$ using \eqref{sensor:closest_theta}, \eqref{sensor:closest_obstacle}.
\STATE\textbf{Locate} $\Pi(\mathbf{x}, \mathcal{O}_k)$ defined in \eqref{projection_using_sensor}.
\IF{$\mathbf{x}^\intercal(\mathbf{x} -\Pi(\mathbf{x}, \mathcal{O}_k))\leq 0$, see \eqref{back_region_condition},}
\STATE $(\mathbf{x}, m)\in\mathcal{J}_{m}\times\{m\}.$
\ENDIF
\ELSE
\STATE $(\mathbf{x}, m)\in\mathcal{J}_{m}\times\{m\}.$
\ENDIF
\ENDIF
\end{algorithmic}

\label{alg:jumpset_participation}
\end{algorithm}

\section{Simulation Results}\label{section:simulation}

In this section, we present simulation results for a robot navigating in \textit{a priori} unknown environments. In both simulations discussed below, the robot is assumed to be equipped with a range-bearing sensor (\textit{e.g.} LiDAR) with angular scanning range of $360^{\circ}$ and sensing radius $R_s = 1.5m$. The angular resolution of the sensor is chosen to be $0.5^{\circ}.$ The simulations are performed in MATLAB 2020a.

In the first simulation scenario, we consider an environment with 6 convex obstacles, as shown in Fig. \ref{result_robot_trajectory}. The robot with radius $r = 0.3m$ is initialized at $[-22, 0]^\intercal$. The target is located at the origin. The minimum safety distance $r_s = 0.1m$. We set $\epsilon_d = 0.35m$ and choose the location of the vector $\mathbf{s}$, used in \eqref{direction_decision}, to be $[0, -1]^\intercal$. We set the gain value $\gamma$, used in \eqref{proposed_hybrid_controller_2}, to be 0.2. Fig. \ref{result_robot_trajectory} illustrates the motion of the robot towards the target location while avoiding obstacles. Whenever the robot enters in the $\epsilon_d-$neighbourhood of any obstacle, it identifies the points on the boundary of that respective obstacle, which are in the line-of-sight of the center of the robot, to investigate the collision possibilities while operating in the \textit{move-to-target} mode, as shown by the green curve in Fig. \ref{result_robot_trajectory} for obstacle $\mathcal{O}_2$ and $\mathcal{O}_4$. Then by verifying the condition in \eqref{final_condition}, the robot chooses either to stay in the \textit{move-to-target} mode or switch to the \textit{obstacle-avoidance} mode. When the robot operates in the \textit{obstacle-avoidance} mode, it only needs to identify the closest point on the nearest obstacle, as depicted with the pink dot in Fig. \ref{result_robot_trajectory}, which is used in the rotational control vector $\mathbf{v}(\mathbf{x}, m)$ \eqref{definition_of_vim}. 
The complete simulation video can be found at \url{https://youtu.be/llRrbGfvGBA}.
\begin{figure}[H]
    \centering
    \includegraphics[width = 1\linewidth]{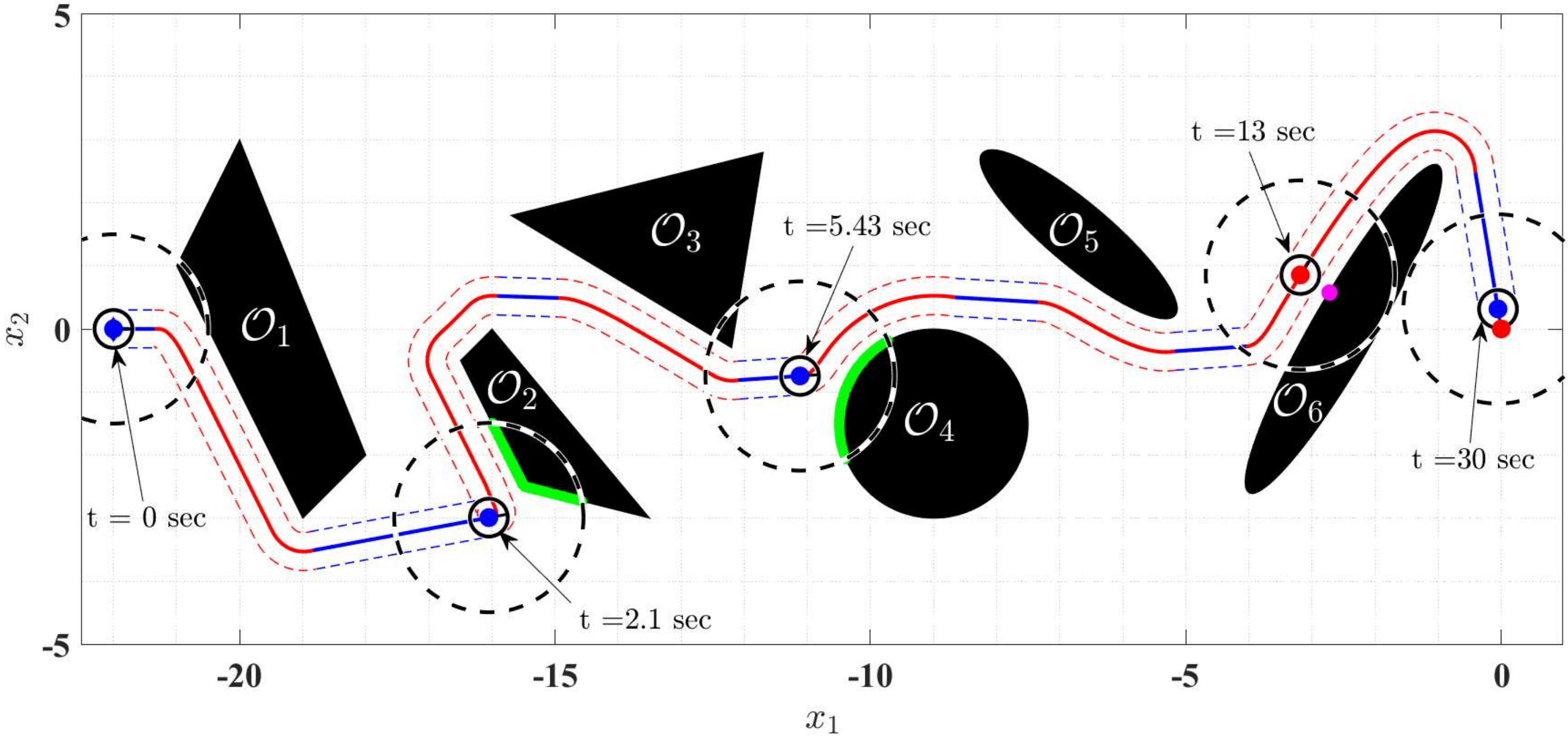}
    \caption{ Robot safely navigating towards the target (red dot), along with three intermediate locations along the path.}
    \label{result_robot_trajectory}
\end{figure}

In the second simulation scenario, and as shown in  Fig. \ref{result_final}, we consider an environment consisting of convex obstacles with smooth and non-smooth boundaries, and apply the proposed hybrid controller \eqref{proposed_hybrid_controller_2} for a point robot navigation initialized at 14 different locations in the obstacle-free workspace. 
The safety distance $r_s = 0.15$, and the value of the variable $\epsilon_d = 0.5$. We set the gain value $\gamma$, used in \eqref{proposed_hybrid_controller_2}, to be 0.2. For each initialization, the vector $\mathbf{s}$, used in \eqref{direction_decision}, is selected such that the initial location of the robot belongs to the half-line $\mathcal{L}_{\geq}(\mathbf{0}, \nu_1(\mathbf{s}))$. It can be noticed that the point robot intersects with the half-line $\mathcal{L}_{>}(\mathbf{0}, \nu_1(\mathbf{s}))$ and with each consecutive intersections it moves closer to the target location at the origin while ensuring obstacle avoidance, as shown in Fig. \ref{result_final}. The complete simulation video can be found at \url{https://youtu.be/_AwDqNY06rU}.

Notice that when the robot operates in the \textit{move-to-target} mode, in addition to its own location and the target location, it only requires its distance from the nearby obstacles. When it operates in the \textit{obstacle-avoidance} mode, it further requires the closest point on the obstacle-occupied workspace $\mathcal{O}_{\mathcal{W}}$. The robot needs to identify all the points on the closest obstacle which are in the line-of-sight of the center of the robot only when it operates in the \textit{move-to-target} mode inside the $\epsilon_d-$neighbourhood of any obstacle so that it can evaluate the possibility of collision and switch to the \textit{obstacle-avoidance} mode, as stated in Remark \ref{remark:information_required}.

\begin{figure}
    \centering
    \includegraphics[width = 0.9\linewidth]{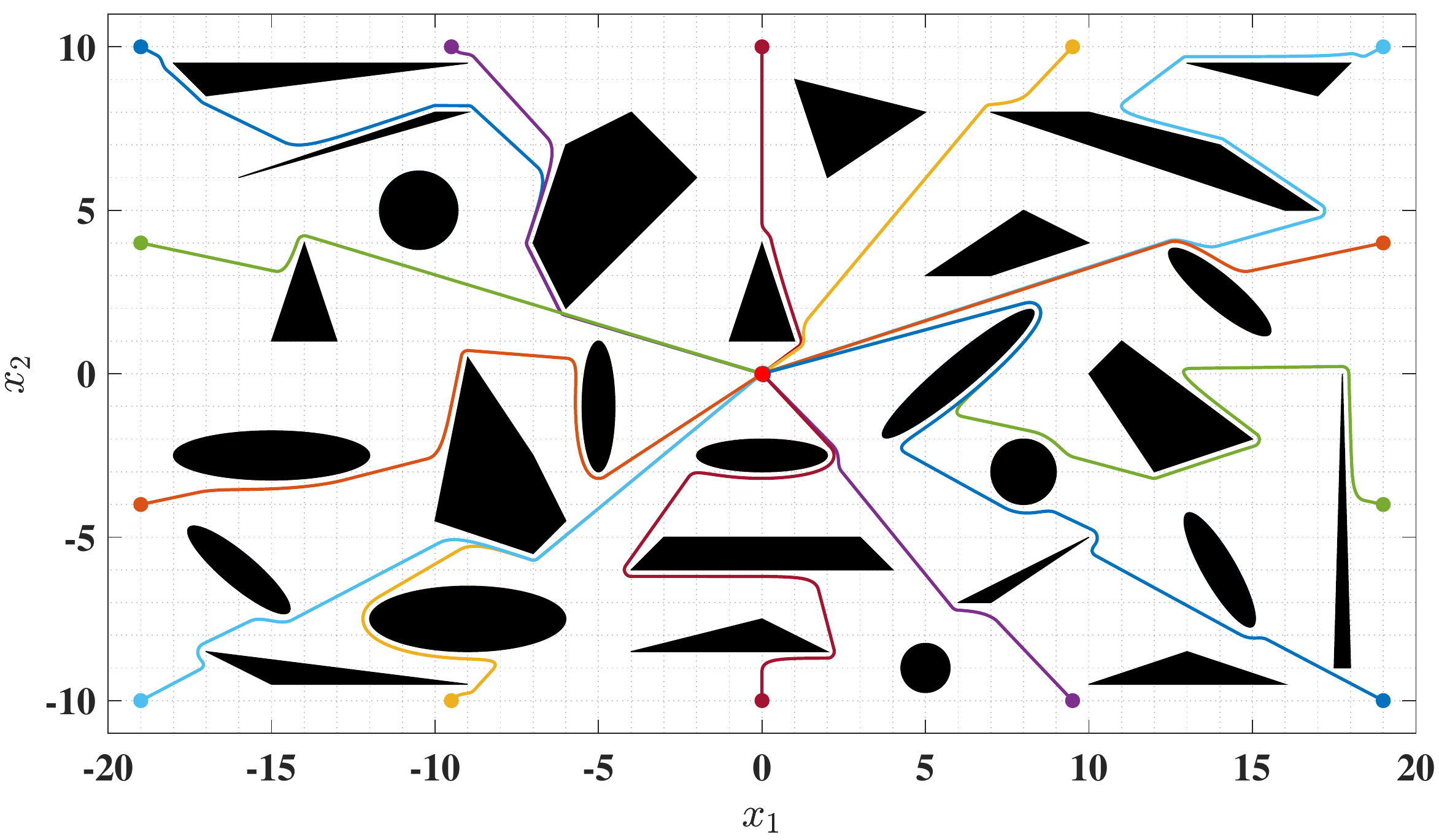}
    \caption{Robot trajectories starting from different locations.}
    \label{result_final}
\end{figure}

The proposed hybrid feedback algorithm has been designed with some robustness to noise properties through the additional safety layers around the obstacles and the overlaps between the flow and jump sets. For example, for a range-bearing sensor with measurement error of $\pm\delta$, one should ensure that the separation between the $\epsilon_d-$safety layer and the $\epsilon_s-$safety layer should be greater than $\delta$. Similarly, to ensure collision free motion while operating in the \textit{obstacle-avoidance} mode, the $\epsilon-$safety layer should be larger than $\delta$, see Fig. \ref{Partitions_of_local_neighbourhood} for the construction of these layers. 

The simulation results given in Fig. \ref{fig:error_trajectory} show the effectiveness of our proposed algorithm implemented with noisy sensor data. We consider an environment similar to the one shown in Fig. \ref{result_robot_trajectory}. The robot radius is $0.3m$ and the minimum safety distance is $r_s = 0.1m$. We set $\epsilon_d = 0.35m$ and choose the gain $\gamma = 0.2$. The range measurements are affected by a Gaussian noise of $0$ mean and $50mm$ standard deviation. Figure \ref{fig:error_trajectory} shows the trajectory of the robot, initialized at $[-22, 0]^{\intercal}$, converging to the target location at the origin. Figure \ref{fig:error_distance} indicates that even in the presence of measurement noise, the robot maintains a safe distance from the obstacle-occupied workspace.

\begin{figure}[h!]
    \centering
    \includegraphics[width = 1\linewidth]{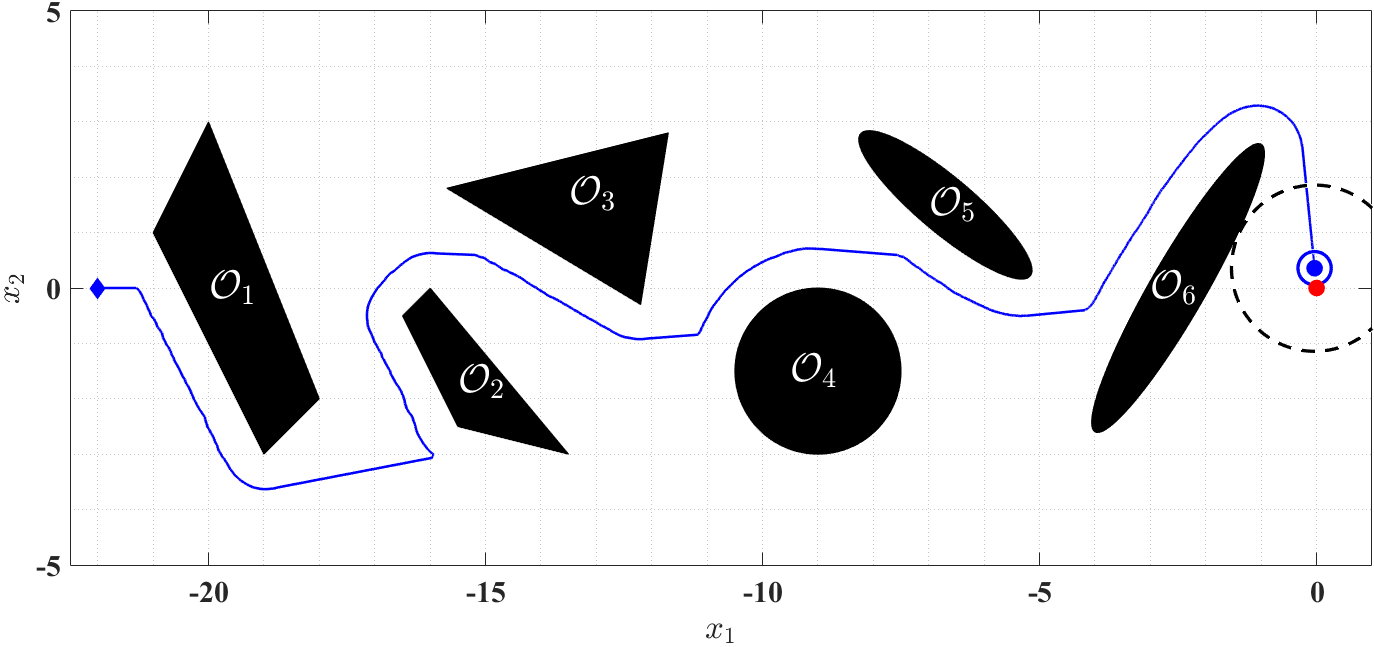}
    \caption{The figure shows the trajectory of a robot, which is equipped with a range-bearing sensor, converging safely towards the target location at the origin. The sensor measurements are affected by the Gaussian noise of $0$ mean and $50mm$ standard deviation.}
    \label{fig:error_trajectory}
\end{figure}
\begin{figure}[h!]
\centering
    \includegraphics[width = 1\linewidth]{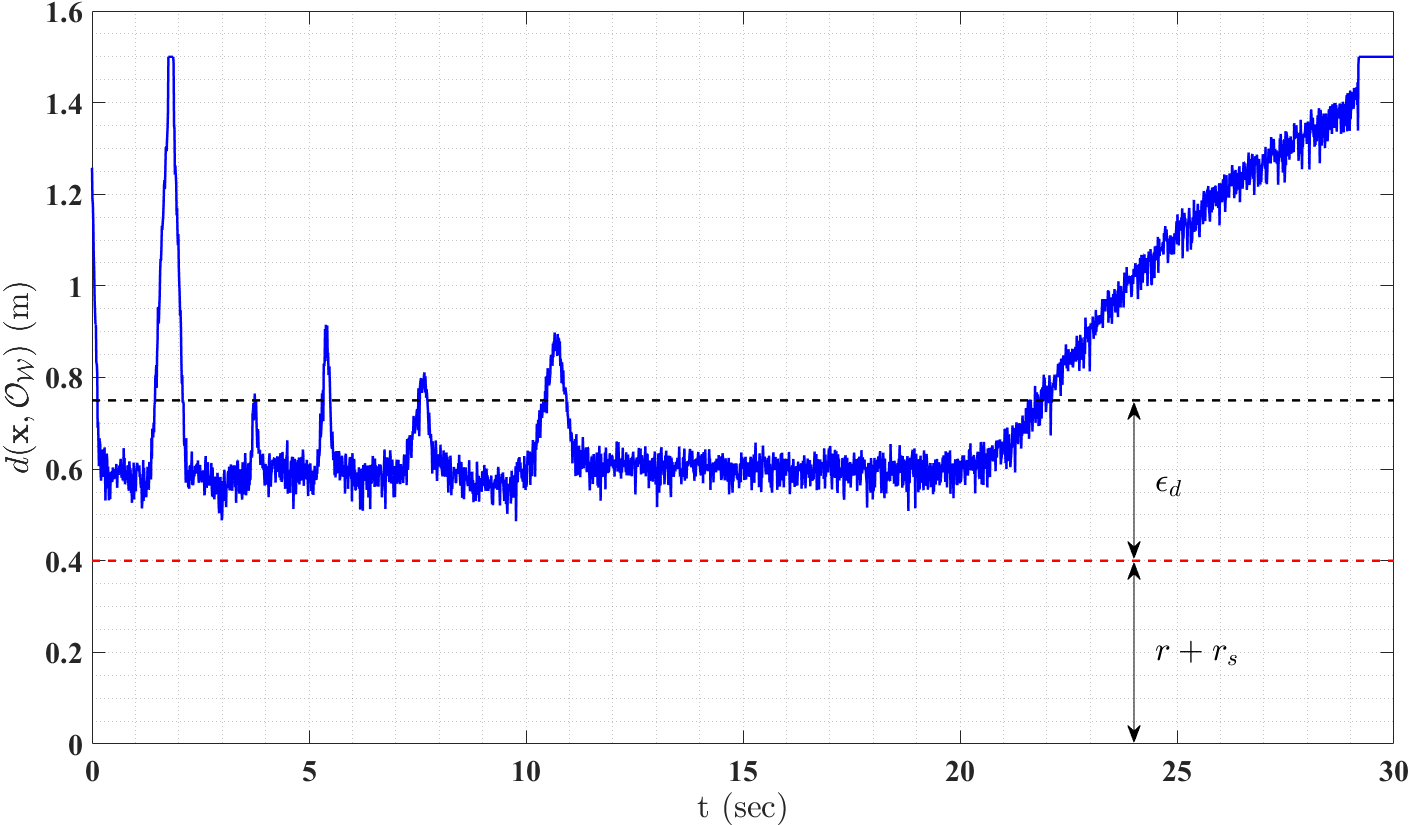}
    \caption{The figure depicts the distance of the center of the robot from the boundary of the obstacle occupied workspace as robot converges towards the target location.}
    \label{fig:error_distance}
\end{figure}


We next provide a comparison with the separating hyperplane approach recently developed in \cite{arslan2019sensor}. Similar to our approach, this approach can be implemented in \textit{a priori} unknown environments using the information obtained via a range-bearing sensor mounted on the robot. Contrary to our approach, this approach only works for convex obstacles that satisfy the curvature condition \cite[Assumption 2]{arslan2019sensor}. When this assumption is not satisfied, the separating hyperplane approach suffers from the presence of local minima. Some of the differences between our approach and the separating hyperplane approach are given below:

In the separating hyperplane-based navigation approach the robot has to construct a local obstacle-free region by first identifying the lines joining the closest point on each obstacle within the sensing range with its location and then by constructing the hyperplanes perpendicular to these lines that separate the robot's body from the obstacles. Then at each control update step, it has to locate the projection of the target location onto the boundary of the local obstacle-free region. Compared to this approach, in our proposed sensor-based hybrid feedback approach, the robot only requires the closest point on the nearest obstacle when it operates in the \textit{obstacle-avoidance} mode. It only needs to identify all the points on the closest obstacle, which are in the line-of-sight of the center of the robot only when it operates in the \textit{move-to-target} mode inside the $\epsilon_d-$neighbourhood of that obstacle, to evaluate the possibility of collision and switch to the \textit{obstacle-avoidance} mode, if necessary, as stated in Remark \ref{remark:information_required}.
      
When the environment consists of obstacles that do not satisfy the obstacle curvature condition \cite[Assumption 2]{arslan2019sensor}, the separating hyperplane approach suffers from the presence of undesired local minima, as can be seen in Fig. \ref{fig:comparison_not_working_condition}(left). On the other hand, our proposed hybrid feedback approach always guarantees convergence to the target location regardless of the shape and size of the convex obstacles, see in Fig. \ref{fig:comparison_not_working_condition}(right).

\begin{figure}[h]
    \centering
    \includegraphics[width = 0.49\linewidth]{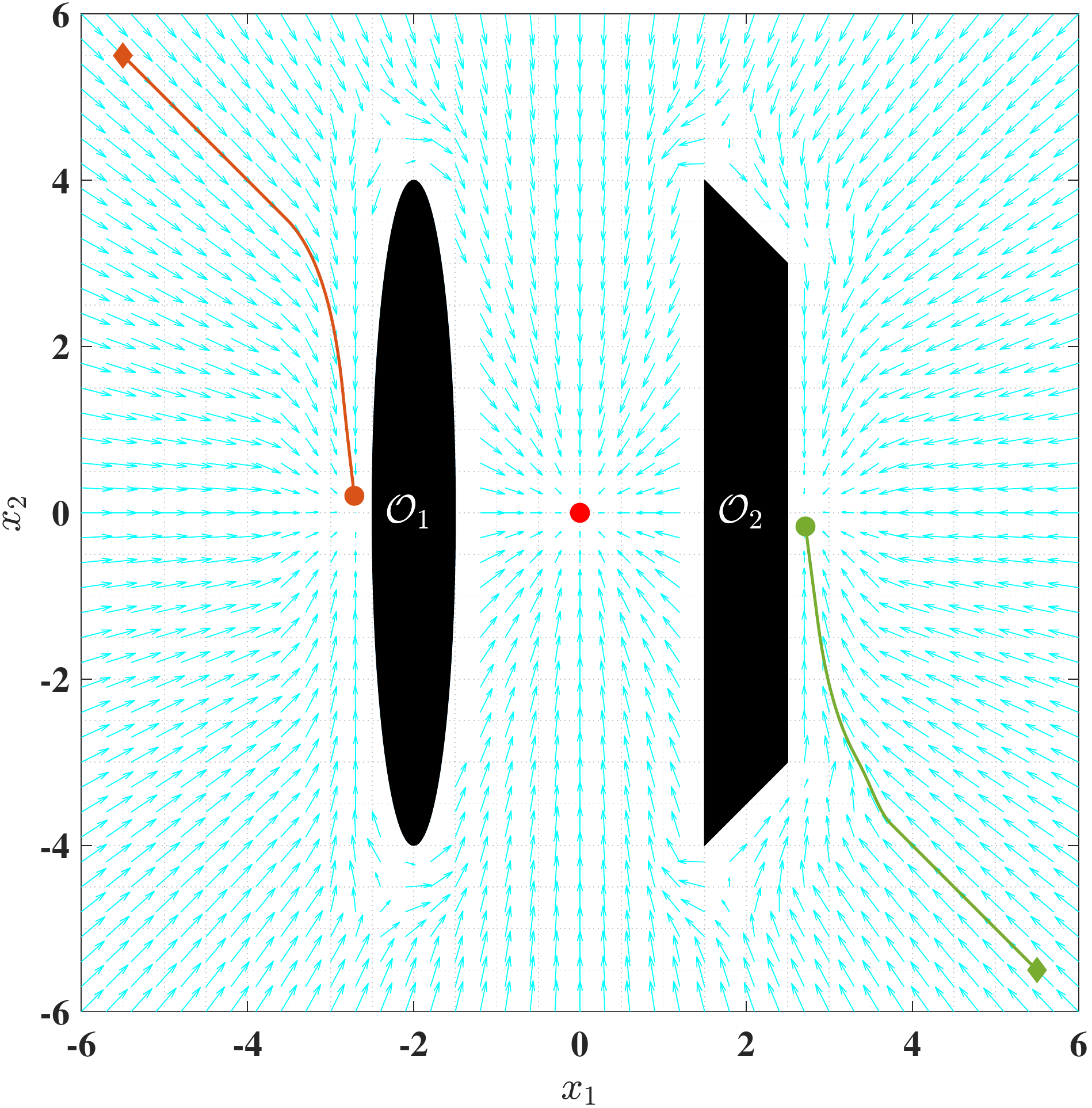}
    \includegraphics[width = 0.49\linewidth]{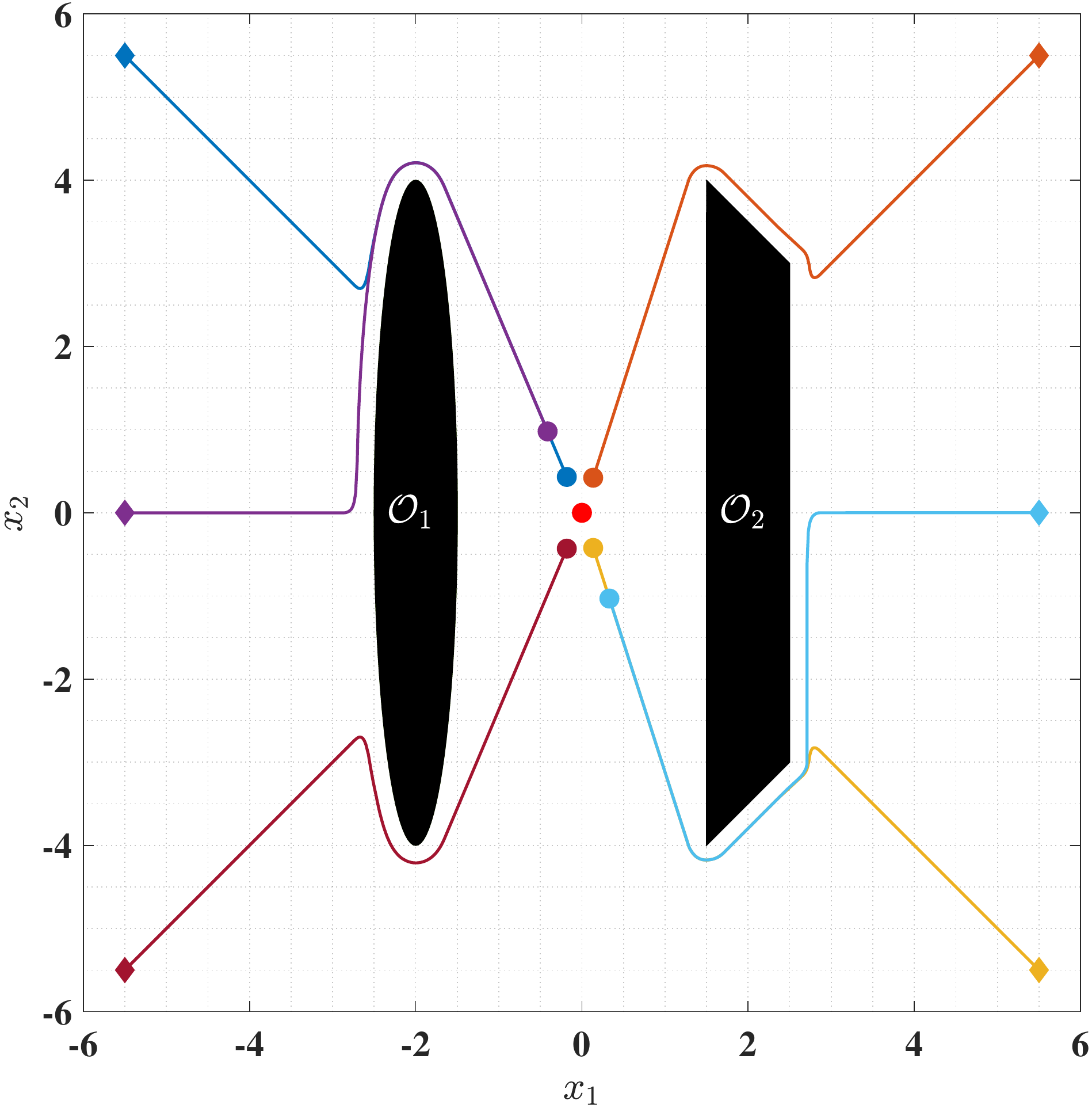}
\caption{Trajectories of a point robot obtained using two different approaches \textit{i.e.}, the separating hyperplane approach \cite{arslan2019sensor} and our proposed hybrid navigation scheme. The left figure represents the robot trajectories obtained using the sensor-based separating hyperplane approach \cite{arslan2017smooth} and illustrates the presence of undesired local minima. The right figure presents the robot trajectories converging towards the target location, which are obtained using the proposed hybrid navigation scheme.}
\label{fig:comparison_not_working_condition}
\end{figure}

If the workspace satisfies the obstacle curvature condition \cite[Assumption 2]{arslan2019sensor}, then for almost all initial locations, the robot trajectories obtained using the separating hyperplane approach converge asymptotically to the target location, while strictly decreasing the Euclidean distance from the robot to the target location \cite[Theorem 3]{arslan2019sensor}, see Fig. \ref{fig:comparison_working_condition}(left). This feature is advantageous compared to our approach. In fact, when the robot operates using our approach, it may travel away from the target when operating in the \textit{obstacle-avoidance} mode, as seen in Fig. \ref{fig:comparison_working_condition}(right) for the robot trajectory initialized at $[-7.5, -10.5]^T$, around the obstacle $\mathcal{O}_1$.

\begin{figure}[h]
    \centering
    \includegraphics[width = 0.49\linewidth]{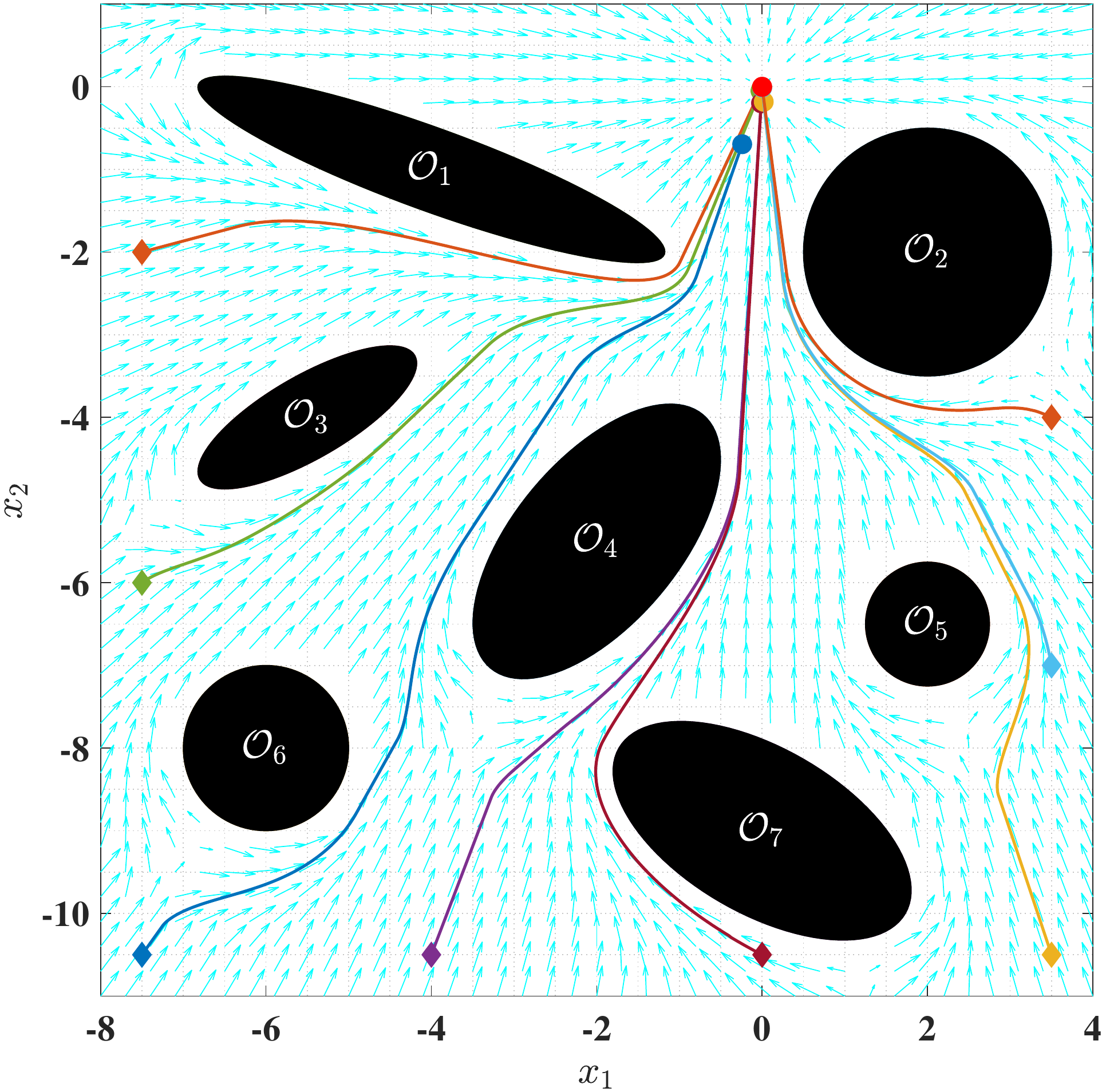}
    \includegraphics[width = 0.49\linewidth]{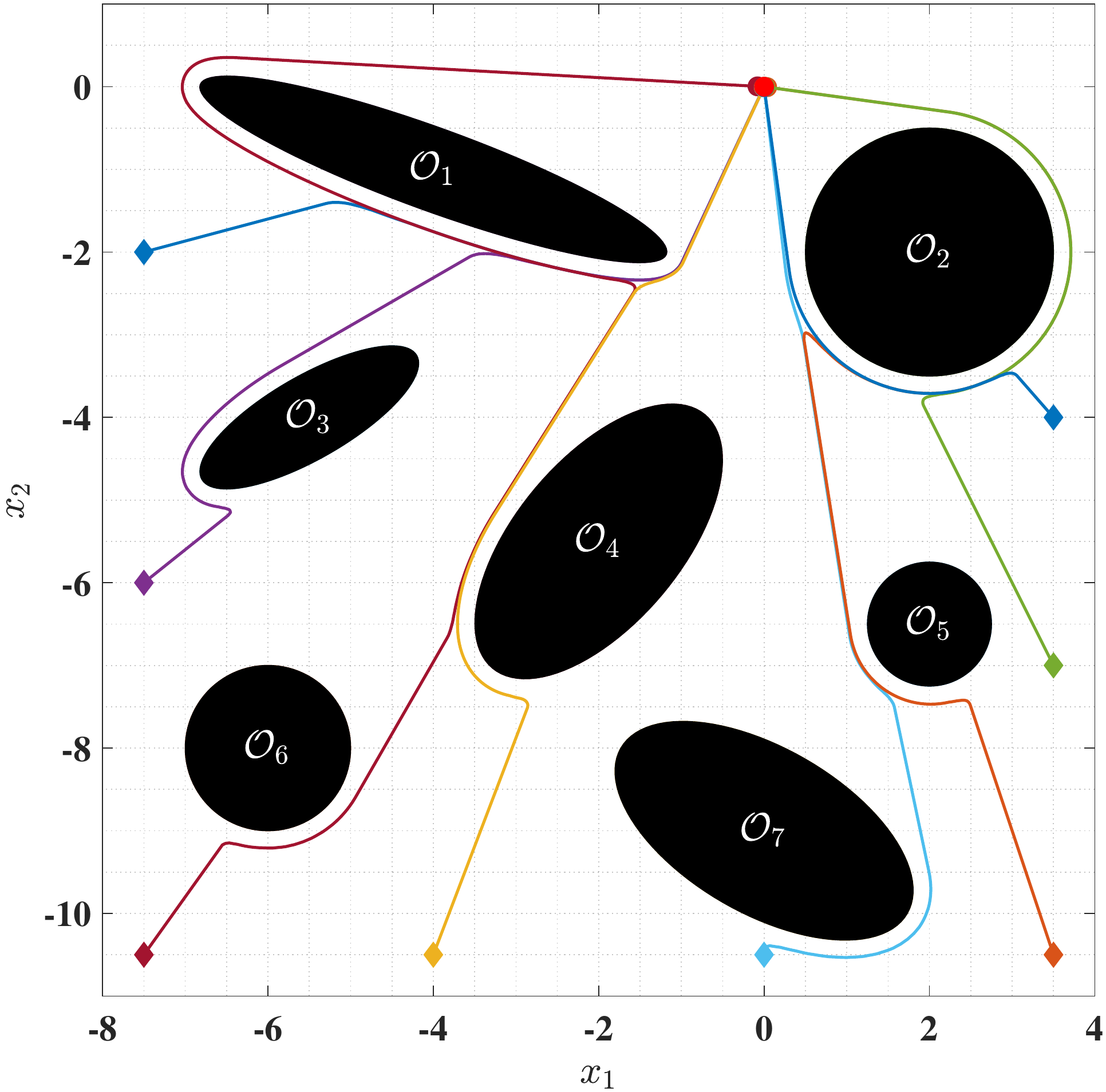}
\caption{Trajectories of a point with 8 different initial locations that are marked by the diamond symbols, converging towards the target location at the origin. The left figure illustrates the robot trajectories obtained using the hyperplane approach \cite{arslan2019sensor}. The right figure presents the robot trajectories obtained using our proposed sensor-based hybrid feedback approach.}
\label{fig:comparison_working_condition}
\end{figure}

\section{Conclusions}\label{sec:conclusion}
We proposed a hybrid feedback controller for safe autonomous navigation in two-dimensional environments with arbitrary convex obstacles. The obstacles can have non-smooth boundaries and large sizes, and can be placed arbitrarily provided that some mild disjointedness requirements are satisfied as per Assumption 1. The proposed hybrid controller guarantees global asymptotic stability, which with the practical adjustments provided in Remark \ref{dwell_time_remark}, can easily be applied in the global sense. 
The mode switching strategy along with the geometric construction of the flow and jump sets ensure the continuity of the control input, which is one of the interesting practical features of the proposed hybrid control scheme. Since the obstacle avoidance part of the control law depends on the projection of the center of the robot on the nearest obstacle, the proposed hybrid control scheme can be applied in \textit{a priori} unknown environments, as discussed in Section \ref{sensor-based-implementation}.
As it can be seen from Fig. \ref{result_final}, the trajectories generated by our algorithm do not necessarily correspond to the shortest paths from the initial configuration to the final one. Designing an update law for the vector $\mathbf{s}$, used in \eqref{direction_decision}, relying on the available local information, which might help in generating optimal trajectories, would be an interesting extension to the present work. Other interesting extensions consist in considering robots with second-order dynamics and three-dimensional environments with non-convex obstacles. 

\begin{appendix}
\subsection{Hybrid basic conditions}
\begin{lemma}
The hybrid closed-loop system \eqref{hybrid_closed_loop_system} with the data $(\mathcal{F}, \mathbf{F}, \mathcal{J}, \mathbf{J})$ satisfies the hybrid basic conditions stated in \cite[Assumption 6.5]{goedel2012hybrid}.\label{hybrid_basic_conditions}
\end{lemma}
\begin{proof}
The flow set $\mathcal{F}$ and the jump set $\mathcal{J}$, defined in \eqref{composite_jump_and_flow_sets}, \eqref{overall_flow_set_jump_set}, are closed subsets of $\mathbb{R}^2\times\mathbb{M}\times\mathbb{I}.$ The flow map $\mathbf{F}$, given in \eqref{hybrid_closed_loop_system}, is continuous on $\mathcal{F}_0\times\{0\}\times\mathbb{I}$. In \eqref{hybrid_closed_loop_system}, due to the structure of the scalar function $\kappa(\xi)$ in \eqref{kappa_function_formulation}-\eqref{beta_function_definition}, the rotational vector $\mathbf{v}(\xi)$, defined in \eqref{definition_of_vim}, is active only within the $\epsilon_s-$neighbourhood of any obstacle, which according to Assumption \ref{assumption:robot_pass_through}, is valid for at most one obstacle at a time, namely $\mathcal{O}_k$, $k\in\mathbb{I}$. Also, as $\mathcal{O}_k$ is convex, the projection of $\mathbf{x}$ on $\mathcal{O}_k$ \textit{i.e.}, $\Pi(\mathbf{x}, \mathcal{O}_k)$ is continuous with respect to $\mathbf{x}$. As a result, $\mathbf{F}$ is continuous on $\mathcal{F}_m\times\{m\}\times\mathbb{I}, \;m\in\{-1, 1\}$. Hence $\mathbf{F}$ is continuous on $\mathcal{F}$. The jump map $\mathbf{J}$, defined in \eqref{hybrid_closed_loop_system}, is single-valued on $\mathcal{J}_m\times\{m\}\times\mathbb{I},\;m\in\{-1, 1\}$. Also, $\mathbf{J}$ has a closed graph relative to $\mathcal{J}_0\times\{0\}\times\mathbb{I}$ as $\mathbf{M}$, defined in \eqref{direction_decision}, is allowed to be set-valued whenever $\mathbf{x}\in\mathcal{P}(\mathbf{0}, \mathbf{s})$. Hence, according to \cite[Lemma 5.10]{goedel2012hybrid}, $\mathbf{J}$ is outer semi-continuous and locally bounded relative to $\mathcal{J}$.
\end{proof}
\subsection{Proof of Lemma \ref{lemma:normal_to_dilated_obstacle}}\label{proof:normal_lemma}

We consider a closed convex set $\mathcal{V}\subset\mathbb{R}^n$ and $\mathbf{q}\in\mathbb{R}^n\backslash\mathcal{V}.$ According to \cite[Section 8.1]{boyd2004convex} the projection of $\mathbf{q}$ on the closed convex set $\mathcal{D}_r(\mathcal{V})$, $r\in[0, d(\mathbf{q}, \mathcal{V})]$, \cite[Theorem 3.1]{rockafellar1970convex} in the sense of the Euclidean norm \textit{i.e.}, $\Pi(\mathbf{q}, \mathcal{D}_{r}(\mathcal{V}))$ is unique. Since the set $\mathcal{B}_{d(\mathbf{q}, \mathcal{D}_r(\mathcal{V}))}(\mathbf{q})$ just touches the set $\mathcal{D}_r(\mathcal{V})$ at $\Pi(\mathbf{q}, \mathcal{D}_r(\mathcal{V}))$, the vector $\mathbf{q} - \Pi(\mathbf{q}, \mathcal{D}_r(\mathcal{V}))$ is normal to the set $\mathcal{D}_r(\mathcal{V})$ at $\Pi(\mathbf{q}, \mathcal{D}_r(\mathcal{V})).$ Hence, according to \cite[Section 2.5.2]{boyd2004convex}, the hyperplane $\mathcal{P}(\Pi(\mathbf{q}, \mathcal{D}_r(\mathcal{V})), \mathbf{q} - \Pi(\mathbf{q}, \mathcal{D}_r(\mathcal{V})))$ is a supporting hyperplane to the set $\mathcal{D}_{r}(\mathcal{V})$ at $\Pi(\mathbf{q}, \mathcal{D}_{r}(\mathcal{V}))$.

If we show that the points $\mathbf{q}, \Pi(\mathbf{q}, \mathcal{V})$ and $\Pi(\mathbf{q}, \mathcal{D}_r(\mathcal{V})), r\in(0, d(\mathbf{q}, \mathcal{V}))$, are collinear, then the proof is complete. We next consider $\mathcal{B}_{r}(\Pi(\mathbf{q}, \mathcal{V}))$ with $r\in[0, d(\mathbf{q}, \mathcal{V})]$. Let $\mathbf{p}$ be the projection of $\mathbf{q}$ on the set $\mathcal{B}_{r}(\Pi(\mathbf{q}, \mathcal{V}))$ \textit{i.e.}, $\mathbf{p} = \Pi(\mathbf{q}, \mathcal{B}_{r}(\Pi(\mathbf{q}, \mathcal{V}))).$ It is straightforward to notice that the point $\mathbf{p}\in\mathcal{L}_s(\mathbf{q}, \Pi(\mathbf{q}, \mathcal{V})).$ Hence $\mathbf{p}\in\partial\mathcal{D}_{r}(\mathcal{V})$, where the set $\mathcal{D}_{r}(\mathcal{V})$ is convex according to \cite[Theorem 3.1]{rockafellar1970convex}. Moreover, $\mathbf{p} = \Pi(\mathbf{q}, \mathcal{D}_{r}(\mathcal{V}))$, otherwise $\exists \mathbf{p}_1\in \big(\mathcal{L}_{s}(\mathbf{p}, \Pi(\mathbf{q}, \mathcal{D}_{r}(\mathcal{V})))\big)^{\circ}\subset\mathcal{D}_{r}(\mathcal{V})$ such that $\mathbf{p}_1\in\big(\mathcal{B}_{\norm{\mathbf{q} - \Pi(\mathbf{q}, \mathcal{D}_r(\mathcal{V}))}}(\mathbf{q})\big)^{\circ}$ which is a contradiction. Hence, the points $\mathbf{q}, \Pi(\mathbf{q}, \mathcal{D}_r(\mathcal{V}))$ with $r\in(0, d(\mathbf{q}, \mathcal{V}))$, and $\Pi(\mathbf{q}, \mathcal{V})$ are collinear.

\subsection{Proof of Lemma \ref{forward_invariance_theorem_1}}\label{proof:lemma:normal_to_dilated_obstacle}
First, we prove that the union of the flow and jump sets covers exactly the obstacle-free state space $\mathcal{K}$. Inspired by \cite[Appendix 11]{berkane2021arxiv}, for all $i\in\mathbb{I}$ and $m\in\{-1, 1\}$, the satisfaction of the following equation:
\begin{equation}
    \begin{aligned}
        \mathcal{F}_0\cup\mathcal{J}_0 = \mathcal{F}_{m}^i\cup\mathcal{J}_m^i = \mathcal{W}_{r_a},\label{cover_to_be_proved}
    \end{aligned}
\end{equation}
along with \eqref{composite_jump_and_flow_sets} and \eqref{overall_flow_set_jump_set} implies $\mathcal{F}\cup\mathcal{J} = \mathcal{W}_{r_a}\times\mathbb{M}\times\mathbb{I}=:\mathcal{K}.$ 
Next we prove \eqref{cover_to_be_proved}. It is clear that
\begin{align}
    \mathcal{F}_0 \cup \mathcal{J}_0&\stackrel{\eqref{stabilization_mode_jumpflow_set_final}}{=} \left(\bigcap_{i\in\mathbb{I}}\mathcal{F}_0^i\right)\cup\Bigg(\bigcup_{i\in\mathbb{I}}\mathcal{J}_0^i\Bigg)\nonumber\\
    &=\bigcap_{i\in\mathbb{I}}\Bigg[\left(\mathcal{F}_0^i\cup\mathcal{J}_0^{i}\right)\cup \Bigg(\bigcup_{i^{\prime}\in\mathbb{I}, i^{\prime}\ne i}\left(\mathcal{F}_0^i\cup\mathcal{J}_0^{i^{\prime}}\right)\Bigg)\Bigg]\nonumber\\
    &\stackrel{\eqref{jumpset_individual_obstacle}, \eqref{flowset_individual_obstacle}}{=}\bigcap_{i\in\mathbb{I}}\left(\mathcal{W}_{r_a}\cup\mathcal{F}_0^i\right) \stackrel{\eqref{flowset_individual_obstacle}}{=} \bigcap_{i\in\mathbb{I}}\mathcal{W}_{r_a} = \mathcal{W}_{r_a}\nonumber.
\end{align}
Similarly, for each $i\in\mathbb{I}$ and $m \in\{-1, 1\}$, according to \eqref{avoidance_flow_set_individual}, \eqref{individual_jump_set_avoidance_mode}, by construction $\mathcal{F}_m^i\cup\mathcal{J}_m^i = \mathcal{W}_{r_a}.$

Now, inspired by \cite[Appendix 1]{berkane2021arxiv}, for the hybrid closed-loop system \eqref{hybrid_closed_loop_system}, with data $\mathcal{H} = (\mathcal{F}, \mathbf{F}, \mathcal{J}, \mathbf{J})$, define $\mathbf{S}_{\mathcal{H}}(\mathcal{K})$ as the set of all maximal solutions $\xi$ to $\mathcal{H}$ with $\xi(0, 0)\in\mathcal{K}$.
Each $\xi\in\mathbf{S}_{\mathcal{H}}(\mathcal{K})$ has range $\text{rge }\xi \subset \mathcal{K} = \mathcal{F}\cup\mathcal{J}$. Furthermore, if for each $\xi(0, 0)\in\mathcal{K}$ there exists one solution and each $\xi\in\mathbf{S}_{\mathcal{H}}(\mathcal{K})$ is complete, then the set $\mathcal{K}$ will be in fact forward invariant \cite[Definition 3.3]{chai2018forward}. To that end using \cite[Proposition 6.10]{goedel2012hybrid}, we show the satisfaction of the following viability condition: 
\begin{equation}
    \mathbf{F}(\mathbf{x}, m, k) \cap \mathbf{T}_{\mathcal{F}}(\mathbf{x}, m, k) \ne \emptyset, \forall(\mathbf{x}, m, k)\in\mathcal{F}\backslash\mathcal{J},\label{viability_condition}
\end{equation}
which will allow us to establish the completeness of the solution $\xi$ to the hybrid closed-loop system \eqref{hybrid_closed_loop_system}. In \eqref{viability_condition}, the notation $\mathbf{T}_{\mathcal{F}}(\mathbf{x}, m, k)$ denotes the tangent cone\footnote{The tangent cone to a set $\mathcal{K} \subset \mathbb{R}^n$ at a point $x \in\mathbb{R}^n$, denoted $\mathbf{T}_{\mathcal{K}}(x)$, is defined as in~\cite[Def.~5.12 and Fig.~5.4]{goedel2012hybrid}.} to the set $\mathcal{F}$ at $(\mathbf{x}, m, k)$. Let $(\mathbf{x}, m, k)\in\mathcal{F}\backslash\mathcal{J},$ which implies by \eqref{composite_jump_and_flow_sets} that $(\mathbf{x}, k)\in\left(\mathcal{F}_m\backslash\mathcal{J}_m\right)\times\mathbb{I}$ for some $m\in\mathbb{M}$. We consider two cases corresponding to $m = 0$ and $m \in\{-1, 1\}$. 

When $m = 0$, according to \eqref{stabilization_mode_jumpflow_set_final}, there exists $k\in\mathbb{I}$ such that $\xi \in\mathcal{F}_0\backslash\mathcal{J}_0\times\{0\}\times\{k\}.$ For $\mathbf{x}\in(\mathcal{F}_0)^{\circ}\backslash\mathcal{J}_0$, $\mathbf{T}_{\mathcal{F}}(\xi) = \mathbb{R}^2\times\{0\}\times\{0\}$, and \eqref{viability_condition} holds. According to \eqref{jumpset_individual_obstacle}, \eqref{flowset_individual_obstacle} and \eqref{stabilization_mode_jumpflow_set_final}, one has
\begin{equation}
    \partial\mathcal{F}_0\backslash\mathcal{J}_0 \in\bigcup_{k\in\mathbb{I}}\left(\partial\mathcal{D}_{r_a}(\mathcal{O}_k)\cap\mathcal{R}_b^k\right),\nonumber
\end{equation}
and for every $k\in\mathbb{I}$, according to Lemma \ref{lemma:normal_to_dilated_obstacle},  $\mathcal{P}(\Pi(\mathbf{x}, \mathcal{D}_{r_a}(\mathcal{O}_k)), \mathbf{x} - \Pi(\mathbf{x}, \mathcal{O}_k))$ is a supporting hyperplane to $\mathcal{D}_{r_a}(\mathcal{O}_k)$ at $\Pi(\mathbf{x}, \mathcal{D}_{r_a}(\mathcal{O}_k))$, hence $\forall \mathbf{x}\in\partial\mathcal{F}_0\backslash\mathcal{J}_0$
\begin{equation}
\mathbf{T}_{\mathcal{F}}(\mathbf{x}, 0, k) = \mathcal{P}_{\geq}(\mathbf{0}, (\mathbf{x} - \Pi(\mathbf{x}, \mathcal{O}_k))\times\{0\}\times\{0\}.\nonumber
\end{equation}
Also, for $m = 0$, $\mathbf{u}(\mathbf{x}, 0, k) = -\gamma\mathbf{x}, \;\gamma>0$, \eqref{fxim}. Hence, according to \eqref{back_region}, for $\mathbf{x}\in\partial\mathcal{D}_{r_a}(\mathcal{O}_k)\cap\mathcal{R}_b^k$, $\mathbf{u}^{\intercal}(\mathbf{x} - \Pi(\mathbf{x}, \mathcal{O}_k)) \geq 0$, hence $\mathbf{u}(\mathbf{x}, 0, k)\in\mathcal{P}_{\geq}(\mathbf{0}, (\mathbf{x} - \Pi(\mathbf{x}, \mathcal{O}_k))$, and \eqref{viability_condition} holds for $m = 0$. 

When $m \in \{-1, 1\}$, according to \eqref{avoidance_final_set}, there exists $k\in\mathbb{I}$ such that $\mathbf{x}\in\mathcal{F}_m^k\backslash\mathcal{J}_m^k$. For $\mathbf{x}\in(\mathcal{F}_m^k)^{\circ}\backslash\mathcal{J}_m^k,$ $\mathbf{T}_{\mathcal{F}_m}(\mathbf{x}) = \mathbb{R}^2,$ so that $\mathbf{T}_{\mathcal{F}}(\xi) = \mathbb{R}^2\times\{0\}\times\{0\}$, and \eqref{viability_condition} holds. According to \eqref{avoidance_flow_set_individual}, \eqref{individual_jump_set_avoidance_mode} and \eqref{avoidance_final_set}, one has
\begin{equation}
    \partial\mathcal{F}_m^k\backslash\mathcal{J}_m^k\in\partial\mathcal{D}_{r_a}(\mathcal{O}_k),\nonumber
\end{equation}
and according to Lemma \ref{lemma:normal_to_dilated_obstacle},  $\mathcal{P}(\Pi(\mathbf{x}, \mathcal{D}_{r_a}(\mathcal{O}_k)), \mathbf{x} - \Pi(\mathbf{x}, \mathcal{O}_k))$ is a supporting hyperplane to $\mathcal{D}_{r_a}(\mathcal{O}_k)$ at $\Pi(\mathbf{x}, \mathcal{D}_{r_a}(\mathcal{O}_k))$, hence $\forall\mathbf{x}\in\partial\mathcal{F}_m^k\backslash\mathcal{J}_m^k$
\begin{equation}
    \mathbf{T}_{\mathcal{F}}(\mathbf{x}, m, k)=\mathcal{P}_{\geq}(\mathbf{0}, (\mathbf{x} - \Pi(\mathbf{x}, \mathcal{O}_k))\times\{0\}\times\{0\}\nonumber,
\end{equation}
and according to \eqref{fxim}, $\mathbf{u}(\mathbf{x}, m, k) =\gamma\mathbf{v}(\mathbf{x}, m, k), \gamma > 0$. Since $\Pi(\mathbf{x}, \mathcal{O}_{\mathcal{W}})$ equals $\Pi(\mathbf{x}, \mathcal{O}_k)$, $\mathbf{v}(\mathbf{x}, k, m)^{\intercal}(\mathbf{x} - \Pi(\mathbf{x}, \mathcal{O}_k)) = 0$, and the condition in \eqref{viability_condition} holds true for $m = \{-1, 1\}$.

Hence, according to \cite[Proposition 6.10]{goedel2012hybrid}, since \eqref{viability_condition} holds for all $\xi\in\mathcal{F}\backslash\mathcal{J}$, there exists a nontrivial solution to $\mathcal{H}$ for each initial condition in $\mathcal{K}$. Finite escape time can only occur through flow. They can neither occur for $\mathbf{x}$ in the set $\mathcal{F}_{-1}^k\cup\mathcal{F}_{1}^k$, $k\in\mathbb{I}$, as these sets are bounded by definition \eqref{avoidance_flow_set_individual}-\eqref{avoidance_final_set}, nor for $\mathbf{x}$ in the set $\mathcal{F}_0$ as this would make $\mathbf{x}^{\intercal}\mathbf{x}$ grow unbounded, and would contradict the fact that $\frac{d}{dt}(\mathbf{x}^{\intercal}\mathbf{x}) \leq 0$ in view of the definition of $\mathbf{u}(\mathbf{x}, 0, k)$. Therefore, all maximal solutions do not have finite escape times. Furthermore, according to \eqref{hybrid_closed_loop_system}, $\mathbf{x}^+ = \mathbf{x}$, and from the definition of the update laws in \eqref{update_law_for_m}, \eqref{update_law_for_k}, it follows immediately that 
$
 \mathbf{J}(\mathcal{J})\subset\mathcal{K}.   
$
Hence, solutions to the hybrid closed-loop system \eqref{hybrid_closed_loop_system} cannot leave $\mathcal{K}$ through jump and, as per \cite[Proposition 6.10]{goedel2012hybrid}, all maximal solutions are complete.

\subsection{Proof of Lemma \ref{no_revolution_around_the_target}}\label{proof:no_closed_trajectories}

Let $\xi := (\mathbf{x}, m, k)$ be the solution to the hybrid closed-loop system \eqref{hybrid_closed_loop_system}. Notice that for the robot operating in the \textit{move-to-target} mode $(m = 0)$, if $\mathbf{x}(t_0, j_0)\notin\mathcal{L}_{>}(\mathbf{0}, \nu_{-1}(\mathbf{s}))$ at some $(t_0, j_0)\in\text{ dom }\xi$, then $\mathbf{x}(t, j)\notin\mathcal{L}_{>}(\mathbf{0}, \nu_{-1}(\mathbf{s})), \forall(t, j)\succeq(t_0, j_0)$, as long as it does not encounter any obstacle in the way, since $\mathcal{L}(\mathbf{0}, \mathbf{x}(t_0, j_0))\cap\mathcal{L}_{>}(\mathbf{0}, \nu_{-1}(\mathbf{s})) = \emptyset.$ Hence, we investigate the case where the solution $\xi$ evolves in the \textit{obstacle-avoidance} mode $(m\in\{-1, 1\}).$
\begin{lemma}
Under Assumption \ref{assumption:robot_pass_through}, each maximal solution $\mathbf{x}$ to the flow-only  system
\begin{equation}
    \mathbf{\dot{x}} = \mathbf{u}(\mathbf{x}, m, k), \;\mathbf{x}\in\mathcal{F}_m^k,\label{flow_only_system}
\end{equation}
with $m \in\{-1, 1\}$ and $k\in\mathbb{I}$, has $T = \text{sup}_t\text{ dom }\mathbf{x} < +\infty$ and $\mathbf{x}(T)\in\mathcal{G}_{m}^k$.
\label{eventually_exit_avoidance_mode}
\end{lemma}
\begin{proof}
See Appendix \ref{proof:even_exit_avoidance}.
\end{proof}

Lemma \ref{eventually_exit_avoidance_mode} indicates that the solution $\mathbf{x}(t)$ to the flow-only system \eqref{flow_only_system}, evolving in the \textit{obstacle-avoidance} flow set $\mathcal{F}_m^k$ related to some obstacle $\mathcal{O}_k, k\in\mathbb{I}$ with some $m\in\{-1, 1\}$, will enter in the gate region $\mathcal{G}_m^k$ in finite time. As $\mathcal{G}_m^k\times\{m\}\times\{k\}\subset\mathcal{J}_m\times\{-1, 1\}\times\mathbb{I}$, according to Lemma \ref{eventually_exit_avoidance_mode}, the solution $\xi$ evolving in the \textit{obstacle-avoidance} mode with some mode indicator $m \in\{-1, 1\}$, with respect to some obstacle $\mathcal{O}_k, \;k\in\mathbb{I}$, will ultimately enter in the \textit{move-to-target} mode.

Next, we introduce several notations to denote the instances where the solution $\xi$ enters and leaves the \textit{obstacle-avoidance} mode with respect to some obstacle $\mathcal{O}_k, \;k\in\mathbb{I}$.

Let $(t_{\mathcal{F}}^k, j_{\mathcal{F}}^k)\in\text{ dom }\xi$ such that $\xi_{\mathcal{F}}^k = \xi(t_{\mathcal{F}}^k, j_{\mathcal{F}}^k) = (\mathbf{x}_{\mathcal{F}}^k, m, k)\in\mathcal{F}_m^k\times\{-1, 1\}\times\{k\}$ be the state at the instant when a solution $\xi$ to the hybrid closed-loop system \eqref{hybrid_closed_loop_system}, earlier flowing in the \textit{move-to-target} mode, enters in the flow set of the \textit{obstacle-avoidance} mode related to obstacle $\mathcal{O}_k, k\in\mathbb{I},$ for some $m\in\{-1, 1\}.$ Similarly, let $(t_{\mathcal{J}}^k, j_{\mathcal{J}}^k)\in\text{ dom }\xi$ such that  $\xi_{\mathcal{J}}^k = \xi(t_{\mathcal{J}}^k, j_{\mathcal{J}}^k) = (\mathbf{x}_{\mathcal{J}}^k, m, k)\in\mathcal{J}_m^k\times\{-1, 1\}\times\{k\}$ be the state at the instant when the solution $\xi$, earlier flowing in the \textit{obstacle-avoidance} mode with respect to the obstacle $\mathcal{O}_k$, enters the jump set of the \textit{obstacle-avoidance} mode associated with the respective obstacle.

We partition the obstacle-free state space $\mathcal{K}$ \eqref{composite_state_vector} into three subsets, based on the location of the vector $\mathbf{s}\in\mathbb{R}^2\backslash\{\mathbf{0}\}$ used in the update law \eqref{direction_decision}, as follows:
\begin{equation}
    \mathcal{K} = \tilde{\mathcal{K}}_{1}(\mathbb{M}) \cup \tilde{\mathcal{K}}_{-1}(\mathbb{M})\cup\tilde{\mathcal{K}}_0(\mathbb{M}),
    \end{equation}where
\begin{equation}
\begin{aligned}
    \tilde{\mathcal{K}}_{1}(\{z\}) &= \left(\mathcal{P}_{>}(\mathbf{0}, \mathbf{s})\cap\mathcal{W}_{r_a}\right)\times\{z\}\times\mathbb{I},\\
    \tilde{\mathcal{K}}_{-1}(\{z\}) &= \left(\mathcal{P}_{<}(\mathbf{0}, \mathbf{s})\cap\mathcal{W}_{r_a}\right)\times\{z\}\times\mathbb{I},\\
    \tilde{\mathcal{K}}_0(\{z\}) &= \left(\mathcal{P}(\mathbf{0}, \mathbf{s})\cap\mathcal{W}_{r_a}\right)\times\{z\}\times\mathbb{I}, \nonumber 
\end{aligned}
\end{equation}
furthermore, let
\begin{equation}
\begin{aligned}
    \tilde{\mathcal{K}}_{>}(\{z\}) &= \mathcal{L}_{>}(\mathbf{0}, \nu_1(\mathbf{s}))\times \{z\}\times\mathbb{I},\\
    \tilde{\mathcal{K}}_{<}(\{z\}) &= \mathcal{L}_{>}(\mathbf{0}, \nu_{-1}(\mathbf{s}))\times\{z\}\times\mathbb{I},\nonumber
\end{aligned}
\end{equation}
where $\{z\} \subseteq \mathbb{M},$ such that $\tilde{\mathcal{K}}_0(\mathbb{M}) = \mathcal{A}\cup\tilde{\mathcal{K}}_{>}(\mathbb{M})\cup\tilde{\mathcal{K}}_{<}(\mathbb{M})$. Now, we proceed with the proof.

Given $\xi(t_0, j_0)\in\mathcal{W}_{r_a}\backslash\mathcal{L}_{>}(\mathbf{0}, \nu_{-1}(\mathbf{s}))\times\{0\}\times\mathbb{I}$, we further assume $\xi(t_0, j_0)\in\tilde{\mathcal{K}}_{m^*}(\{0\})$ for some $m^*\in\{-1, 1\}$ and that $\exists(t_{\mathcal{F}}^k, j_{\mathcal{F}}^k - 1)\succeq(t_0, j_0)$ such that $\xi(t_{\mathcal{F}}^k, j_{\mathcal{F}}^k - 1)\in\mathcal{J}_0\times\{0\}\times\mathbb{I}, \; k\in\mathbb{I}.$ According to \eqref{update_law_for_m} and \eqref{direction_decision}, $\xi(t_{\mathcal{F}}^k, j_{\mathcal{F}}^k) \in \mathcal{F}_{m^*}^k\times\{m^*\}\times\{k\}$. Then according to Lemma \ref{eventually_exit_avoidance_mode}, $\exists(t_{\mathcal{J}}^k, j_{\mathcal{J}}^k)\succ(t_{\mathcal{F}}^k, j_{\mathcal{F}}^k)$ such that $\xi(t_{\mathcal{J}}^k, j_{\mathcal{J}}^k + 1)\in\mathcal{F}_0\times\{0\}\times\mathbb{I}.$ Now, in order to prove our claim, we show that 
\begin{equation}
    \mathbf{u}(\mathbf{x}(t, j), m^*, k) \in \mathcal{P}_{\geq}(\mathbf{0}, \nu_{m^*}(\mathbf{x}(t, j))),\label{claim_about_no_revolution}
\end{equation}
for all $(t, j)\in([t_{\mathcal{F}}^k, t_{\mathcal{J}}^k]\times[j_{\mathcal{F}}^k, j_{\mathcal{J}}^k])$. Since $\xi(t_0, j_0)\in\tilde{\mathcal{K}}_{m^*}(\{0\})$, $\mathcal{P}_{\geq}(\mathbf{0}, \nu_{m^*}(\mathbf{x}_{\mathcal{F}}^k))\cap\mathcal{L}_{>}(\mathbf{0}, \nu_{-1}(\mathbf{s})) = \emptyset$. The satisfaction of \eqref{claim_about_no_revolution} ensures that the component $\mathbf{x}$ of the solution $\xi$ always evolves towards the positive half-space generated by the hyperplane $\mathcal{P}(\mathbf{0}, \nu_{m^*}(\mathbf{x}(t, j)))$, as shown in Fig. \ref{figure_no_revolution}, such that
\begin{equation}
    \xi(t, j)\notin\tilde{\mathcal{K}}_{<}(\mathbb{M}), \forall (t, j)\in([t_{\mathcal{F}}^k, t_{\mathcal{J}}^k]\times[j_{\mathcal{F}}^k, j_{\mathcal{J}}^k]).\nonumber
\end{equation}
 To that end, we analyze the behaviour of the solution $\xi(t, j)$, $\forall (t, j)\in([t_{\mathcal{F}}^k, t_{\mathcal{J}}^k]\times[j_{\mathcal{F}}^k, j_{\mathcal{J}}^k]).$
\begin{figure}
    \centering
    \includegraphics[width = 0.75\linewidth]{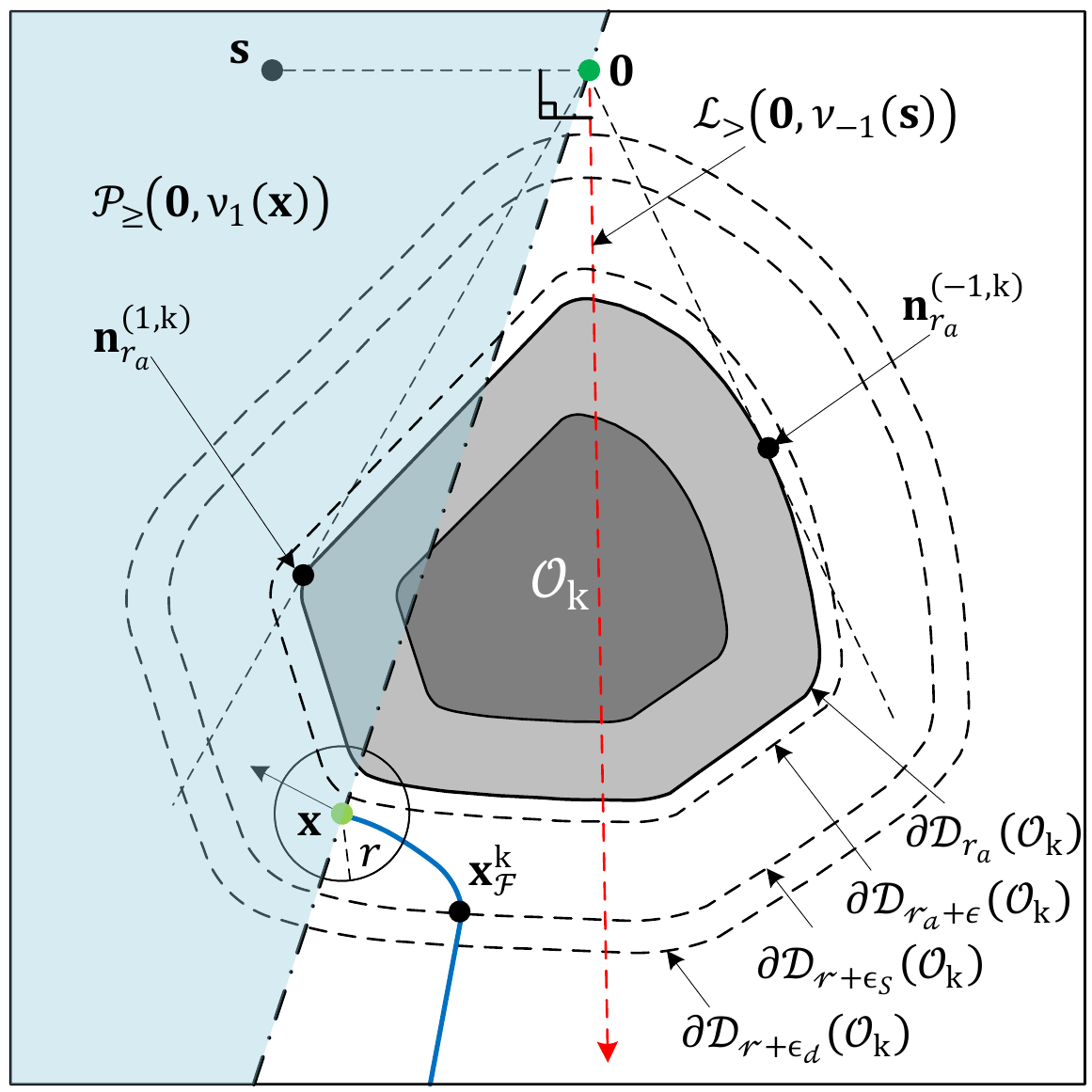}
    \caption{Illustration of a robot trajectory (blue curve) not intersecting the half-line $\mathcal{L}_{>}(\mathbf{0}, \nu_{-1}(\mathbf{s}))$ (red line) while operating in the \textit{obstacle-avoidance} mode in the flow set $\mathcal{F}_1^k\times\{1\}\times\{k\}$. The velocity vector at the current location of the robot (green dot) always points to the positive half-space $\mathcal{P}_{\geq}(\mathbf{0}, \nu_{1}(\mathbf{\mathbf{x}}))$ (shaded blue region).}
    \label{figure_no_revolution}
\end{figure}
For $(t, j)\in([t_{\mathcal{F}}^k,t_{\mathcal{J}}^k]\times[ j_{\mathcal{F}}^k, j_{\mathcal{J}}^k])$, consider the flow-only system \eqref{flow_only_system} for $\mathbf{x}\in\mathcal{F}_{m^*}^k$. According to \eqref{fxim}, the control input vector for the robot operating in the \textit{obstacle-avoidance} mode is given by
\begin{equation}
    \mathbf{u}(\mathbf{x}, m^*, k) = -\gamma\kappa(\xi)\mathbf{x} + \gamma[1 - \kappa(\xi)]\mathbf{v}(\mathbf{x}, m^*, k),\label{what_is_control_1}
\end{equation}
where $\gamma > 0$. Since $\mathbf{x}^\intercal\nu_{m^*}(\mathbf{x}) = 0$, $-\gamma\kappa(\xi)\mathbf{x}\in\mathcal{P}_{\geq}(\mathbf{0}, \nu_{m^*}(\mathbf{x})).$ Next, we consider the rotational vector $\gamma[1- \kappa(\xi)]\mathbf{v}(\mathbf{x}, m^*, k), \gamma[1-\kappa(\xi)] \geq 0$ such that
\begin{equation}
    \begin{aligned}
    \mathbf{v}(\xi)^\intercal\nu_{m^*}(\mathbf{x}) &=\frac{\norm{\mathbf{x}} \nu_{m^*}(\mathbf{x} - \Pi(\mathbf{x}, \mathcal{O}_k))^\intercal\nu_{m^*}(\mathbf{x})}{\norm{\mathbf{x} - \Pi(\mathbf{x}, \mathcal{O}_k)}},\\
    &= \frac{\norm{\mathbf{x}}(\mathbf{x} - \Pi(\mathbf{x}, \mathcal{O}_k))^{\intercal}\mathbf{x}}{\norm{\mathbf{x} - \Pi(\mathbf{x}, \mathcal{O}_k)}}.
    \end{aligned}
\end{equation}

According to Lemma \ref{eventually_exit_avoidance_mode}, the solution $\mathbf{x}$ will evolve within the set $\mathcal{F}_{m^*}^k$ for all $(t, j)\in([t_{\mathcal{F}}^k, t_{\mathcal{J}}^k]\times[j_{\mathcal{F}}^k, j_{\mathcal{J}}^k])$, and at $(t_{\mathcal{J}}^k, j_{\mathcal{J}}^k)$ will enter in the gate region $\mathcal{G}_{m^*}^k$. Hence, as per Remark \ref{remark:back_region_is_closed_connected_subset} and \eqref{gate_region}, for all $(t, j)\in([t_{\mathcal{F}}^k, t_{\mathcal{J}}^k]\times[j_{\mathcal{F}}^k, j_{\mathcal{J}}^k])$, one has
\begin{equation}
    \mathbf{v}(\xi)^\intercal\nu_{m^*}(\mathbf{x}) = \frac{\norm{\mathbf{x}}(\mathbf{x} - \Pi(\mathbf{x}, \mathcal{O}_k))^{\intercal}\mathbf{x}}{\norm{\mathbf{x} - \Pi(\mathbf{x}, \mathcal{O}_k)}} \geq 0,\nonumber
\end{equation}
\textit{i.e.}, the vector $\gamma[1 - \kappa(\xi)]\mathbf{v}(\xi)$ in \eqref{what_is_control_1} belongs to the positive half-space $\mathcal{P}_{\geq}(\mathbf{0}, \nu_{m^*}(\mathbf{x}))$, and \eqref{claim_about_no_revolution} is satisfied.

\subsection{Proof of Lemma \ref{eventually_exit_avoidance_mode}}
\label{proof:even_exit_avoidance}

We consider the flow-only system \eqref{flow_only_system} for the robot operating in the \textit{obstacle-avoidance} mode with respect to some obstacle $\mathcal{O}_k, k\in\mathbb{I}$ \textit{i.e.}, $\mathbf{x}\in\mathcal{F}_m^k$ and $m\in\{-1, 1\}$. We partition the set $\partial\mathcal{F}_m^k$ into several subsets, as shown in Fig. \ref{lemma_flow_set}, based on the similarity between the tangent cones to the set $\mathcal{F}_m^k$ at these regions, as follows:
\begin{equation}
\begin{aligned}
    \partial\mathcal{F}_m^k &= \mathcal{S}_{1}\cup\mathcal{S}_{2}\cup\mathcal{S}_{3}\cup\mathcal{S}_{4}\cup\mathcal{G}_m^k,\\
    \mathcal{S}_{1} &= \left(\partial\mathcal{D}_{r_a + \epsilon_d}(\mathcal{O}_k)\cap\mathcal{F}_m^k\right)\backslash\mathcal{G}_m^k,\\
    \mathcal{S}_{2} &= \mathcal{L}\big(\mathbf{0}, \mathbf{n}_{r_a + \epsilon}^{(- m, k)}\big)\cap\mathcal{F}_m^k,\\
    \mathcal{S}_{3} &= \mathcal{G}_{-m}^k\cap\mathcal{F}_{m}^k,\\
    \mathcal{S}_{4} &= \partial\mathcal{D}_{r_a}(\mathcal{O}_k)\cap\mathcal{F}_m^k.\label{partitions_of_the_boundary}
\end{aligned}
\end{equation}
We proceed to prove the claims in two parts.
\begin{enumerate}
\item We show that if the solution $\mathbf{x}$ to the flow-only system \eqref{flow_only_system} leaves the set $\mathcal{F}_m^k$, it cannot leave via the boundary $\partial\mathcal{F}_m^k\backslash\mathcal{G}_m^k$ \textit{i.e.}, it can only leave the set $\mathcal{F}_m^k$ via $\mathcal{G}_m^k$. To that end, for all $\mathbf{x}\in\partial\mathcal{F}_m^k\backslash\mathcal{G}_m^k$, according to Nagumo's theorem \cite[Theorem 11.2.3]{aubin2011viability}, we verify the following condition:
    \begin{equation}
        \mathbf{u}(\mathbf{x}, m, k)\in\mathbf{T}_{\mathcal{F}_m^k}(\mathbf{x})
    \end{equation}
    where $\mathbf{T}_{\mathcal{F}_m^k}(\mathbf{x})$ is the tangent cone to the set $\mathcal{F}_m^k$ at $\mathbf{x}$.
    \item We show that the solution $\mathbf{x}$ to the flow-only system \eqref{flow_only_system}, flowing in the set $\mathcal{F}_m^k$, away from the region $\mathcal{G}_m^k$, will always enter in the set $\mathcal{G}_m^k$ in finite time. We show that, if $\mathbf{x}(t_0) \in\mathcal{F}_m^k\backslash\mathcal{G}_m^k$ for some $t_0\geq 0$, then $\exists T = \text{ sup}_t\text{ dom }\mathbf{x}<+\infty$, $T > t_0$, such that $\mathbf{x}(T) \in\mathcal{G}_m^k.$
\end{enumerate}
\begin{figure}
    \centering
    \includegraphics[width = 0.49\linewidth]{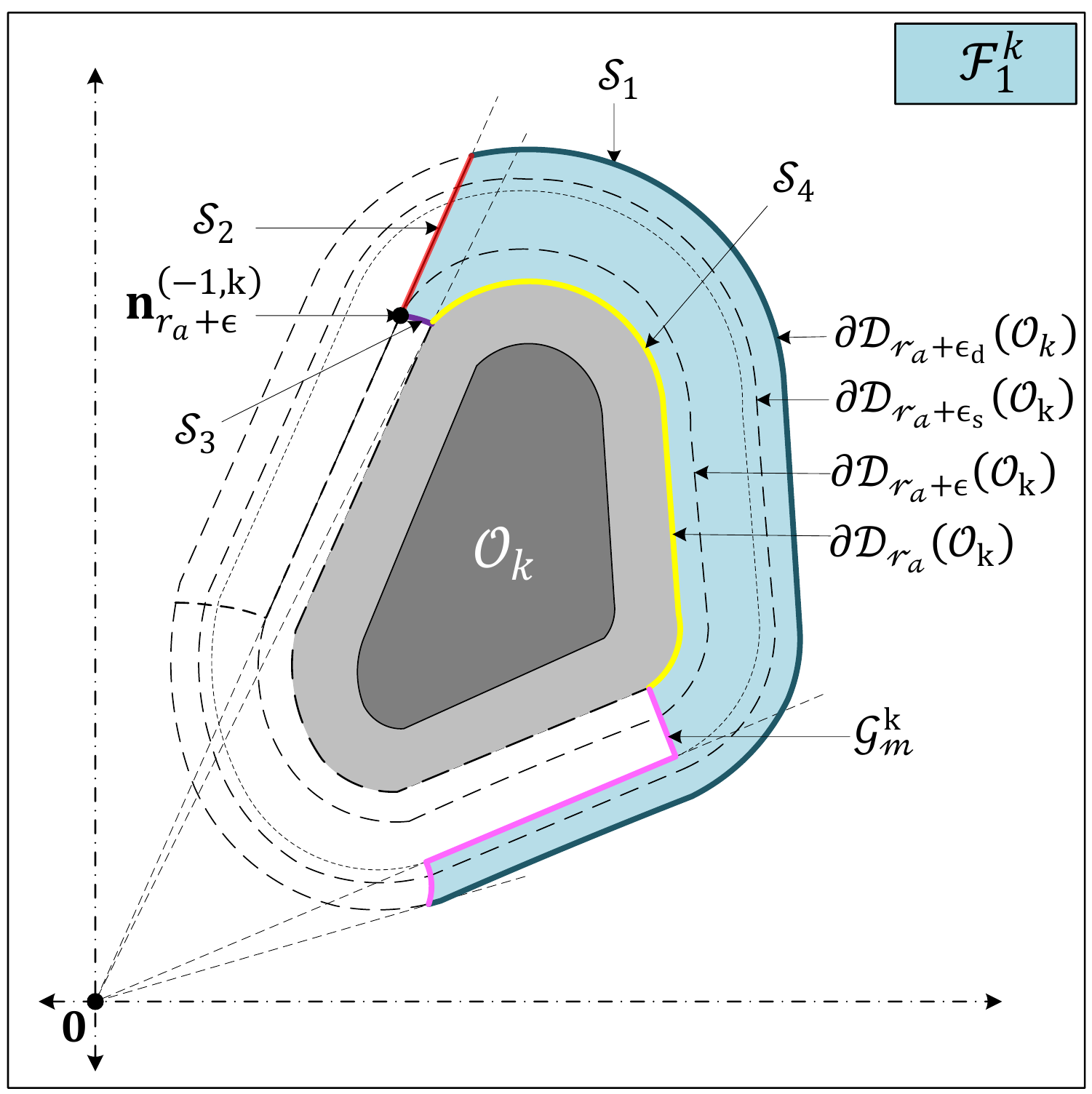}
    \includegraphics[width = 0.49\linewidth]{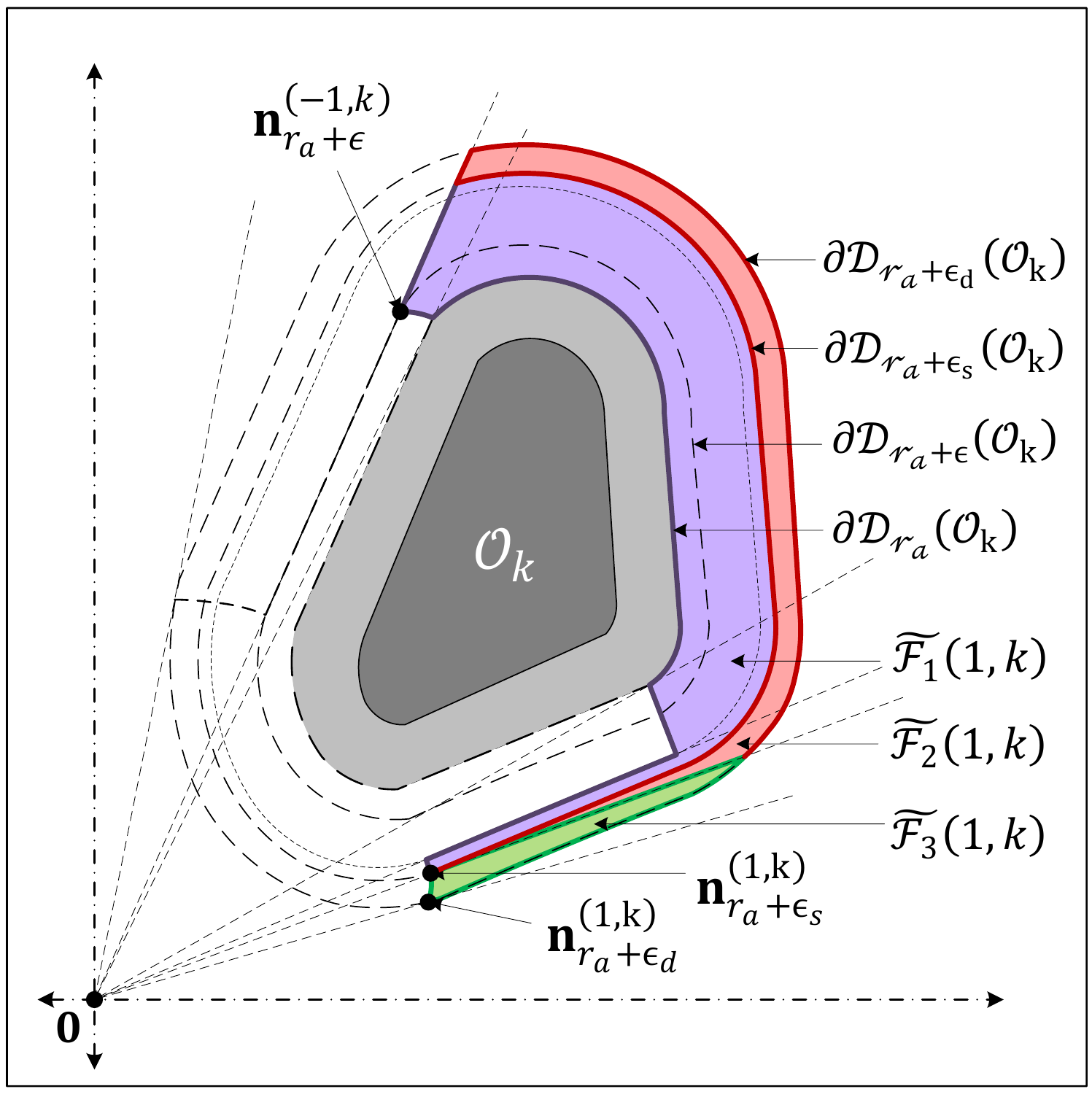}
    \caption{The geometric representation of the \textit{obstacle-avoidance} mode flow set $\mathcal{F}_1^k$ related to an obstacle $\mathcal{O}_k, k\in\mathbb{I}$. The left figure shows the partitions of the set $\partial\mathcal{F}_1^k$. The right figure depicts the partitions of the set $\mathcal{F}_1^k.$}
     \label{lemma_flow_set}
\end{figure}

For $\mathbf{x}\in(\mathcal{S}_{1})^{\bullet}$, according to Lemma \ref{lemma:normal_to_dilated_obstacle}, the vector $(\mathbf{x} - \Pi(\mathbf{x}, \mathcal{O}_k))$ is normal to the convex set $\mathcal{D}_{r_a + \epsilon_d}(\mathcal{O}_k)$ at $\mathbf{x}$. Hence, the tangent cone to the set $\mathcal{F}_m^k$ at $\mathbf{x}\in(\mathcal{S}_{1})^{\bullet}$ is given as
\begin{align}
    \mathbf{T}_{\mathcal{F}_m^k}(\mathbf{x}) = \mathcal{P}_{\leq}(\mathbf{0}, \mathbf{x} - \Pi(\mathbf{x}, \mathcal{O}_k)).\nonumber
\end{align}
Also, for $\mathbf{x}\in(\mathcal{S}_{1})^{\bullet}$, $\mathbf{u}(\xi) = -\gamma\mathbf{x}, \;\gamma>0$. Since $(\mathcal{S}_{1})^{\bullet}\notin\mathcal{R}_b^k$, according to Remark \ref{remark:back_region_is_closed_connected_subset}, $-\gamma\mathbf{x}^{\intercal}(\mathbf{x} - \Pi(\mathbf{x}, \mathcal{O}_k)) \leq 0$. Hence, one has
\begin{equation}
    \forall \mathbf{x}\in(\mathcal{S}_{1})^{\bullet}, \mathbf{u}(\mathbf{x}, m, k)\in \mathbf{T}_{\mathcal{F}_m^k}(\mathbf{x}).\label{s1}
\end{equation}

For $\mathbf{x}\in(\mathcal{S}_{2}\backslash\mathcal{S}_3)^{\bullet}$, the tangent cone to the set $\mathcal{F}_m^k$ at $\mathbf{x}$ is given by
\begin{equation}
    \mathbf{T}_{\mathcal{F}_m^k}(\mathbf{x}) = \mathcal{P}_{\geq}(\mathbf{0}, \nu_{m}(\mathbf{n}_{r_a + \epsilon}^{(-m, k)})),\nonumber
\end{equation}
and $\mathbf{u}(\xi) = -\gamma\kappa(\xi)\mathbf{x} + \gamma[1 - \kappa(\xi)]\mathbf{v}(\xi), \gamma > 0$. For $\mathbf{x}\in(\mathcal{S}_2\backslash\mathcal{S}_3)^{\bullet}$, $\mathbf{x}^\intercal\nu_{m}(\mathbf{n}_{r_a + \epsilon}^{(-m, k)}) = 0$, hence $-\gamma\kappa(\xi)\mathbf{x}\in\mathcal{P}_{\geq}\big(\mathbf{0}, \nu_{m}(\mathbf{n}_{r_a + \epsilon}^{(-m, k)})\big)$. We know that $\gamma[1 - \kappa(\xi)]\geq 0$, and 
\begin{equation}
\begin{aligned}
    \mathbf{v}(\xi)^\intercal\nu_{m}(\mathbf{n}_{r_a + \epsilon}^{(-m, k)})&=\frac{\norm{\mathbf{x}}\nu_{m}(\mathbf{x} - \Pi(\mathbf{x}, \mathcal{O}_k))^\intercal\nu_{m}\big(\mathbf{n}_{r_a + \epsilon}^{(-m, k)}\big)}{\norm{\mathbf{x} - \Pi(\mathbf{x}, \mathcal{O}_k)}},\nonumber\\
    &= \frac{\norm{\mathbf{x}}(\mathbf{x} - \Pi(\mathbf{x}, \mathcal{O}_k))^\intercal\mathbf{n}_{r_a + \epsilon}^{(-m, k)}}{\norm{\mathbf{x} - \Pi(\mathbf{x}, \mathcal{O}_k)}}.\nonumber
\end{aligned}
\end{equation}
Since for $\mathbf{x}\in(\mathcal{S}_{2}\backslash\mathcal{S}_3)^{\bullet}$, the vectors $\mathbf{n}_{r_a + \epsilon}^{(-m, k)}$ and $\mathbf{x}$ point in the same direction and $(\mathcal{S}_{2}\backslash\mathcal{S}_{3})^{\bullet}\notin\mathcal{R}_b^k$, according to Remark \ref{remark:back_region_is_closed_connected_subset}, $\mathbf{x}^\intercal(\mathbf{x} - \Pi(\mathbf{x}, \mathcal{O}_k)) \geq 0$. Hence, for $\mathbf{x}\in(\mathcal{S}_{2}\backslash\mathcal{S}_{3})^{\bullet}$, $\gamma[1 - \kappa(\xi)]\mathbf{v}(\xi)\in\mathcal{P}_{\geq}\big(\mathbf{0}, \nu_{m}(\mathbf{n}_{r_a + \epsilon}^{(-m, k)})\big)$, and 
\begin{equation}
    \forall \mathbf{x}\in(\mathcal{S}_{2}\backslash\mathcal{S}_3)^{\bullet}, \mathbf{u}(\mathbf{x}, m, k)\in \mathbf{T}_{\mathcal{F}_m^k}(\mathbf{x}).\label{s2}
\end{equation}

For $\mathbf{x}\in(\mathcal{S}_4)^{\bullet}$, the tangent cone to the set $\mathcal{F}_m^k$ at $\mathbf{x}$ is given by
\begin{equation}
    \mathbf{T}_{\mathcal{F}_m^k}(\mathbf{x}) = \mathcal{P}_{\geq}(\mathbf{0}, \mathbf{x} - \Pi(\mathbf{x}, \mathcal{O}_k)),\nonumber
\end{equation}
and $\mathbf{u}(\xi) = \gamma\mathbf{v}(\xi)$. Since $\mathbf{v}(\xi)^{\intercal}(\mathbf{x} - \Pi(\mathbf{x}, \mathcal{O}_k)) = 0$, one has
\begin{equation}
    \forall\mathbf{x}\in(\mathcal{S}_4)^{\bullet}, \mathbf{u}(\mathbf{x}, m, k)\in\mathbf{T}_{\mathcal{F}_m^k}(\mathbf{x}).
\end{equation}

For $\mathbf{x}\in\mathcal{S}_{1}\cap\mathcal{S}_{2}$, the tangent cone to the set $\mathcal{F}_m^k$ at $\mathbf{x}$ is given as
\begin{equation}
    \mathbf{T}_{\mathcal{F}_m^k}(\mathbf{x}) = \mathcal{P}_{\leq}(\mathbf{0}, \mathbf{x} - \Pi(\mathbf{x}, \mathcal{O}_k))\cap\mathcal{P}_{\geq}(\mathbf{0}, \nu_{m}(\mathbf{n}_{r_a + \epsilon}^{(-m,k)})),\nonumber
\end{equation}
and $\mathbf{u}(\xi) = -\gamma\mathbf{x}, \; \gamma > 0$. Since $\mathcal{S}_{1}\cap\mathcal{S}_{2}\notin\mathcal{R}_b^k$ and the fact that the vectors $\mathbf{n}_{r_a + \epsilon}^{(-m,k)}$ and $-\gamma\mathbf{x}$ are collinear, ensures that
\begin{equation}
    \forall \mathbf{x}\in\mathcal{S}_{1}\cap\mathcal{S}_{2}, \mathbf{u}(\mathbf{x}, m, k)\in\mathbf{T}_{\mathcal{F}_m^k}(\mathbf{x}).\label{s1s2}
\end{equation}

Next, to show that the vector $\mathbf{u}(\mathbf{x}, m, k)$ evaluated at $\mathbf{x}\in\mathcal{S}_3$, for some $k\in\mathbb{I}$ and  $m\in\{-1, 1\}$, does not point outside the set $\mathcal{F}_m^k$, we require the following lemma.

\begin{lemma}
\label{lemma:no_close_avoidance} Consider the following flow-only system: 
\begin{equation}
    \dot{\mathbf{x}} = \mathbf{v}(\mathbf{x}, m, k),\; \mathbf{x}\in\mathcal{D}_{r_a + \epsilon}(\mathcal{O}_k)\backslash(\mathcal{D}_{r_a}(\mathcal{O}_k))^{\circ},\label{flow_only_rotational_system}
\end{equation}
for some $k\in\mathbb{I}$ and $m\in\{-1, 1\}$. Let $d(\mathbf{x}(t_0), \mathcal{D}_{r_a}(\mathcal{O}_k)) = \beta\in[0, \epsilon]$ for some $t_0\geq 0$, then $\mathbf{x}(t) \in \partial\mathcal{D}_{r_a + \beta}(\mathcal{O}_k)$ for all $t\geq t_0$.
\end{lemma}
\begin{proof}
According to Lemma \ref{lemma:normal_to_dilated_obstacle}, for $\mathbf{x}\in\partial\mathcal{D}_{r_a + \beta}(\mathcal{O}_k)$, $\mathcal{P}(\mathbf{x}, (\mathbf{x} - \Pi(\mathbf{x}, \mathcal{O}_k)))$ is a supporting hyperplane to the set $\mathcal{D}_{r_a + \beta}(\mathcal{O}_k)$ at $\mathbf{x}$. Hence, the tangent cone to the set $\partial\mathcal{D}_{r_a + \beta}(\mathcal{O}_k)$ at $\mathbf{x}\in\partial\mathcal{D}_{r_a + \beta}(\mathcal{O}_k)$ is given as $\mathbf{T}_{\partial\mathcal{D}_{r_a + \beta}(\mathcal{O}_k)}(\mathbf{x}) = \mathcal{P}(\mathbf{0}, (\mathbf{x} - \Pi(\mathbf{x}, \mathcal{O}_k))).$ Since $\mathbf{v}(\mathbf{x}, m, k)^\intercal(\mathbf{x} - \Pi(\mathbf{x}, \mathcal{O}_k)) = 0$, $\mathbf{v}(\mathbf{x}(t), m, k)\in\mathbf{T}_{\partial\mathcal{D}_{r_a + \beta}(\mathcal{O}_k)}(\mathbf{x}(t))$, for all $t\geq t_0$, which implies that the solution $\mathbf{x}(t)$ to the flow-only system \eqref{flow_only_rotational_system} belongs to the set $\partial\mathcal{D}_{r_a + \beta}(\mathcal{O}_k)$ for all $t \geq t_0$.
\end{proof}

According to Lemma \ref{lemma:no_close_avoidance}, the solution $\mathbf{x}(t)$ to the flow-only system \eqref{flow_only_rotational_system}, which belongs to the  obstacle-free workspace within the $(r_a + \epsilon)-$neighbourhood of an obstacle $\mathcal{O}_k, k\in\mathbb{I},$ at some time $t_0\geq 0$, will revolve around the obstacle $\mathcal{O}_k$ in the direction decided by the parameter $m\in\{-1, 1\}$, while maintaining the same proximity $d(\mathbf{x}(t), \mathcal{O}_k)$ for all $t\geq t_0$.

Now, for $\mathbf{x}\in\mathcal{S}_3$, according to \eqref{proposed_hybrid_controller_2}, $\mathbf{u}(\mathbf{x}, m, k) = \gamma\mathbf{v}(\mathbf{x}, m, k).$ Since $\mathcal{S}_3\in\mathcal{D}_{r_a + \epsilon}(\mathcal{O}_k)\backslash(\mathcal{D}_{r_a}(\mathcal{O}_k))^{\circ}$, according to Lemma \ref{lemma:no_close_avoidance}, for $\mathbf{x}\in\mathcal{S}_3$, the vector $\mathbf{u}(\mathbf{x}, m, k)$ points in the direction which is tangential to the curve $\partial\mathcal{D}_{d(\mathbf{x}, \mathcal{O}_k)}(\mathcal{O}_k)$ at $\mathbf{x}$. Also, for $\mathbf{x}\in\mathcal{S}_2\cap\mathcal{S}_3$, according to Lemma \ref{lemma:normal_to_dilated_obstacle}, and \eqref{back_region}, the tangent to the set $\partial\mathcal{D}_{d(\mathbf{x}, \mathcal{O}_k)}(\mathcal{O}_k)$ lies along the line $\mathcal{L}(\mathbf{0}, \mathbf{n}_{r_a + \epsilon}^{(-m, k)}$). Moreover, as $\mathcal{S}_3\subset\mathcal{G}_{-m}^k$, $\mathbf{v}(\mathbf{x}, m, k)$ evaluated for $\mathbf{x}\in\mathcal{S}_3$ points in the direction of the vector $\mathbf{x}$, radially outward from the origin. Hence, it is straightforward to notice that for $\mathbf{x}\in\mathcal{S}_3$, $\mathbf{u}(\mathbf{x}, m, k)$ does not point outside the set $\mathcal{F}_m^k$.

According to \eqref{s1}-\eqref{s1s2}, and Lemma \ref{lemma:no_close_avoidance}, if the solution $\mathbf{x}$ to the flow-only system \eqref{flow_only_system} leaves the set $\mathcal{F}_m^k$, it cannot leave from the region $\partial\mathcal{F}_m^k\backslash\mathcal{G}_m^k$. Next, we show that the solution $\mathbf{x}$ to the flow-only system \eqref{flow_only_system}, flowing in the set $\mathcal{F}_m^k$, away from the region $\mathcal{G}_m^k$, will always enter in the set $\mathcal{G}_m^k$ in finite time.

We partition the set $\mathcal{F}_m^k, \; k\in\mathbb{I}$, into three subsets as follows:
\begin{equation}
    \mathcal{F}_m^k = \tilde{\mathcal{F}}_1(m, k) \cup \tilde{\mathcal{F}}_2(m, k)\cup\tilde{\mathcal{F}}_3(m, k),\label{partitions_of_the_flow_set}
    \end{equation}
    where
\begin{equation}
\begin{aligned} 
    \tilde{\mathcal{F}}_1(m, k) &= \left(\mathcal{F}_m^k\cap\mathcal{C}(\mathbf{n}_{r_a+\epsilon}^{(-m, k)}, \mathbf{n}_{r_a + \epsilon_s}^{(m, k)})\right)\cap(\mathcal{D}_{r_a + \epsilon_s}(\mathcal{O}_k))^{\circ},\\
    \tilde{\mathcal{F}}_2(m, k) &= \left(\mathcal{F}_m^k\cap\mathcal{C}(\mathbf{n}_{r_a+\epsilon}^{(-m, k)}, \mathbf{n}_{r_a + \epsilon_s}^{(m, k)})\right)\backslash(\mathcal{D}_{r_a + \epsilon_s}(\mathcal{O}_k))^{\circ},\\
    \tilde{\mathcal{F}}_3(m, k) &= \mathcal{F}_m^k\cap\mathcal{C}(\mathbf{n}_{r_a+\epsilon_s}^{(m, k)}, \mathbf{n}_{r_a + \epsilon_d}^{(m, k)}).
\end{aligned}\label{partitions_of_flow_set}
\end{equation}

For $\mathbf{x}\in\tilde{\mathcal{F}}_3(m, k)$, the control law $\mathbf{u}(\mathbf{x}, m, k)$, according to \eqref{fxim}, is given as $-\gamma\mathbf{x},\;\gamma > 0$, \textit{i.e.}, the solution $\mathbf{x}(t)$ will evolve along the line $\mathcal{L}(\mathbf{0}, \mathbf{x})$ towards the origin. Since $\mathbf{0}\notin\tilde{\mathcal{F}}_3(m, k)$, it is straightforward to verify that if for some $t_0 \geq 0$, $\mathbf{x}(t_0)\in\tilde{\mathcal{F}}_3(m, k)$, then there exists a finite time at which $\mathbf{x}(t)$ will leave the set $\tilde{\mathcal{F}}_3(m, k)$  via $\mathcal{G}_m^k$, see Fig. \ref{lemma_flow_set}.

According to \eqref{fxim}, for $\mathbf{x}\in\tilde{\mathcal{F}}_{2}(m, k)$, $\mathbf{u}(\mathbf{x}, m, k) = -\gamma\mathbf{x}, \gamma > 0$, hence in this region the solution will evolve towards the origin on a straight line $\mathcal{L}(\mathbf{0}, \mathbf{x})$ which implies that eventually it will enter $\tilde{\mathcal{F}}_{1}(m, k)$, see Fig. \ref{lemma_flow_set}.

Now, consider the case where $\mathbf{x} \in\tilde{\mathcal{F}}_{1}(m, k)$. We show that,
\begin{equation}
\forall \mathbf{x}\in\tilde{\mathcal{F}}_{1}(m,k)\backslash\mathcal{G}_m^k, \mathbf{u}(\mathbf{x}, m, k)\in \mathcal{P}_{>}(\mathbf{0}, \nu_m(\mathbf{x})).\label{to_show_11}
\end{equation}
The satisfaction of \eqref{to_show_11} implies that, if at some $t_0\geq 0$, $\mathbf{x}(t_0)\in\tilde{\mathcal{F}}_{1}(m, k)$, the solution $\mathbf{x}$ to the flow-only system \eqref{flow_only_system} cannot live indefinitely in the positive half-space $\mathcal{P}_{>}(\mathbf{0}, \nu_{m}(\mathbf{x}))$ since, in view of Lemma \ref{hybrid_basic_conditions}, the set $\mathcal{F}_m^k$ is closed and bounded as the obstacles are compact.

For all $\mathbf{x}\in\tilde{\mathcal{F}}_{1}(m, k)$, the control input vector $\mathbf{u}(\mathbf{x}, k, m)$ is given by
\begin{equation}
    \mathbf{u}(\mathbf{x}, m, k) = -\gamma\kappa({\xi})\mathbf{x} + \gamma[1 - \kappa(\xi)]\mathbf{v}(\mathbf{x}, m, k),\nonumber
\end{equation}
where $\gamma > 0.$ Since the vectors $-\gamma\kappa(\xi)\mathbf{x}$ and $\nu_{m}(\mathbf{x})$ are orthogonal, $-\gamma\kappa(\xi)\mathbf{x}\in\mathcal{P}(\mathbf{0}, \nu_{m}(\mathbf{x}))$. In order to check whether, $\mathbf{u}(\mathbf{x}, m, k)$  belongs to $\mathcal{P}_{>}(\mathbf{0}, \nu_m(\mathbf{x}))$, let us evaluate
\begin{equation}
    \nu_{m}(\mathbf{x})^\intercal\mathbf{u}(\mathbf{x}, m, k) = \gamma[1 - \kappa(\xi)]\nu_{m}(\mathbf{x})^\intercal\mathbf{v}(\mathbf{x}, m, k),\nonumber
\end{equation}
where $\forall \mathbf{x}\in\tilde{\mathcal{F}}_{1}(m,k)$,  $\gamma > 0, [1 - \kappa(\xi)] > 0.$ Hence, we consider
\begin{equation}
\begin{aligned}
\nu_{m}(\mathbf{x})^\intercal\mathbf{v}(\mathbf{x}, m, k)&=\nu_{-m}(\nu_{m}(\mathbf{x}))^\intercal\nu_{-m}(\mathbf{v}(\mathbf{x}, m, k)),\\
    &= \frac{\norm{\mathbf{x}}\mathbf{x}^{\intercal}(\mathbf{x} - \Pi(\mathbf{x}, \mathcal{O}_k))}{\norm{\mathbf{x} - \Pi(\mathbf{x}, \mathcal{O}_k))}}.\nonumber
\end{aligned}
\end{equation}
Since $\left(\tilde{\mathcal{F}}_{1}(m,k)\backslash\mathcal{G}_m^k\right)\cap\mathcal{R}_b^k=\emptyset$, according to Remark \ref{remark:back_region_is_closed_connected_subset}, $\forall \mathbf{x}\in\tilde{\mathcal{F}}_{1}(m, k)\backslash\mathcal{G}_m^k$, $\mathbf{x}^{\intercal}(\mathbf{x} - \Pi(\mathbf{x}, \mathcal{O}_k)) > 0$, hence \eqref{to_show_11} is satisfied.

\subsection{Proof of Theorem \label{proof:the_main_theorem}\ref{theorem:global_asymptotic_convergence}}

The forward invariance of the obstacle-free set $\mathcal{K}$ defined in \eqref{composite_state_vector}, for the hybrid closed-loop system \eqref{hybrid_closed_loop_system}, is immediate from Lemma \ref{forward_invariance_theorem_1}. We next prove stability of $\mathcal{A}$ using \cite[Definition 7.1]{goedel2012hybrid}. Since $\mathcal{W}_{r_a}$ is compact by construction and $\mathbf{0}\in\left(\mathcal{W}_{r_a}\right)^{\circ}$, $\exists\bar{\delta} > 0, $ such that $\mathcal{B}_{\bar{\delta}}(\mathbf{0})\cap\big(\mathcal{D}_{r_a}(\mathcal{O}_k)\big)^{\circ} = \emptyset, \;\forall k \in\mathbb{I}.$ It can be easily shown that for each $\delta\in[0, \bar{\delta}],$ the set $\mathcal{S}:= \mathcal{B}_{{\delta}}(\mathbf{0})\times\mathbb{M}\times\mathbb{I}$ is forward invariant because $\mathcal{B}_{\delta}(\mathbf{0})$ is disjoint from $\mathcal{J}_0$ as for all $k\in\mathbb{I}$, $\mathcal{D}_{r_a}(\mathcal{O}_k)$ is situated in between the set $\mathcal{J}_0^k$ and the target, see Fig. \ref{jump_and_flow_sets}. Hence, the component $\mathbf{x}\in\mathcal{B}_{\delta}(\mathbf{0})$ of solutions evolves, after at most one jump, in the \textit{move-to-target} mode $\mathbf{\dot{x}} = -\gamma\mathbf{x}, \;\gamma> 0$. Hence, similar to \cite[Appendix 2]{berkane2021arxiv}, the stability of $\mathcal{A}$ for the hybrid closed-loop system \eqref{hybrid_closed_loop_system} is immediate from \cite[Definition 7.1]{goedel2012hybrid}. Next, we proceed to establish almost global convergence properties of the set $\mathcal{A}$.
\begin{figure}
    \centering
    \includegraphics[width = 0.8\linewidth]{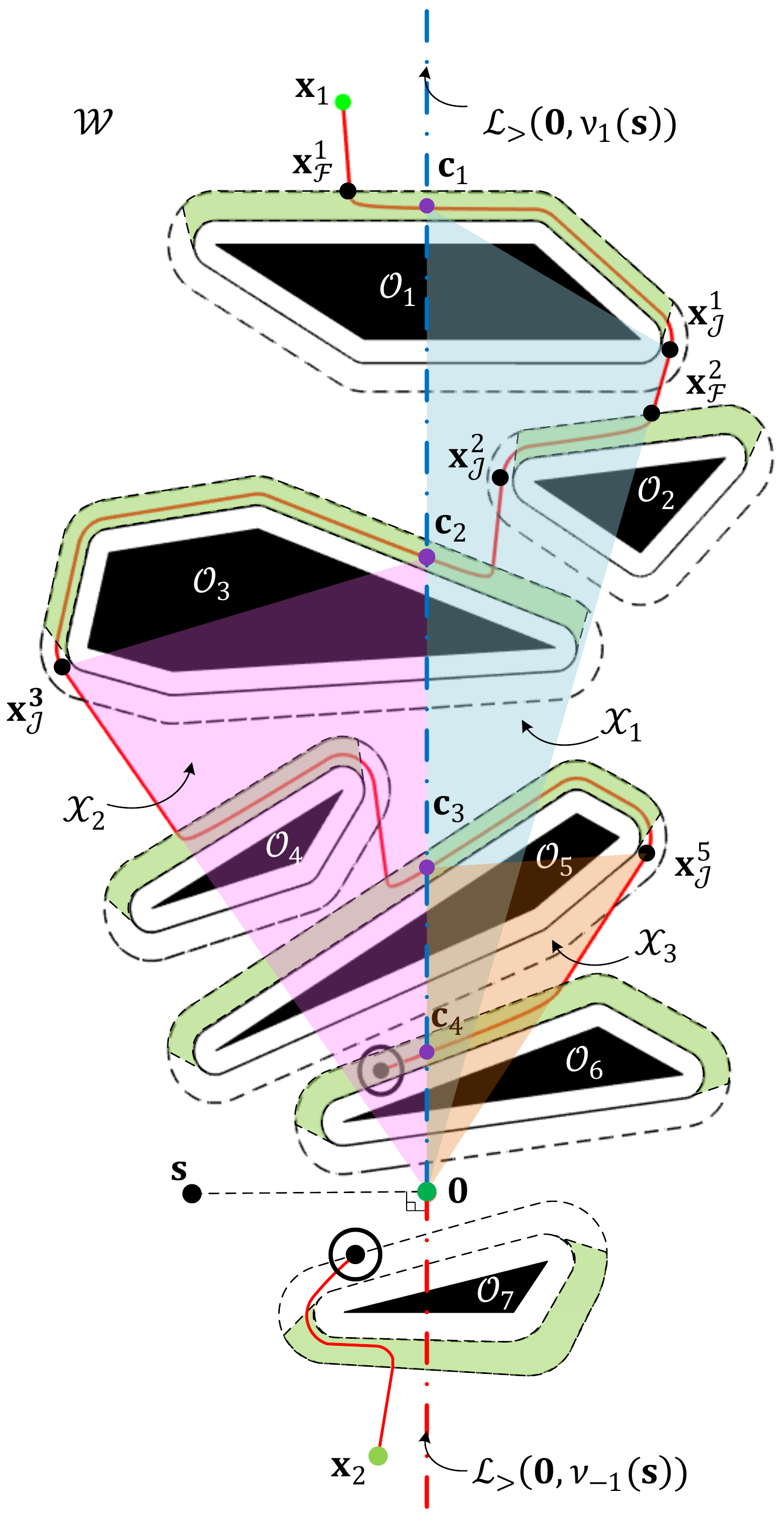}
    \caption{1)  Trajectory $\mathbf{x}(t)$ starting from $\mathbf{x}_1$, illustrates the property stated in \eqref{intersecting_consecutive_points} \textit{i.e.,} with consecutive intersections with the half-line $\mathcal{L}_{>}(\mathbf{0}, \nu_{1}(\mathbf{s}))$, the state $\mathbf{x}$ approaches towards the origin. 2) Trajectory $\mathbf{x}(t)$ starting from $\mathbf{x}_2$, does not intersect with the half-line $\mathcal{L}_{>}(\mathbf{0}, \nu_{-1}(\mathbf{s}))$, as $\mathbf{x}(0, 0)\notin\mathcal{M}_0$ and $m(0, 0)= 0$, as per Lemma \ref{no_revolution_around_the_target}.}
    \label{final_image}
\end{figure}

The next lemma aids in establishing the fact that the solution $\xi$ to the hybrid closed-loop system \eqref{hybrid_closed_loop_system} can enter the set $\mathcal{Z}$ only if it is initialized in the set $\mathcal{Z}_0$.

\begin{lemma}
Consider the hybrid closed-loop system \eqref{hybrid_closed_loop_system} and let Assumption \ref{assumption:robot_pass_through} hold. If $\xi(t_0, j_0) \in\mathcal{K}\backslash\mathcal{Z}_0$ for some $(t_0, j_0)\in\text{dom }\xi$, then $\xi(t, j)\notin\mathcal{Z}_0$ for all $(t, j)\succeq(t_0, j_0)$.\label{never_close_avoidance}
\end{lemma}
\begin{proof}
See Appendix \ref{sec:never_close_avoidance}.
\end{proof}

Lemma \ref{never_close_avoidance} indicated that the solution $\xi$ that does not belong to the set $\mathcal{Z}_0$ at some time $(t_0, j_0)\in\text{dom }\xi$ can never enter in the set $\mathcal{Z}_0$ for all $(t, j)\succeq(t_0, j_0)$. Since $\mathcal{M}\subset\mathcal{M}_0$ \eqref{initial_condition_for_zeno_behaviour}, it is straightforward to conclude that the solution $\xi$ can enter the set $\mathcal{Z}$ only if $\xi(0, 0)\in\mathcal{Z}_0.$

We proceed to prove that all solutions $\xi$ to the hybrid closed-loop system \eqref{hybrid_closed_loop_system} with $\xi(0, 0)\in\mathcal{K}\backslash\mathcal{Z}_0$ converge towards the set $\mathcal{A}$. Towards that end we require the following lemma.

\begin{lemma}
Consider the hybrid closed-loop system \eqref{hybrid_closed_loop_system} and let Assumption \ref{assumption:robot_pass_through} hold. If $\xi(t_0, j_0)\in\tilde{\mathcal{K}}_{m^*}(\{0\})\backslash\mathcal{Z}_0$ with some $m^*\in\{-1, 1\}$ for some $(t_0, j_0)\in\text{ dom }\xi$, then one of the following holds true:
\begin{enumerate}
    \item $\xi(t, j)\notin\tilde{\mathcal{K}}_{>}(\mathbb{M})$ for all $(t, j)\succeq (t_0, j_0)$, and $\underset{t\to\infty, j\to\infty}{\text{ lim }}\xi(t, j) \to \mathcal{A}.$
    \item there exists $(T, J)\succeq(t_0, j_0)$ such that $\xi(T, J) \in\tilde{\mathcal{K}}_{>}(\{-1, 1\})\backslash\mathcal{Z}_0$ where $T + J < +\infty$.
\end{enumerate}
\label{eventually_it_crosses_line}
\end{lemma}
\begin{proof}
See Appendix \ref{sec:eventuallY_crosses_line}.
\end{proof}


Lemma \ref{eventually_it_crosses_line} shows that for the solution $\xi$ operating in the \textit{move-to-target} mode within the obstacle-free workspace, if the $\mathbf{x}$ component of the solution belongs to the interior of one of the half-spaces generated by the hyperplane $\mathcal{P}(\mathbf{0}, \mathbf{s})$ away from the set $\mathcal{M}_0$, then the solution $\xi$ will either directly converge towards the set $\mathcal{A}$ or enter the set $\tilde{\mathcal{K}}_{>}(\mathbb{M})\cup\tilde{\mathcal{K}}_{-m^*}(\mathbb{M})$ only through the set $\tilde{\mathcal{K}}(\{-1, 1\})\backslash\mathcal{Z}_0$ in finite time, while flowing in the \textit{obstacle-avoidance} mode. For the former case, it is straightforward to establish the convergence of the solution $\xi$ to the set $\mathcal{A}$. Therefore, focus on the latter case.

We show that for the solution $\xi$, initialized in the set $\mathcal{K}\backslash\mathcal{Z}_0$, if $\mathbf{x}$ component of the solution intersects the half-line $\mathcal{L}_{>}(\mathbf{0}, \nu_{1}(\mathbf{s}))$ more than once, then with each consecutive intersection with the set $\tilde{\mathcal{K}}_{>}(\{-1, 1\})\backslash\mathcal{Z}_0$, the solution $\xi$ moves closer to the set $\mathcal{A}.$

Let $\mathbf{c}_l\in\mathcal{W}_{r_a},$ be the location where the solution $\xi$ enters in the set $\tilde{\mathcal{K}}_{>}(\{-1, 1\})\backslash\mathcal{Z}_0$, as shown in Fig. \ref{final_image}. The parameter $l\in\mathbb{N}$ indicates the instance of occurrence, for example $\mathbf{c}_1\in\mathcal{W}_{r_a}$ represents the location where the solution $\xi$, flowing in the \textit{obstacle-avoidance} mode, first entered in the set $\tilde{\mathcal{K}}_{>}(\{-1, 1\})\backslash\mathcal{Z}_0$. Let $(t_l^c, j_l^c)\in\text{ dom }\xi$ such that $\xi(t_l^c, j_l^c) \in \mathbf{c}_l\times \{-1, 1\}\times\mathbb{I}$.

We assume $\xi(t_0, j_0)\in\left(\tilde{\mathcal{K}}_{m^*}(\mathbb{M})\backslash\mathcal{Z}_0\right)\cap((\mathcal{F}_0^*\backslash\mathcal{J}_0^*)\times\{0\}\times\mathbb{I})$ for some $(t_0, j_0)\in\text{ dom }\xi$ and $m^*\in\{-1, 1\}$, which can always be the case by virtue of Lemma \ref{eventually_exit_avoidance_mode} and Lemma \ref{never_close_avoidance}. Furthermore, assume that for some $i\in\mathbb{I}$, $\exists(t_{\mathcal{F}}^i, j_{\mathcal{F}}^i)\succ(t_0, j_0)$ such that, according to \eqref{update_law_for_m},\eqref{direction_decision} and \eqref{update_law_for_k}, $\xi(t_{\mathcal{F}}^i, j_{\mathcal{F}}^i) = (\mathbf{x}_{\mathcal{F}}^i, m^*, i)\in\tilde{\mathcal{K}}_{m^*}(\mathbb{M})\backslash\mathcal{Z}_0$. Hence, according to Lemma \ref{eventually_exit_avoidance_mode} and Lemma \ref{never_close_avoidance}, $\exists(t_{\mathcal{J}}^i, j_{\mathcal{J}}^i)\succ(t_{\mathcal{F}}^i, j_{\mathcal{F}}^i)$ such that $\xi(t_{\mathcal{J}}^i, j_{\mathcal{J}}^i) = (\mathbf{x}_{\mathcal{J}}^i, m^*, i)\in\mathcal{K}\backslash\mathcal{Z}_0$. We assume $\xi(t_{\mathcal{J}}^i, j_{\mathcal{J}}^i)\in\tilde{\mathcal{K}}_{-m^*}(\mathbb{M})$ \textit{i.e.}, $\exists(t_l^c, j_l^c)\in((t_{\mathcal{F}}^i, t_{\mathcal{J}}^i]\times[j_{\mathcal{F}}^i, j_{\mathcal{J}}^i])$. This case is similar to the evolution of the robot position around obstacle $\mathcal{O}_1$ in Fig. \ref{final_image}. After $(t_{\mathcal{J}}^i, j_{\mathcal{J}}^i)$, according to Lemma \ref{eventually_it_crosses_line}, the solution will either directly converge to the set $\mathcal{A}$, or will again enter the set $\tilde{\mathcal{K}}_{>}(\{-1, 1\})\backslash\mathcal{Z}_0$.  We assume $\exists(t_{l + 1}^c, j_{l + 1}^c)\succ(t_{\mathcal{J}}^i, j_{\mathcal{J}}^i)$, then according to \eqref{update_law_for_m},\eqref{direction_decision} and \eqref{update_law_for_k}, $\xi(t_{l+ 1}^c, j_{l+1}^c) = (\mathbf{c}_{l + 1}, -m^*, i^{\prime}), i^{\prime}\in\mathbb{I}_l^{l+1}\backslash\{i\}$, where the set $\mathbb{I}_{l}^{l + 1}\subset\mathbb{I}$ consists of the indices of the obstacles encountered by the solution $\xi(t, j)$, for all $(t, j) \in ([t_l^c, t_{l+1}^c]\times[j_{l}^c, j_{l + 1}^c])$. 

We show that $\exists\sigma^{c_l}>0$ such that
\begin{align}
    \norm{\mathbf{c}_{l+1}(t_{l+1}^c, j_{l+1}^c)}\leq \norm{\mathbf{c}_{l}(t_{l}^c, j_{l}^c)} - \sigma^{c_l}.\label{intersecting_consecutive_points}
\end{align}
The satisfaction of the \eqref{intersecting_consecutive_points} ensures that if the component $\mathbf{x}$ of the solution $\xi$ crosses the half-line $\mathcal{L}_{>}(\mathbf{0}, \nu_{1}(\mathbf{s}))$ more than once while operating in the \textit{obstacle-avoidance} mode, then each consecutive intersection with the half-line $\mathcal{L}_{>}(\mathbf{0}, \nu_{1}(\mathbf{s}))$ is closer to the origin than the previous one. Then, by virtue of Lemma \ref{eventually_it_crosses_line} and the satisfaction of \eqref{intersecting_consecutive_points}, it is straightforward to show that any solution $\xi$ which belongs to $\mathcal{K}\backslash\mathcal{Z}_0$, at some instant of time will converge towards the target set $\mathcal{A}$. 

For $\xi(t_l^c, j_l^c) = (\mathbf{c}_l, m^*, i)$, we define a set $\mathcal{KT}_l:=  \mathcal{X}_l\times\mathbb{M}\times\mathbb{I},$ where
the set $\mathcal{X}_l$ is defined as
    \begin{equation*}
    \mathcal{X}_l := \mathcal{C}(\nu_{1}(\mathbf{s}), \mathbf{x}_{\mathcal{J}}^i)\cap\mathcal{P}_{\geq}(\mathbf{c}_l, \nu_{m^*}(\mathbf{x}_{\mathcal{J}}^{i} - \mathbf{c}_l)), \label{the_domain_where_solution_evolves}
\end{equation*}
and show that $\xi(t, j)\in\mathcal{KT}_l, \;\forall(t, j)\in([t_{\mathcal{J}}^i, t_{l+1}^c]\times [j_{\mathcal{J}}^i, j_{l+1}^c]).$ This would imply that for all $(t, j)\in([t_{\mathcal{J}}^i, t_{l+1}^c]\times [j_{\mathcal{J}}^i, j_{l+1}^c])$, the solution $\xi$ evolves in the set $\mathcal{KT}_l$ until it enters $\tilde{\mathcal{K}}_{>}(\{-1, 1\})\backslash\mathcal{Z}_0$ according to Lemma \ref{eventually_it_crosses_line}, and \eqref{intersecting_consecutive_points} holds.

First, we consider a special scenario wherein $\mathbf{x}_{\mathcal{J}}^i = \mathbf{c}_l$. In this case, the set $\mathcal{X}_l$ can be represented by a line segment $\mathcal{L}_s(\mathbf{0}, \mathbf{x}_{\mathcal{J}}^i)$. At $(t_{\mathcal{J}}^i, j_{\mathcal{J}}^i + 1)$ the robot enters in the \textit{move-to-target} flow set and moves along the line segment $\mathcal{L}_s(\mathbf{0}, \mathbf{x}_{\mathcal{J}}^i)$ towards the origin. Then it is straightforward to verify \eqref{intersecting_consecutive_points}. Next, we consider the case wherein $\mathbf{x}_{\mathcal{J}}^i \ne \mathbf{c}_l$.

Since $\xi(t_0, j_0)\in\left(\tilde{\mathcal{K}}_{m^*}(\mathbb{M})\backslash\mathcal{Z}_0\right)\cap((\mathcal{F}_0^*\backslash\mathcal{J}_0^*)\times\{0\}\times\mathbb{I})$, according to Lemma \ref{never_close_avoidance}, the solution $\xi(t, j)\notin\mathcal{Z}_0, \forall(t, j)\succeq(t_0, j_0).$ The boundary of the set $\mathcal{X}_l$ is defined as
\begin{equation*}
    \partial\mathcal{X}_l = \mathcal{L}_s(\mathbf{0}, \mathbf{c}_l)\cup\mathcal{L}_s(\mathbf{0}, \mathbf{x}_{\mathcal{J}}^i)\cup\mathcal{L}_s(\mathbf{c}_l, \mathbf{x}_{\mathcal{J}}^i),
\end{equation*}
where $\mathcal{L}_s(\mathbf{c}_l, \mathbf{x}_{\mathcal{J}}^i)\subset\mathcal{D}_{r_a + \epsilon_d}(\mathcal{O}_i)$. According to Assumption \ref{assumption:robot_pass_through} and \eqref{avoidance_final_set}, $\mathcal{F}_{m}^q\cap\mathcal{F}_{m}^s = \emptyset, \forall q, s\in\mathbb{I}, q\ne s,$ for some $m\in\{-1, 1\}$. Hence, at $(t_{\mathcal{J}}^i, j_{\mathcal{J}}^i + 1)$, the solution $\xi$ starts to flow in the \textit{move-to-target} mode and $\mathbf{x}$ component of the solution evolves along the line segment $\mathcal{L}_s(\mathbf{x}_{\mathcal{J}}^i, \mathbf{0})$. Since $\mathcal{L}_s(\mathbf{c}_l, \mathbf{x}_{\mathcal{J}}^i)\cap\mathcal{L}_s(\mathbf{x}_{\mathcal{J}}^i, \mathbf{0}) = \mathbf{x}_{\mathcal{J}}^i$, the $\mathbf{x}$ component of the solution cannot enter the line segment $\mathcal{L}_s(\mathbf{c}_l, \mathbf{x}_{\mathcal{J}}^i)$ from within the set $\mathcal{KT}_l$ using the stabilizing feedback $-\gamma\mathbf{x}, \;\gamma> 0$.

At $(t_{\mathcal{J}}^i, j_{\mathcal{J}}^i + 1)$, the solution $\xi$ starts to evolve towards the origin along the line $\mathcal{L}(\mathbf{0}, \mathbf{x}_{\mathcal{J}}^i)$ in the \textit{move-to-target} mode. This solution (\textit{i.e.,} flowing in the \textit{move-to-target} mode) cannot leave the line segment $\mathcal{L}_s(\mathbf{0}, \mathbf{x}_{\mathcal{J}}^i)$ unless it encounters $\mathcal{J}_0\times\{0\}\times\mathbb{I}.$ Let us assume $\exists(t_{\mathcal{F}}^{i+1}, j_{\mathcal{F}}^{i+1})\succ(t_{\mathcal{J}}^i, j_{\mathcal{J}}^i)$ for some $i+1\in\mathbb{I}_l^{l+1}\backslash\{i, i^{\prime}\}$, hence, according to \eqref{update_law_for_m}, \eqref{direction_decision} and \eqref{update_law_for_k}, $\xi_{\mathcal{F}}^{i+1} = (\mathbf{x}_{\mathcal{F}}^{i+1}, -m^*, i+1)\in\mathcal{F}_{-m^*}^{i +1}\times\{-1, 1\}\times\mathbb{I}$. At this instance, since $\mathbf{x}_{\mathcal{F}}^{i+1}\in\overline{\tilde{\mathcal{F}}_1(-m^*, i+1)}$, defined in \eqref{partitions_of_the_flow_set}, \eqref{partitions_of_flow_set}, according to \eqref{to_show_11}, $\mathbf{x}(t, j)\in\mathcal{P}_{>}(\mathbf{0}, \nu_{-m^*}(\mathbf{x}_{\mathcal{F}}^{i+1})), \forall(t, j)\in((t_{\mathcal{F}}^{i+1}, t_{\mathcal{J}}^{i+1}], [j_{\mathcal{F}}^{i+1}, j_{\mathcal{J}}^{i+1}])$, and the $\mathbf{x}$ component of the solution $\xi$ enters in the interior of the set $\mathcal{X}_l$. The solution $\xi$ does not enter the set $\mathcal{L}_s(\mathbf{0}, \mathbf{x}_{\mathcal{J}}^i)\times\mathbb{M}\times\mathbb{I}, \forall(t, j)\in((t_{\mathcal{F}}^{i+1}, t_{\mathcal{J}}^{i+1}], [j_{\mathcal{F}}^{i+1}, j_{\mathcal{J}}^{i+1}])$.

Now, if $\mathbb{I}_{l}^{l+1}\backslash\{i, i + 1\}\ne \emptyset$, then for each $i^{\prime\prime}\in\mathbb{I}_{l}^{l+1}\backslash\{i, i + 1\}$, according to \eqref{direction_decision}, $m(t, j) = -m^*, \forall(t, j)\in([t_{\mathcal{F}}^{i^{\prime\prime}}, t_{\mathcal{J}}^{i^{\prime\prime}}]\times[j_{\mathcal{F}}^{i^{\prime\prime}}, j_{\mathcal{J}}^{i^{\prime\prime}}])$. Hence, if the solution $\xi(t, j)$ enters in the \textit{obstacle-avoidance} mode at any $(t, j)\in([t_{\mathcal{J}}^{i+1}, t_{l+1}^c]\times[j_{\mathcal{J}}^{i+1}, j_{l+1}^c])$, it will evolve in the interior of the set $\mathcal{KL}_l$. Hence, as per Lemma \ref{eventually_it_crosses_line}, it follows that the solution can only leave the set $\mathcal{KT}_l$ through $\mathcal{L}_s(\mathbf{0}, \mathbf{c}_l)\backslash\{\mathbf{c}_l\}$, which ensures that there exists some $\sigma^{c_l} > 0$, such that \eqref{intersecting_consecutive_points} is satisfied. Hence, every solution starting in  $\mathcal{K}\backslash\mathcal{Z}_0$ will converge to $\mathcal{A}$.

Finally, if we remove the jump set $\mathcal{J}$ from the flow set $\mathcal{F}$ to obtain the hybrid system with data $(\mathcal{F}\setminus\mathcal{J},\mathbf{F},\mathcal{J},\mathbf{J})$, and thus forcing the flows over jumps \cite{SANFELICE2010239}, the Zeno solution starting from $\mathcal{Z}_0$ is no longer a valid solution for the closed-loop system with these new data. In fact, for $\xi(0,0)\in\mathcal{Z}_0$, the solution will flow with the \textit{obstacle-avoidance} mode until it reaches $\mathcal{Z}$ and then flows with the \textit{move-to-target} mode afterwards. 

\subsection{Proof of lemma \ref{never_close_avoidance}}
\label{sec:never_close_avoidance}

According to \eqref{stabilization_mode_jumpflow_set_final} and \eqref{initial_condition_for_zeno_behaviour}, as $\mathcal{M}_0\subset\mathcal{J}_0$, the solution $\xi$ with $\xi(t_0, j_0)\in\mathcal{K}\backslash\mathcal{Z}_0$, cannot enter the set $\mathcal{Z}_0$ $\forall(t, j)\succeq(t_0, j_0)$, while flowing in the \textit{move-to-target} mode.

We consider the flow-only system \eqref{flow_only_system}, where $\mathbf{x}\in\mathcal{F}_m^k, m\in\{-1, 1\}, k\in\mathbb{I}$ \textit{i.e.}, the case wherein the solution is flowing in the \textit{obstacle-avoidance} mode, in the vicinity of an obstacle $\mathcal{O}_k, k\in\mathbb{I}$, and show that if $\mathbf{x}(t_0)\in\mathcal{F}_m^k\backslash\mathcal{M}_0$, then $\mathbf{x}(t)\notin\mathcal{M}_0, \forall t\geq t_0$. Since, according to Assumption \ref{assumption:robot_pass_through}, \eqref{individual_jump_set_avoidance_mode} and \eqref{avoidance_final_set}, the flow sets of the \textit{obstacle-avoidance} modes for the state $\mathbf{x}$, related to different obstacles, are disjoint \textit{i.e.}, $\mathcal{F}_m^p\cap\mathcal{F}_m^q = \emptyset, p, q\in\mathbb{I}, p\ne q$, one can repeat the analysis for the solution evolving in the flow set \textit{obstacle-avoidance} mode, related to the remaining obstacles.

Assume that $\mathbf{x}(t_0)\in\mathcal{F}_m^k$ such that $d(\mathbf{x}(t_0), \mathcal{D}_{r_a}(\mathcal{O}_k)) = \beta_1 \in(0, \epsilon_d]$. Let $\beta_2 = \min\{\epsilon, \beta_1\}.$ According to \eqref{avoidance_flow_set_individual} and \eqref{initial_condition_for_zeno_behaviour}, it is easy to see that the solution can enter the set $\mathcal{M}_0$ only via $\partial\mathcal{D}_{r_a + \beta_3}(\mathcal{O}_k)\cap\mathcal{F}_m^k$, where $\beta_3\in(0, \beta_2)$. For all $\mathbf{x}\in\partial\mathcal{D}_{r_a + \beta_3}(\mathcal{O}_k)\cap\mathcal{F}_m^k$, according to \eqref{proposed_hybrid_controller_2}, the control input $\mathbf{u}(\mathbf{x}, m, k) = \gamma\mathbf{v}(\mathbf{x}, m, k), \gamma > 0$. Hence, according to Lemma \ref{lemma:no_close_avoidance}, the solution $\mathbf{x}(t)$ to the flow-only system \eqref{flow_only_system} cannot enter the set $\mathcal{M}_0$ for all $t \geq t_0$, and according to Lemma \ref{eventually_exit_avoidance_mode}, it will ultimately enter in the \textit{move-to-target} in finite time.

\subsection{Proof of lemma \ref{eventually_it_crosses_line}}
\label{sec:eventuallY_crosses_line}

We consider $\xi(t_0, j_0)\in\left(\tilde{\mathcal{K}}_{m^*}(\mathbb{M})\backslash\mathcal{Z}_0\right)\cap((\mathcal{F}_0^*\backslash\mathcal{J}_0^*)\times\{0\}\times\mathbb{I})$ where $m^*\in\{-1, 1\}$ for some $(t_0, j_0)\in\text{ dom }\xi$, which can always be the case by virtue of Lemma \ref{eventually_exit_avoidance_mode} and Lemma \ref{never_close_avoidance}. Also, according to Lemma \ref{never_close_avoidance}, $\xi(t, j)\notin\mathcal{Z}_0, \forall(t, j)\succeq(t_0, j_0)$. Then at $(t_0, j_0)$, the state $\mathbf{x}$ will evolve along the line $\mathcal{L}(\mathbf{0}, \mathbf{x}(t_0, j_0))$ in the \textit{move-to-target} mode. Since $\mathcal{L}(\mathbf{0}, \mathbf{x}(t_0, j_0))\cap(\mathcal{P}(\mathbf{0}, \mathbf{s})\backslash\mathbf{0}) = \emptyset$, the solution cannot enter the set $\tilde{\mathcal{K}}_0(\mathbb{M})$ from $\tilde{\mathcal{K}}_{m^*}(\mathbb{M})$ while operating in the \textit{move-to-target} mode. Moreover, according to Lemma \ref{no_revolution_around_the_target}, the solution will never enter the set $\tilde{\mathcal{K}}_{<}(\mathbb{M})$ for all $(t, j)\succeq(t_0, j_0).$ Hence, according to Lemma \ref{eventually_it_crosses_line}, the only remaining possibilities, which we need to prove, are that the solution $\xi$ with $\xi(t_0, j_0)\in\tilde{\mathcal{K}}_{m^*}(\mathbb{M})$ will either enter the set $\tilde{\mathcal{K}}_{>}(\{-1, 1\})\backslash\mathcal{Z}_0$ while flowing in the \textit{obstacle-avoidance} mode or directly converge towards the set $\mathcal{A}$ while operating with the \textit{move-to-target} mode without entering the set $\tilde{\mathcal{K}}_{>}(\mathbb{M})$.

Let $\alpha_{-1}(\mathbf{a}, \mathbf{b})$ and $\alpha_{1}(\mathbf{a}, \mathbf{b})$ denote the absolute values of an angle measured from vector $\mathbf{a}$ to vector $\mathbf{b}$ in the counter-clockwise and clockwise directions, respectively. If $\mathbf{\xi}(t, j)\notin \mathcal{J}_0\times\{0\}\times\mathbb{I}, \;\forall (t, j)\succeq(t_0, j_0)$, then the $\mathbf{x}$ component of the solution $\xi$, under the influence of the stabilizing vector $-\gamma\mathbf{x}, \gamma > 0$, will asymptotically converge towards the origin.

On the other hand, assume that the solution $\xi$ encounters the jump set $\mathcal{J}_0^i\times\{0\}\times\mathbb{I}$ for some $i\in\mathbb{I}$ \textit{i.e.},  $\exists(t_{\mathcal{F}}^{i}, j_{\mathcal{F}}^i)\in\text{ dom }\xi, (t_{\mathcal{F}}^{i}, j_{\mathcal{F}}^i)\succ(t_0, j_0)$. According to \eqref{update_law_for_m},\eqref{direction_decision} and \eqref{update_law_for_k}, $\xi_{\mathcal{F}}^i = (\mathbf{x}_{\mathcal{F}}^i, m^*, i)$. Then according to Lemma \ref{eventually_exit_avoidance_mode}, $\exists(t_{\mathcal{J}}^i, j_{\mathcal{J}}^i )\succ(t_{\mathcal{F}}^i, j_{\mathcal{F}}^i)$ such that $\xi_{\mathcal{J}}^i = (\mathbf{x}_{\mathcal{J}}^i, m^*, i)$. According to Lemma \ref{never_close_avoidance}, the solution cannot enter the set $\mathcal{Z}_0$, hence the locations $\mathbf{x}_{\mathcal{F}}^i$ and $\mathbf{x}_{\mathcal{J}}^i$ belong to the set $(\overline{\tilde{\mathcal{F}}_{1}(m^*, i)}\backslash\mathcal{M}_0)\subset\mathcal{F}_{m^*}^i$, defined in \eqref{partitions_of_the_flow_set}, \eqref{partitions_of_flow_set}. Then, according to  \eqref{to_show_11}, $\mathbf{x}_{\mathcal{J}}^i\in\mathcal{P}_{>}(\mathbf{0}, \nu_{m^*}(\mathbf{x}_{\mathcal{F}}^i))$. Assuming $\xi(t_{\mathcal{J}}^i, j_{\mathcal{J}}^i)\in\tilde{\mathcal{K}}_{m^*}(\mathbb{M})$, one has
\begin{equation}
    \alpha_{m^*}(\mathbf{x}_{\mathcal{F}}^i, \nu_{1}(\mathbf{s})) > \alpha_{m^*}(\mathbf{x}_{\mathcal{J}}^i, \nu_{1}(\mathbf{s})) > 0.\label{moving_closer_to_the_line_1}
\end{equation}

Hence, for any solution $\xi$ to the hybrid closed-loop system \eqref{hybrid_closed_loop_system} with $\xi(t_0, j_0)\in\left(\tilde{\mathcal{K}}_{m^*}(\mathbb{M})\backslash\mathcal{Z}_0\right)\cap(\mathcal{F}_0^*\backslash\mathcal{J}_0^*)\times\{0\}\times\mathbb{I}$ for some $(t_0, j_0)\in\text{ dom }\xi$, if there exists $(t_{\mathcal{F}}^i, j_{\mathcal{F}}^i)\succ(t_0, j_0)$ for some $i\in\mathbb{I}$ with $\xi_{\mathcal{J}}^i\in\tilde{\mathcal{K}}_{m^*}(\mathbb{M})$, then the angle between the vectors $\mathbf{x}$ and $\nu_1(\mathbf{s})$ \textit{i.e.}, $\alpha_{m^*}(\mathbf{x}(t, j), \nu_{1}(\mathbf{s}))$ reduces, otherwise, if there does not exist $(t_{\mathcal{F}}^i, j_{\mathcal{F}}^i)\succ(t_0, j_0)$ for any $i\in\mathbb{I}$, \textit{i.e.}, if the solution does not encounter the \textit{move-to-target} mode jump set after $(t_0, j_0)$, then it will asymptotically converge to the target set $\mathcal{A}$, under the influence of the stabilizing control input $-\gamma\mathbf{x}, \gamma> 0$.

Next, assume that the solution $\xi$ again encounters the jump set $\mathcal{J}_0\times\{0\}\times\mathbb{I}$ for $i+1\in\mathbb{I}\backslash\{i\}$. Hence, $\exists(t_{\mathcal{F}}^{i+1}, j_{\mathcal{F}}^{i +1})\succ(t_{\mathcal{J}}^{i}, j_{\mathcal{J}}^i)$ and according to Lemma \ref{eventually_exit_avoidance_mode}, $\exists(t_{\mathcal{J}}^{i+1}, j_{\mathcal{J}}^{i +1})\succ(t_{\mathcal{F}}^{i + 1}, j_{\mathcal{F}}^{i + 1})$ such that $\xi_{\mathcal{F}}^{i+1} = (\mathbf{x}_{\mathcal{F}}^{i+1}, i+1, m^*)$ and $\xi_{\mathcal{J}}^{i + 1} = (\mathbf{x}_{\mathcal{J}}^{i + 1}, {i + 1}, m^*)$. Again assume $\xi_{\mathcal{J}}^{i + 1}\in\tilde{\mathcal{K}}_{m^*}(\mathbb{M})$.  Hence, similar to the previous case,
\begin{equation}
        \alpha_{m^*}(\mathbf{x}_{\mathcal{F}}^{i + 1}, \nu_{1}(\mathbf{s})) > \alpha_{m^*}(\mathbf{x}_{\mathcal{J}}^{i + 1}, \nu_{1}(\mathbf{s})) > 0.\label{moving_closer_to_the_line_2}
\end{equation}
Also, as the $\mathbf{x}$ component of the solution $\xi$ evolved on a straight line towards the origin while operating in the \textit{move-to-target} mode, $\forall (t, j)\in([t_{\mathcal{J}}^i, t_{\mathcal{F}}^{i+1}]\times[j_{\mathcal{J}}^i, j_{\mathcal{F}}^{i+1}])$, $\alpha_{m^*}( \mathbf{x}(t, j), \nu_1(\mathbf{s})) = \alpha_{m^*}( \mathbf{x}_{\mathcal{J}}^i, \nu_1(\mathbf{s}))$. As a result, one has
\begin{equation*}
        \alpha_{m^*}(\mathbf{x}_{\mathcal{F}}^{i}, \nu_{1}( \mathbf{s})) > \alpha_{m^*}(\mathbf{x}_{\mathcal{F}}^{i + 1}, \nu_{1}(\mathbf{s}))> \alpha_{m^*}(\mathbf{x}_{\mathcal{J}}^{i + 1}, \nu_{1}(\mathbf{s})) > 0
\end{equation*}
The angle $\alpha_{m^*}(\nu_{1}(\mathbf{s}), \mathbf{x}) = 0$ implies that the state $\xi\in\tilde{\mathcal{K}}_{0}(\mathbb{\mathbb{M}})\backslash\tilde{\mathcal{K}}_{<}(\mathbb{M}).$ 

This implies that with each hybrid sequence of jumps from the \textit{move-to-target} mode to the \textit{obstacle-avoidance} and \textit{vice versa}, the solution $\xi$, operating in the set $\tilde{\mathcal{K}}_{m^*}(\{0\})\backslash\mathcal{Z}_0$, evolve towards the set $\tilde{\mathcal{K}}(\mathbb{M})\cup\tilde{\mathcal{K}}_{-m^*}(\mathbb{M})$, in the sense that the angle between the vectors $\mathbf{x}$ and $\nu_{1}(\mathbf{s})$ \textit{i.e.}, $\alpha_{m^*}(\mathbf{x}, \nu_1(\mathbf{s}))$ decreases. Also, according to Lemma \ref{eventually_exit_avoidance_mode}, the solution $\xi$, which is operating in the \textit{obstacle-avoidance} mode, always enter in the \textit{move-to-target} mode in finite time. Hence, it can be concluded that the solution, which belongs to the set $\tilde{\mathcal{K}}_{m^*}(\{0\}), m^*\in\{-1, 1\},$ at some time, either directly converges towards the set $\mathcal{A}$ or intersects the set $\tilde{\mathcal{K}}_{>}(\{-1, 1\})$ in finite time. 

\subsection{Proof of Proposition \ref{proposition:continuous_control}}\label{proof:proposition_continuous}
According to Lemma \ref{hybrid_basic_conditions}, the control input $\mathbf{u}(\mathbf{x}, m, k)$ in \eqref{proposed_hybrid_controller_2} is continuous while the robot is operating not only in the \textit{move-to-target} mode \textit{i.e.}, when $(\mathbf{x}, m, k)\in\mathcal{F}_0\times\{0\}\times\mathbb{I}$ but also in the \textit{obstacle-avoidance} mode \textit{i.e.}, when $(\mathbf{x}, m, k)\in\mathcal{F}_{z}\times\{z\}\times\mathbb{I}, \;z\in\{-1, 1\}.$ We only need to verify the continuity of the control input $\mathbf{u}(\xi)$ at instances when the solution $\xi$ to the hybrid closed-loop system \eqref{hybrid_closed_loop_system} leaves the \textit{move-to-target} mode and enters the \textit{obstacle-avoidance} mode, and \textit{vice versa}.

Note that since $\xi(t_0, j_0)\in\mathcal{K}\backslash\mathcal{Z}_0$ for some $(t_0, j_0)\in\text{dom }\xi$, according to Lemma \ref{never_close_avoidance}, the solution $\xi$ cannot enter the set $\mathcal{Z}_0$ for all $(t, j)\succeq(t_0, j_0)$, and hence cannot get stuck in the Zeno behaviour for all future times.

During the \textit{move-to-target} mode, the state $\mathbf{x}$ evolves along the line joining the center of the robot and the origin. Hence, as can be observed from Fig. \ref{jump_and_flow_sets}, for the robot operating in the \textit{move-to-target} mode, a solution $\xi$ can enter in the jump set of the \textit{move-to-target} mode for some obstacle $\mathcal{O}_i, \;i\in\mathbb{I}$, only via the region $\left(\partial\mathcal{D}_{r_a + \epsilon_s}(\mathcal{O}_i)\cap \mathcal{J}_0^i\right)\times\{0\}\times\{i\}$. Let $(t_0, j_0)\in\text{ dom }\xi$ such that $\xi(t_0, j_0) \in \big(\partial\mathcal{D}_{r_a + \epsilon_s}(\mathcal{O}_i)\cap \mathcal{J}_0^i\big)\times\{0\}\times\{i\}$. Hence, according to \eqref{proposed_hybrid_controller_2}, the control input vector at $(t_0, j_0)$ is given as
\begin{equation}
    \mathbf{u}(\xi(t_0, j_0)) = -\gamma\mathbf{x}(t_0, j_0).
\end{equation}
According to \eqref{update_law_for_m}, $\xi(t_0, j_0 + 1) \in \big(\partial\mathcal{D}_{r_a + \epsilon_s}(\mathcal{O}_i)\cap \mathcal{J}_0^i\big)\times\{-1, 1\}\times\{i\}$, and the control input $\mathbf{u}(\xi(t_0, j_0 + 1))$, according to \eqref{proposed_hybrid_controller_2}-\eqref{beta_function_definition}, is given as
\begin{equation}
    \mathbf{u}(\xi(t_0, j_0 + 1)) = -\gamma\mathbf{x}(t_0, j_0 + 1).
\end{equation}
Since, according to \eqref{hybrid_closed_loop_system}, $\mathbf{x}(t_0, j_0 + 1) = \mathbf{x}(t_0, j_0)$, when the solution leaves the \textit{move-to-target} mode and enters in the \textit{obstacle-avoidance} mode, the control vector trajectories remain continuous.

Next, we consider the case where the robot operating in the \textit{obstacle-avoidance} mode enters in the \textit{move-to-target} mode. According to Lemma \ref{eventually_exit_avoidance_mode}, the component $\mathbf{x}$ of the solutions, evolving in the \textit{obstacle-avoidance} mode in the flow set $\mathcal{F}_z^i, i\in\mathbb{I}$, for some $z\in\{-1, 1\}$, will eventually leave the \textit{obstacle-avoidance} mode via the gate region $\mathcal{G}_z^i$. Let $\xi(t_1, j_1)\in\mathcal{F}_z\times\{z\}\times\mathbb{I}, \;z\in\{-1,1\}$ for some $(t_1, j_1)\in\text{ dom }\xi$, then according to Lemma \ref{eventually_exit_avoidance_mode}, $\exists(t_2, j_1)\succeq(t_1, j_1)$ such that $\mathbf{x}(t_2, j_1)\in\mathcal{G}_z^i$. Then according to \eqref{individual_jump_set_avoidance_mode} and \eqref{update_law_for_m}, at $(t_2, j_1 + 1)$ the solution enters in the \textit{move-to-target} mode flow set \textit{i.e.}, $\xi(t_2, j_1 + 1)\in\mathcal{F}_0\times\{0\}\times\mathbb{I}$. Hence, at $(t_2, j_1)$, the control input vector $\mathbf{u}(\xi(t_2, j_1))$ is evaluated as
\begin{equation}
\begin{aligned}
    \mathbf{u}(\xi(t_2, &j_1)) = -\gamma\kappa(\xi(t_2, j_1))\mathbf{x}(t_2, j_1) \\&+ \gamma[1 - \kappa(\xi(t_2, j_1))]\mathbf{v}(\xi(t_2, j_1)).
    \end{aligned}
\end{equation}
According to the definition of the vector $\mathbf{v}(\xi)$ in \eqref{definition_of_vim} and the gate region $\mathcal{G}_z^i$ in \eqref{gate_region}, it is evident that at $(t_2, j_1)$ the vectors $-\mathbf{x}(t_2, j_1)$ and $\mathbf{v}(\xi(t_2, j_1))$ are equal. Hence, $\mathbf{u}(\xi(t_2, j_1))$ can equivalently be expressed as
\begin{equation}
    \mathbf{u}(\xi(t_2, j_1)) = -\gamma\mathbf{x}(t_2, j_1).
\end{equation}
At $(t_2, j_1 + 1)$, according to \eqref{update_law_for_m}, $\xi(t_2, j_1 + 1)\in\mathcal{F}_0\times\{0\}\times\mathbb{I}$. Hence, the control input vector $\mathbf{u}(\xi(t_2, j_1 + 1))$ is given as
\begin{equation}
    \mathbf{u}(\xi(t_2, j_1 + 1)) = -\gamma\mathbf{x}(t_2 , j_1 + 1).
\end{equation}
Since, according to \eqref{hybrid_closed_loop_system}, $\mathbf{x}(t_2, j_1) = \mathbf{x}(t_2, j_1 + 1)$, $\mathbf{u}(\xi(t_2, j_1)) = \mathbf{u}(\xi(t_2, j_1 + 1))$. As a result, when the solution flowing in the \textit{obstacle-avoidance} mode, enters the \textit{move-to-target} mode, the control vector trajectories remain continuous.
\end{appendix}

\bibliographystyle{IEEEtran}
\bibliography{reference}

\end{document}